\documentclass[conference]{IEEEtran}
\renewcommand{\IEEEtitletopspaceextra}{0pt}
\IEEEoverridecommandlockouts
\usepackage{cite}
\usepackage[table]{xcolor}
\usepackage{colortbl}
\usepackage{amsmath,amssymb,amsfonts}
\usepackage{algorithm}
\usepackage{algorithmic}
\usepackage{graphicx}
\usepackage{amsthm}
\newtheorem{thm}{Theorem}

\newtheorem{lemma}[thm]{Lemma}
\newtheorem{prop}{Proposition}

\usepackage{textcomp}
\usepackage{multirow}
\usepackage[utf8]{inputenc} 
\usepackage[T1]{fontenc}    
\usepackage{hyperref}       
\usepackage{url}            
\usepackage{booktabs}       
\usepackage{amsfonts}       
\usepackage{nicefrac}       
\usepackage{microtype}      

\usepackage[utf8]{inputenc} 
\usepackage[T1]{fontenc}    
\usepackage{hyperref}       
\usepackage{url}            
\usepackage{booktabs}       
\usepackage{amsfonts}       
\usepackage{nicefrac}       
\usepackage{microtype}      
\usepackage{algorithm}
\usepackage{tikz}
\usetikzlibrary{arrows.meta, calc, decorations.pathreplacing, positioning, shapes.geometric}
\usepackage{subcaption}
\usepackage{float}
\usepackage{svg}
\usepackage{lineno}
\usepackage{cuted}
\usepackage{caption}
\def\BibTeX{{\rm B\kern-.05em{\sc i\kern-.025em b}\kern-.08em
    T\kern-.1667em\lower.7ex\hbox{E}\kern-.125emX}}

\begin{document}

\title{Model-Based Geometry-Aware Generative Optimization for Constrained Locomotion Planning}

\author{
\IEEEauthorblockN{
Zhilin He\textsuperscript{1},
Xinyuan Wang\textsuperscript{2},
Changliu Liu\textsuperscript{1}
}
\IEEEauthorblockA{
\textsuperscript{1}Carnegie Mellon University,
\textsuperscript{2}University of Michigan, Ann Arbor
}
\IEEEauthorblockA{
\texttt{\{hectorh, cliu6\}@andrew.cmu.edu},
\texttt{\{wxinyuan\}@umich.edu}
}
}
\begingroup
\IEEEaftertitletext{\begin{minipage}{\textwidth}
    \centering
    
    \includegraphics[width=\linewidth]{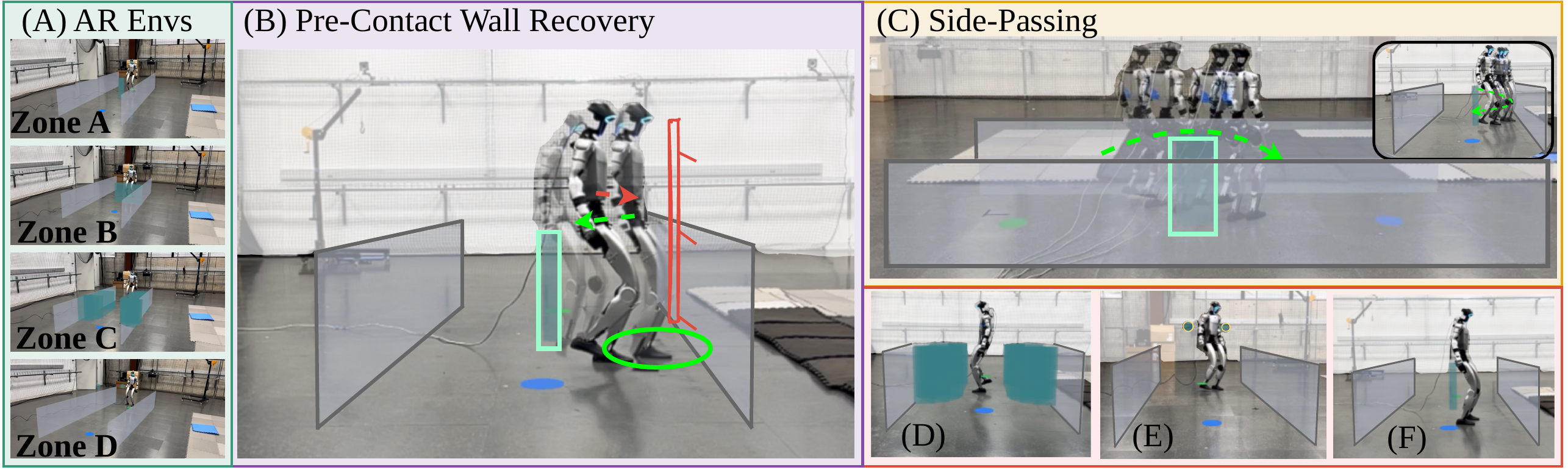}
    \captionof{figure}{\textbf{Real-world deployment of 2GO for constrained humanoid locomotion.}
    \textbf{(A)} Four Augmented Reality (AR)-rendered constrained environments (Zone A-D) for real-world humanoid locomotion evaluation. The robot is AR obstacle-aware, although the AR overlay does not show full depth rendering or realistic rigid-body contact.
    \textbf{(B)} Pre-contact recovery as the humanoid redirects away from an impending wall collision during online accelerated motion.
    \textbf{(C)} Multimodal side-passing around an obstacle.
    \textbf{(D)} Traversing a narrow corridor.
    \textbf{(E)} Passing between suspended ball obstacles.
    \textbf{(F)} Standing stably at the final target.}
    \label{fig:2go_deploy}
\end{minipage}
}
\maketitle

\bstctlcite{twogo_bib_format}
\setlength{\abovedisplayskip}{1.5pt plus 1pt minus 0.5pt}
\setlength{\belowdisplayskip}{1.5pt plus 1pt minus 0.5pt}
\setlength{\abovedisplayshortskip}{1pt plus 0.5pt}
\setlength{\belowdisplayshortskip}{1pt plus 0.5pt minus 0.5pt}
\setlength{\jot}{1.5pt}
\setlength{\IEEEiedtopsep}{0pt}
\setlength{\textfloatsep}{8pt plus 2pt minus 2pt}
\setlength{\dbltextfloatsep}{12pt plus 2pt minus 1pt}
\setlength{\floatsep}{6pt plus 2pt minus 2pt}
\setlength{\dblfloatsep}{6pt plus 2pt minus 2pt}
\setlength{\intextsep}{6pt plus 2pt minus 2pt}

\begin{abstract}
Constrained Locomotion Planning (CLP) for quadrupeds and humanoids, where robots must satisfy collision avoidance, contact consistency, kinematic feasibility, and support constraints, is challenging under high-dimensional dynamics and highly non-convex environments. Recent Model-Based Diffusion (MBD) approaches recast trajectory optimization as posterior sampling over trajectories, using known dynamics and Monte Carlo rollouts to analytically estimate the denoising score function without demonstration learning. While constrained variants further incorporate feasibility into model-based score rollouts and show promising performance, they are still limited by (1) lacking a task-modulated active constraint geometry that shapes the score direction and reverse stochasticity, and (2) using deterministic DDPM-style reverse transport without adaptive scheduling across different generative transports. Therefore, we introduce Model-Based Geometry-Aware Generative Optimization (2GO) for constrained locomotion, which turns active constraint geometry into executable denoising operators through normal- induced metric shaping, tangent-space stochastic filtering, and CFS-based retraction. 2GO further decouples generative transport from reverse stochasticity through an adaptive diffusion and flow-like schedule. Experiments on constrained quadruped and humanoid locomotion demonstrate strong performance in discrete foothold selection and continuous posture planning, with higher success rates, fewer violations, and improved execution compatibility.
Code: \href{https://github.com/hhhhzl/genedynamic}{\texttt{https://github.com/hhhhzl/genedynamics}}
\end{abstract}

\vspace{-0.2cm}
\section{Introduction}
\vspace{-0.2cm}

Consider a quadruped crossing sparse stepping stones or a humanoid traversing a narrow obstacle field: the planner must decide a sequence of body motions, foothold or contact choices, and posture references while satisfying collision, kinematic, support, and actuation constraints~\cite{posa2014direct, wensing2023optimization, gu2026humanoid}. The challenge is not merely to find a collision-free path, but to produce a dynamically executable reference that remains consistent with contact transitions, hybrid dynamics, and the tracking limitations of the low-level controller~\cite{kim2019highly,sleiman2021unified}.

Recent approaches recast locomotion planning as probabilistic inference, using score-based generative models to sample and denoise trajectory posteriors into clean, multimodal behavior trajectories~\cite{huang2025flexible, ye2025pegasusflow, huang2024diffuseloco}. Compared with conventional trajectory optimization and Model Predictive Control (MPC) methods, some of which rely on local linearization, prescribed contact schedules, or good initialization and may struggle with multimodal contact choices and non-smooth feasibility boundaries in high-dimensional locomotion~\cite{mayne2000constrained, kim2025contact, dhedin2025simultaneous}, generative approaches can flexibly synthesize trajectories, condition on goals or rewards, and represent high-dimensional multimodal action distributions. However, learned generative policies often rely heavily on demonstrations, and their learned priors alone do not ensure consistency with robot dynamics, environment geometry, or imposed safety constraints.

Model-based approaches address this limitation by leveraging Monte Carlo rollouts under known dynamics to estimate the denoising score function, thereby eliminating the need for demonstration data~\cite{pan2024model,xue2025full}. Recent constrained variants and related diffusion planners further incorporate feasibility into the denoising process through Emerging Barrier (EB) or inexact Augmented Lagrangian Method (iALM) priors, Control Barrier Function (CBF)-constrained projections, and adaptive scheduling of feasibility enforcement~\cite{mishra2025eb,fishman2023diffusion,he2026model,he2026motion}. Although these methods move beyond post-hoc projection, existing methods remain limited in (\textbf{1}) explicitly using task-modulated active constraint geometry to shape both denoising refinement and reverse stochasticity, and (\textbf{2}) coordinating deterministic reverse transport with separately controlled stochastic exploration under a unified adaptive schedule. Here, reverse stochasticity refers to perturbations added to the reverse update, distinct from randomness in candidate sampling. Consequently, task-specific feasible geometry remains underused in directing multimodal exploration and balancing it with deterministic refinement as the local constraint geometry becomes better resolved.

We propose a Model-Based \textbf{\underline{G}}eometry-aware \textbf{\underline{G}}enerative \textbf{\underline{O}}ptimization (\textbf{2GO}) framework (Fig.~\ref{fig:2go_overview}) for Constrained Locomotion Planning (CLP), which builds an active constraint geometry from violation proxies and task context to shape model-based denoising refinement through a metric and stochastic exploration through a regularized tangent filter, complemented by optional Convex Feasible Set (CFS)-based correction~\cite{liu2018convex}. 2GO further extends the deterministic Denoising Diffusion Probabilistic Models (DDPM)~\cite{ho2020denoising}-style reverse transport used in MBD with separately controlled reverse stochasticity, where task-modulated scheduling balances stochastic exploration and deterministic refinement towards a clean trajectory. Our \textbf{C}ontributions are two-fold:
\begin{itemize}
    \item \textbf{C1:} We develop a geometry-aware reverse process with active-set selection, normal-induced metric shaping, and regularized tangent filtering, explicitly shaping both denoising refinement and stochastic exploration, with CFS-based correction for residual constraint violations.
    \item \textbf{C2:} We formulate a model-based generative reverse update that decouples the transport family from reverse stochasticity, accommodating DDPM, Denoising Diffusion Implicit Models (DDIM)~\cite{song2020denoising}, and flow-like transports under adaptive task-geometry scheduling.
\end{itemize}
\begin{figure*}[!t]
    \centering
    \includegraphics[width=\linewidth,trim=35bp 5bp 0bp 0bp,clip]{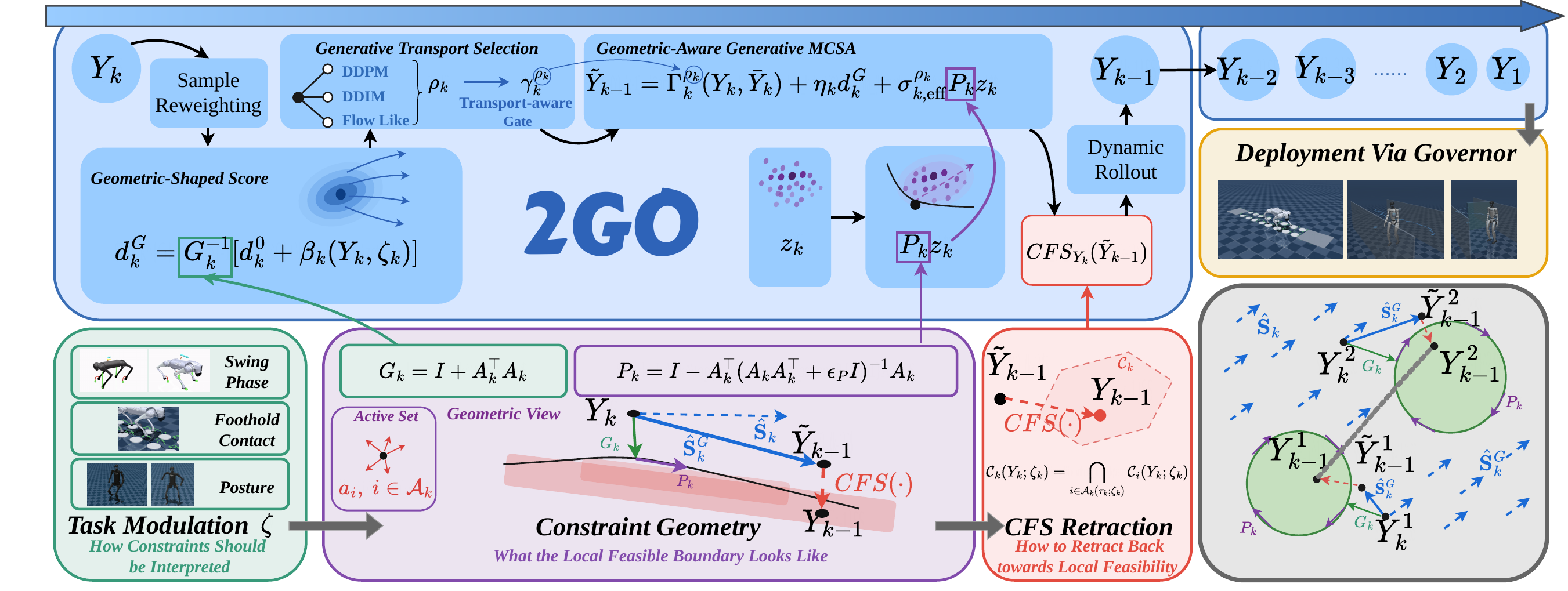}
    \caption{\textbf{Overview of the 2GO framework.}
    At each reverse step, 2GO reweights model-based trajectory samples and
    selects a DDPM-, DDIM-, or flow-like transport. Task-modulated active
    constraints shape the score through $G^{-1}_k$ and filter stochasticity through
    $P_k$, with optional CFS correction; a reference governor converts the clean trajectory into executable commands. In a two-time-step foothold contact example (bottom right gray panel), the score $\hat{\mathbf S}_{k}$ guides two unconstrained foothold points, $Y_{k}^{1}$ and $Y_{k}^{2}$, which are reshaped into $\tilde{Y}_{k-1}^{1}$ and $\tilde{Y}_{k-1}^{2}$ by the $G^{-1}_k$-shaped $\hat{\mathbf S}_{k}^{G}$ and $P_k$-filtered exploration, then retracted into the support regions via CFS. The final swing trajectory is shown as a gray line.}
    \label{fig:2go_overview}
    \vspace{-0.3cm}
\end{figure*}

\vspace{-0.2cm}
\section{Related Work}
\vspace{-0.2cm}
\subsection{Score-Based Generative Models for Robot Motion Planning}
Existing approaches on score-based generative models for robot motion planning aim to obtain a feasible trajectory by learning the score function from demonstration data and then sampling from a posterior over trajectory distributions~\cite{janner2022planning, chi2025diffusion, carvalho2023motion}. Current works further adapt to constrained environments through weighted penalty, or augmented Lagrangian methods~\cite{zhang2025constrained,liang2024multi}, and projection to safe sets within the generative reverse process~\cite{christopher2024constrained,xiao2023safediffuser,liang2025simultaneous}, and real-time control for deployment~\cite{romer2024diffusion, ye2025pegasusflow, huang2025flexible}. While such methods improve safety, they still rely on large-scale demonstrations and do not rigorously enforce dynamics.

Model-Based Diffusion~\cite{pan2024model} recovers trajectories through model rollouts to enforce dynamics, driving the reversing trajectories toward low-cost solutions without demonstrations. More recent works ensure safety through Control Barrier Function (CBF)-based or CFS-based projections within the model-based rollouts~\cite{mishra2025eb,fishman2023diffusion,cheng2025safe, he2026model}. However, those methods use projection to satisfy constraints, but do not expose active constraint geometry to the score direction and reverse stochasticity.

\vspace{-0.2cm}
\subsection{Optimization-Based and Geometry-Aware Planning}
\vspace{-0.1cm}
Optimization-based locomotion methods formulate locomotion as optimal control problems and solve them with contact-implicit MPC, whole-body models, or sampling-based MPC for high-dimensional quadrupeds and humanoids~\cite{kim2025contact, dhedin2025simultaneous, xue2025full,li2025continual}. In parallel, geometry-aware motion generation methods such as covariant trajectory optimization~\cite{ratliff2009chomp}, Riemannian Motion Policies~\cite{cheng2021rmpflow}, and geometric fabrics use task-space metrics, natural-gradient-like updates, or manifold structures to shape robot trajectories and motion policies based on obstacles, tasks, and environments~\cite{van2022geometric, kyaw2026geometry, bakker2025tamedpuma}.

 While these methods exploit models and geometric structure for trajectory optimization and motion-policy generation, we apply geometry-aware shaping to denoising refinement and reverse stochasticity within a model-based generative planner.
\vspace{-0.3cm}
\section{Preliminaries}
\label{sec:preliminary}
\vspace{-0.2cm}
\subsection{Problem Statement}
\label{subsec:problem_setup}

Consider a robot  with state $s^t\in\mathcal S$ and control
$u^t\in\mathcal U$ at time $t$.
A finite-horizon trajectory is denoted by
\begin{math}
    \tau=(s^{1:T},u^{1:T-1}).
\end{math}
The constrained trajectory optimization problem is
{%
\setlength{\abovedisplayskip}{2.5pt}
\setlength{\belowdisplayskip}{2.5pt}
\setlength{\abovedisplayshortskip}{1.5pt}
\setlength{\belowdisplayshortskip}{1.5pt}
\begin{equation}
\begin{aligned}
    &\min_{\tau_r}
    J(\tau,c) = \ell_T(s^T,c) + \sum_{t=1}^{T-1} \ell^t(s^t,u^t,c), \\
    &\mathrm{s.t.} \ s^{t+1}=f(s^t,u^t), \ g^t(s^t,u^t,c)\le 0,
\end{aligned}
\label{eq:trajopt}
\end{equation}
}where $\ell_T$ and $\ell^t$ encode terminal and intermediate task objectives; $f(\cdot)$ denotes dynamics; $g^t$ denotes the inequality constraints; $c$ includes task information such as the initial state, goal condition, environment geometry, dynamics, local map, contact hints, or phase information. A model-based generative view converts Eq.~\eqref{eq:trajopt} into a target distribution~\cite{peters2010relative,levine2018reinforcement}, 
{%
\setlength{\abovedisplayskip}{2.5pt}
\setlength{\belowdisplayskip}{2.5pt}
\setlength{\abovedisplayshortskip}{1.5pt}
\setlength{\belowdisplayshortskip}{1.5pt}
\begin{equation}
    p(\tau\mid c)
    \propto
    \exp\left(-\frac{1}{\lambda}J(\tau,c)\right)
    \mathbf 1_{\mathcal F(c)}(\tau),
\label{eq:inference}
\end{equation}
}where $\mathbf 1_{\mathcal F(c)}(\tau)$ is the indicator of the feasible trajectory set
\begin{equation}
\mathcal F(c) = \left\{
\tau:
s^{t+1}=f(s^t,u^t),\;
g^t(s^t,u^t,c)\le0
\right\}.
\end{equation}
and lower-cost feasible trajectories receive higher probability. The distribution Eq.~\eqref{eq:inference} turns constrained trajectory optimization into sampling from low-cost feasible trajectories, approximating by denoising score from dynamics rollouts within model-based generative sampling.

\vspace{-0.2cm}
\subsection{Model-Based Generative Trajectory Sampling}
\label{subsec:mb_generative_sampling}
\vspace{-0.2cm}
Throughout, $t$ indexes physical time, $k$ indexes reverse refinement, and
$m$ indexes sampled candidates. The generative process starts from approximately
Gaussian noise $U_K$ and iteratively refines it toward clean controls $U_0$.
We use clean coordinates $Y_k=U_k/\sqrt{\bar\alpha_k}$.
A Gaussian forward process maps clean controls $U_0$ to noisy controls:
\begin{math}
U_k
=
\sqrt{\bar\alpha_k}U_0
+
\sqrt{1-\bar\alpha_k}\epsilon,
\label{eq:forward_reparam_2go}
\end{math}
where $\epsilon\sim\mathcal N(0,I)$, $\bar\alpha_k:=\prod_{j=1}^{k}\alpha_j$,
$\bar\alpha_0=1$, and $\alpha_k\in(0,1)$ for $k\geq1$.
The forward process induces the noisy marginal
{%
\setlength{\abovedisplayskip}{1.5pt}
\setlength{\belowdisplayskip}{1.5pt}
\setlength{\abovedisplayshortskip}{1pt}
\setlength{\belowdisplayshortskip}{1pt}
\begin{equation}
p_k(U_k)
:=
\int q(U_k\mid U_0)\,p_0(U_0)\,dU_0 .
\label{eq:pk_def_2go}
\end{equation}
}The reverse generative process is driven by the score function of the noisy marginal \begin{math}
\mathbf S_k(U_k, c) = \nabla_{U_k}\log p_k(U_k). \label{eq:score_def_2go} \end{math}

Unlike DDPM, which learns a score estimator from data~\cite{ho2020denoising},
MBD obtains $\hat{\mathbf S}_k$ from model evaluations: it samples controls,
evaluates their rollouts, and forms a weighted clean estimate. The Gaussian
noising identity then converts this estimate into a score estimate.
Given $U_k$, the forward kernel induces the clean-coordinate proposal
{%
\setlength{\abovedisplayskip}{1.5pt}
\setlength{\belowdisplayskip}{1.5pt}
\setlength{\abovedisplayshortskip}{1pt}
\setlength{\belowdisplayshortskip}{1pt}
\begin{equation}
\tilde Y \sim \Omega_k
:= \mathcal{N}\!\left(
\frac{U_k}{\sqrt{\bar{\alpha}_k}},
\left(\frac{1}{\bar{\alpha}_k} - 1\right)\mathbf{I}
\right),
\label{eq:clean_proposal_2go}
\end{equation}
}

We draw a batch of candidates
\begin{math}
\{\tilde Y_{k,m}\}_{m=1}^{M},
\end{math}
\begin{math}
\tilde Y_{k,m}\sim \Omega_k ,
\label{eq:batch_2go}
\end{math}
evaluate rollouts $\tilde\tau_{k,m}=\mathcal R_f(\tilde Y_{k,m};s_0)$ under $p_0$, and form
the importance-weighted clean estimate
{%
\setlength{\abovedisplayskip}{1.5pt}
\setlength{\belowdisplayskip}{1.5pt}
\setlength{\abovedisplayshortskip}{1pt}
\setlength{\belowdisplayshortskip}{1pt}
\begin{equation}
w_m=\frac{p_0(\tilde Y_{k,m})}{\sum_jp_0(\tilde Y_{k,j})},
\quad\bar Y_k=\sum\nolimits_mw_m\tilde Y_{k,m}.
\label{eq:mc_avg_2go}
\end{equation}
}This weighted mean supplies the denoising reference in the Monte Carlo
score ascent (MCSA) estimator
{%
\setlength{\abovedisplayskip}{1.5pt}
\setlength{\belowdisplayskip}{1.5pt}
\setlength{\abovedisplayshortskip}{1pt}
\setlength{\belowdisplayshortskip}{1pt}
\begin{equation}
\hat{\mathbf{S}}_k
\approx
-\frac{U_k}{1-\bar{\alpha}_k}
+\frac{\sqrt{\bar{\alpha}_k}}{1-\bar{\alpha}_k}\,
\bar Y_k.
\label{eq:mc_score_2go}
\end{equation}
}

\textbf{DDPM.}
Given the model-based score estimate $\hat{\mathbf S}_k$, a DDPM-style score-form reverse update is
\begin{equation}
U_{k-1}
=
\frac{1}{\sqrt{\alpha_k}}
\left(
U_k
+
(1-\bar\alpha_k)\hat{\mathbf S}_k
\right)
+
\sigma_k z_k,
\
z_k\sim\mathcal N(0,\mathbf I).
\label{eq:ddpm_reverse_2go}
\end{equation}
Here $\sigma_k$ controls the reverse stochasticity.

\textbf{DDIM.}
DDIM~\cite{song2020denoising} combines the clean estimate with the inferred
noise direction. With $c_k=\sqrt{1-\bar\alpha_{k-1}-\sigma_k^2}$,
\begin{equation}
U_{k-1}=\sqrt{\bar\alpha_{k-1}}\bar Y_k
-c_k\sqrt{1-\bar\alpha_k}\hat{\mathbf S}_k+\sigma_kz_k.
\label{eq:ddim_update_2go}
\end{equation}
Here $0\leq\sigma_k^2\leq1-\bar\alpha_{k-1}$; $\sigma_k=0$ yields
a deterministic update that retains the inferred residual direction.

\textbf{Probability flow.}
For the variance-preserving schedule, the probability-flow velocity is
$\hat v_k=-\frac12\beta(h_k)[U_k+\hat{\mathbf S}_k]$, where
$\beta(t)=-\frac{d}{dt}\log\bar\alpha(t)$ and $t$ here denotes continuous
diffusion time~\cite{song2021score}.
Flow-like (FL) option integrates this velocity backward:
{%
\setlength{\abovedisplayskip}{1.5pt}
\setlength{\belowdisplayskip}{1.5pt}
\setlength{\abovedisplayshortskip}{1pt}
\setlength{\belowdisplayshortskip}{1pt}
\begin{equation}
U_{k-1}\approx U_k+\frac{\Delta B_k}{2}
\left[U_k+\hat{\mathbf S}_k\right],
\label{eq:fm_update_expanded_2go}
\end{equation}
}where $\Delta B_k=\int_{t_{k-1}}^{t_k}\beta(t)dt
=\log(\bar\alpha_{k-1}/\bar\alpha_k)>0$.
This map integrates a local velocity, whereas DDIM reconstructs the next
endpoint from clean and noise estimates.

MBD~\cite{pan2024model} and its variants~\cite{mishra2025eb, he2026model, he2026motion} fix $\sigma_k=0$ and adopt a deterministic DDPM-style reverse update. Therefore, we aim to construct an adaptive scheduling mechanism that selects among DDPM-like Eq.~\eqref{eq:ddpm_reverse_2go}, DDIM-like Eq.~\eqref{eq:ddim_update_2go}, and flow-like Eq.~\eqref{eq:fm_update_expanded_2go} reverse transports, while allowing adaptive $\sigma_k>0$ as exploration perturbation when needed. 
DDPM-style steps provide robust stochastic exploration but higher variance, DDIM-style steps offer schedule-consistent low-variance refinement but may commit early to one mode, and flow-like Euler steps are smooth and inexpensive but more sensitive to step size and score error.

\vspace{-0.3cm}
\section{Methodology}
\label{sec:method}
\vspace{-0.3cm}

We now present our Model-Based \textbf{\underline{G}}eometry-aware \textbf{\underline{G}}enerative \textbf{\underline{O}}ptimization (\textbf{2GO}) (Fig.~\ref{fig:2go_overview}), which
(\textbf{1}) turns task-modulated active constraint geometry into score-direction shaping,
tangent stochasticity, and CFS-based retraction, (\textbf{2}) adapts the reverse
transport and stochasticity schedule in Section~\ref{subsec:mcsa_geometry_tau},
and (\textbf{3}) connects the clean plan to the tracking policy through the
governor in Section~\ref{subsec:real_deployment_governor}.
\vspace{-0.2cm}
\subsection{Task-modulated Active Constraint Geometry}
\label{subsec:constraint_task_geometry}
\vspace{-0.2cm}
\textbf{Task Modulation.}
At diffusion step $k$, task modulation emphasizes constraints according to
phase or contact structure: support-foot constraints during stance, for
example, and target-patch constraints near foot placement.
We encode this through task-modulated weights
\begin{math}
    \omega_i(k;\zeta)\ge 0,
\label{eq:task_weight_tau}
\end{math}
where $\zeta$ contains phase, contact mode, target patch,
or local reference information.

\textbf{Constraint Geometry.}
For each inequality constraint $g_i(\tau_k,c)\le 0$, we define the
task-weighted violation proxy
{%
\setlength{\abovedisplayskip}{1.5pt}
\setlength{\belowdisplayskip}{1.5pt}
\setlength{\abovedisplayshortskip}{1pt}
\setlength{\belowdisplayshortskip}{1pt}
\begin{equation}
    v_i(\tau_k,c,k;\zeta)
    =
    \varphi(g_i(\tau_k,c))\,\omega_i(k;\zeta),
\label{eq:violation_proxy_tau}
\end{equation}
}where $\varphi(g)=\max\{g,0\}$ assigns zero to a satisfied constraint
and retains the positive residual of a violated constraint.
The active set at diffusion step $k$ is selected as
{%
\setlength{\abovedisplayskip}{1.5pt}
\setlength{\belowdisplayskip}{1.5pt}
\setlength{\abovedisplayshortskip}{1pt}
\setlength{\belowdisplayshortskip}{1pt}
\begin{equation}
    \mathcal{A}_k(\tau_k;\zeta)
    =
    \mathrm{TopK}_{K_A}\left(
        \{v_i(\tau_k,c,k;\zeta)\}_{i=1}^{N_c}
    \right).
\label{eq:active_set_tau}
\end{equation}
}Here, $N_c$ is the number of constraints, and $\mathrm{TopK}$
returns constraint indices. $K_A$ is the maximum number
of constraints retained, and $K_{a,k}=|\mathcal A_k|\leq K_A$
is the number selected.

For each selected constraint $i\in\mathcal A_k$, define the
exact normal in the $d$-dimensional free-control coordinates as
{%
\setlength{\abovedisplayskip}{1.5pt}
\setlength{\belowdisplayskip}{1.5pt}
\setlength{\abovedisplayshortskip}{1pt}
\setlength{\belowdisplayshortskip}{1pt}
\begin{equation}
    a_i(Y_k)
    =
    D_Y\mathcal R_f(Y_k;s_0)^\top
    \nabla_\tau g_i(\tau_k,c),
\label{eq:active_normal_tau}
\end{equation}
}where $\tau_k=\mathcal R_f(Y_k;s_0)$ rolls out from fixed $s_0$ and
$\widetilde g_i(Y,c)=g_i(\mathcal R_f(Y;s_0),c)$.
Experiments use aggregated one-step action-space approximations (automatic
differentiation and kinematic mappings; Appendix~\ref{app:geometry_proxy});
the follower is not differentiated.
We normalize the normal as $\hat a_i=a_i/(\|a_i\|+\varrho)$,
where $\varrho>0$ avoids division by small gradient norms.
Stacking these normalized normals as rows gives
\begin{equation}
    A_k=\bigl[\hat a_i(Y_k)^\top\bigr]_{i\in\mathcal A_k}
    \in\mathbb R^{K_{a,k}\times d}.
\label{eq:A_tau}
\end{equation}
For example, if three constraints have scores $(0.1,0.6,0.3)$ and $K_A=2$,
then $\mathcal A_k=\{2,3\}$. The matrix $A_k$ therefore contains the two
rows $\hat a_2^\top$ and $\hat a_3^\top$.

These constraint normals define the local metric
\begin{equation}
    G_k
    =
    I
    +
    A_k^\top A_k
    =
    I
    +
    \sum\nolimits_{i\in\mathcal{A}_k}
    \hat a_i\hat a_i^\top,
\label{eq:metric_tau}
\end{equation}
whose normal-gradient shaping is related to trajectory optimization and
geometric motion policies~\cite{ratliff2009chomp,cheng2021rmpflow,van2022geometric}.
Task weights affect $G_k$ through active-set selection; each selected
normalized normal enters the metric with unit weight.

The inverse metric $G_k^{-1}$ attenuates constraint-normal components of the
denoising direction. For random exploration, we use a regularized tangent filter
\begin{equation}
    P_k
    =
    I
    -
    A_k^\top
    \left(
        A_kA_k^\top+\epsilon_P I
    \right)^{-1}
    A_k,
\label{eq:tangent_projector_tau}
\end{equation}
where $\epsilon_P>0$ regularizes a linear system whose dimension
$K_{a,k}$ equals the number of selected constraints. Woodbury's identity
evaluates the metric action through a linear system of the same dimension:
\begin{equation}
    G_k^{-1}q
    =
    q-A_k^\top(I+A_kA_k^\top)^{-1}A_kq.
\label{eq:woodbury_metric}
\end{equation}
For a singular value $s$ of $A_k$, the metric and tangent gains are
$1/(1+s^2)$ and $\epsilon_P/(\epsilon_P+s^2)$, respectively.
Both preserve $\ker(A_k)$: with exact normals, these directions leave
selected constraint values unchanged to first order.

\vspace{-0.1cm}
\subsection{Geometry-Aware Generative Reverse Update}
\label{subsec:mcsa_geometry_tau}

\textbf{Geometry-shaped Direction.}
The MCSA clean estimate $\bar Y_k\equiv\hat Y_{0\mid k}$ in
Eq.~\eqref{eq:mc_score_2go} defines a displacement from the current iterate,
expressed in clean coordinates $Y_k=U_k/\sqrt{\bar\alpha_k}$:
\begin{equation}
    d_k^0
    =
    \kappa_k(\bar Y_k-Y_k)
    \approx
    \kappa_k\frac{1-\bar\alpha_k}{\sqrt{\bar\alpha_k}}\hat{\mathbf S}_k.
\label{eq:base_reverse_direction_tau}
\end{equation}
Here $\kappa_k$ scales this displacement, and geometry and goal derivatives
are evaluated at $Y_k$ through the rollout. For goal guidance, let
$e_k=r_{\mathrm{goal}}-\psi_g(Y_k)$ for a goal map $\psi_g$.
The transposed Jacobian maps the normalized goal error back to the planning
coordinates:
\begin{equation}
    \beta_k(Y_k,\zeta_k)
    =
    \lambda_g(1-h_k)
    D\psi_g(Y_k)^\top\frac{e_k}{\|e_k\|+\epsilon_d},
\label{eq:task_dir}
\end{equation}
where $\lambda_g$ is the gain and $\epsilon_d>0$ regularizes normalization.
For scheduled feasibility penalty $\mu_k$ and reference penalty $\mu_{\mathrm{ref}}>0$, define
$h_k=\operatorname{clip}(\log(1+\mu_k)/\log(1+\mu_{\mathrm{ref}}),0,1)$.
This normalized level reduces goal guidance as feasibility emphasis rises;
$\kappa_k=\kappa_0+\kappa_\mu\mu_k$ sets the denoising gain.
Applying the inverse metric to the combined displacement and goal drift gives
\begin{equation}
    d_k^G
    =
    G_k^{-1}
    \left[
        d_k^0+\beta_k(Y_k,\zeta_k)
    \right].
\label{eq:geometry_score_tau}
\end{equation}
Here, $d_k^G$ denotes a shaped displacement. Its equivalent score and
clean estimates formulate
\begin{equation}
\hat{\mathbf S}_k^G=\frac{\sqrt{\bar\alpha_k}}{1-\bar\alpha_k}d_k^G,
\quad \hat Y^G_{0\mid k}=Y_k+d_k^G.
\label{eq:geometry_estimates_tau}
\end{equation}
The corresponding noise estimate is
$\hat\epsilon_k^G=-\sqrt{1-\bar\alpha_k}\hat{\mathbf S}_k^G$.
These representations connect geometry shaping to MCSA; the reverse update
below applies $d_k^G$ as an additive refinement of the base transport.

\textbf{Reverse Update and Stochasticity Scheduling.}
We select the transport and its added exploration separately. Let
\begin{math}
    \rho_k\in\{\mathrm{DDPM},\mathrm{DDIM},\mathrm{FL}\}
\end{math}
denote the transport family at reverse step $k$.
Candidate diversity is measured by a dispersion proxy
\begin{math}
    \chi_k(\tau_k,c,\zeta_k)\in[0,1],
\label{eq:dispersion_proxy_k}
\end{math}
computed here as the exponentially smoothed largest normalized adjacent gap after projecting
candidate control segments onto their leading principal direction.
To also capture residual feasibility risk, define
{%
\setlength{\abovedisplayskip}{1.5pt}
\setlength{\belowdisplayskip}{1.5pt}
\setlength{\abovedisplayshortskip}{1pt}
\setlength{\belowdisplayshortskip}{1pt}
\begin{equation}
    \nu_k(\tau_k,c,\zeta_k)
    =
    \mathrm{CVaR}_{\alpha}
    \left(
        \{V(\tilde\tau_{k,m},c)\}_{m=1}^{M}
    \right).
\label{eq:risk_proxy_k}
\end{equation}
}where $V$ aggregates each candidate's constraint residuals. At tail level
$\alpha=0.9$, $\nu_k$ averages the largest 10\% of candidate residuals.
Risk and dispersion determine the transport:
\begin{equation}
\rho_k=\begin{cases}
\mathrm{DDPM},&\nu_k>\delta_k,\\
\mathrm{DDIM},&\nu_k\leq\delta_k,\ \chi_k>\theta_k,\\
\mathrm{FL},&\text{otherwise},
\end{cases}
\label{eq:transport_selection_main}
\end{equation}
where $\delta_k$ and $\theta_k$ are scheduled risk and dispersion thresholds.
High risk favors the clean-moment branch; resolved risk and route ambiguity
favor interpolation and then flow-like refinement. The separate noise gate is
\begin{equation}
    \gamma_k^{\rho_k}
    =
    \mathbb{I}
    \left[
        \chi_k
        >
        \bar\chi_k^{\rho_k}
        \;\lor\;
        \nu_k
        >
        \bar\nu_k^{\rho_k}
    \right],
\label{eq:rho_aware_gate}
\end{equation}
where $\bar\chi_k^{\rho_k}$ and $\bar\nu_k^{\rho_k}$ are
transport-dependent thresholds. The gate retains exploration when diversity
or residual risk exceeds its threshold. The effective stochasticity is
\begin{math}
    \sigma_{k,\mathrm{eff}}^{\rho_k}
    =
    \gamma_k^{\rho_k}\sigma_k^{\rho_k}.
\label{eq:rho_aware_sigma}
\end{math}
where \(\sigma_k^{\rho_k}\) is the nominal exploration scale of the selected
transport, expressed in clean coordinates.

We filter Gaussian noise $z_k\sim\mathcal N(0,I)$ through $P_k$ to
attenuate perturbations along the selected constraint normals.
Let $\Gamma_k^{\rho_k}$ denote the selected deterministic base transport
in clean coordinates:
\begin{equation}
\Gamma_k^{\rho_k}(Y_k,\bar Y_k)
=(1-r_k^{\rho_k})\bar Y_k+r_k^{\rho_k}Y_k.
\label{eq:clean_base_main}
\end{equation}
Here $r_k^{\rho_k}$ is the fraction of the current iterate retained.
The DDPM-style clean-mean update and deterministic DDIM update in
Eqs.~\eqref{eq:ddpm_reverse_2go} and \eqref{eq:ddim_update_2go} give $r_k^{\mathrm{DDPM}}=0$ and
\begin{equation}
r_k^{\mathrm{DDIM}}=r_k^{\mathrm{FL}}
=\sqrt{\frac{\bar\alpha_k(1-\bar\alpha_{k-1})}
{\bar\alpha_{k-1}(1-\bar\alpha_k)}}.
\label{eq:clean_interpolation_main}
\end{equation}
The reported FL branch shares the DDIM interpolation; the canonical
probability-flow map in Eq.~\eqref{eq:fm_update_expanded_2go} is not
separately evaluated. Exploration is added without changing these coefficients.
With scheduled step size $\eta_k$, the unified update is
\begin{equation}
    \tilde Y_{k-1}
    =
    \Gamma_k^{\rho_k}
    \bigl(
        Y_k,\bar Y_k
    \bigr)
    +\eta_k d_k^G
    +
    \sigma_{k,\mathrm{eff}}^{\rho_k}
    P_k z_k .
\label{eq:unified_reverse_update_k}
\end{equation}
The first term transports the iterate, the second applies geometry-shaped
refinement, and the third adds exploration filtered by $P_k$. Setting
$\sigma_{k,\mathrm{eff}}^{\rho_k}=0$ leaves the selected deterministic update.

\vspace{-0.2cm}

\subsection{CFS as Metric-Proximal Constraint Retraction}
\label{subsec:cfs_retraction_tau}

The geometry-shaped reverse update reduces motion along active normals, but the
transported proposal $\tilde Y_{k-1}$ can still leave the local feasible
geometry. We therefore use CFS~\cite{liu2018convex} as an approximate
metric-proximal retraction. CFS optionally corrects candidates before averaging~\cite{he2026motion}
or retracts the transported proposal using the
constraint geometry underlying \(G_k\) and \(P_k\), and corrects the controls parameterized by $Y$, after which the states
are recomputed by dynamics rollout from the fixed initial state.

For each active constraint
$i\in\mathcal{A}_k(\tau_k;\zeta_k)$, CFS builds a local half-space around
$Y_k$:
{%
\setlength{\abovedisplayskip}{2.5pt}
\setlength{\belowdisplayskip}{2.5pt}
\setlength{\abovedisplayshortskip}{1.5pt}
\setlength{\belowdisplayshortskip}{1.5pt}
\begin{equation}
\begin{aligned}
    \mathcal C_i(Y_k;\zeta_k)
    =
    \{Y:\;&
    \widetilde g_i(Y_k,c)
    +
    a_i(Y_k)^\top(Y-Y_k)
    \\
    &
    \le -b_i(k;\zeta_k)\},
\end{aligned}
\label{eq:cfs_halfspace_tau}
\end{equation}
}where
\begin{math}
a_i(Y_k)=\nabla_Y\widetilde g_i(Y_k,c)
\end{math}
and $b_i(k;\zeta_k)\ge0$ is the CFS clearance offset.
Post-transport CFS reuses derivatives at $Y_k$; candidate CFS recomputes them locally.
Intersecting these half-spaces with control bounds gives the local convex set:
{%
\setlength{\abovedisplayskip}{2.5pt}
\setlength{\belowdisplayskip}{2.5pt}
\setlength{\abovedisplayshortskip}{1.5pt}
\setlength{\belowdisplayshortskip}{1.5pt}
\begin{equation}
\begin{aligned}
    \mathcal C_k(Y_k;\zeta_k)
    =
    \mathcal U^{T-1}\cap\bigcap_{i\in\mathcal A_k(\tau_k;\zeta_k)}
    \mathcal C_i(Y_k;\zeta_k).
\end{aligned}
\label{eq:local_cfs_set_tau}
\end{equation}
}
Given $\tilde Y_{k-1}$ from \eqref{eq:unified_reverse_update_k}, the ideal CFS
retraction is
\begin{equation}
\begin{aligned}
\mathrm{CFS}_{Y_k}(\tilde Y_{k-1})
&=
\arg\min_{Y\in\mathcal C_k(Y_k;\zeta_k)}
\frac{1}{2}\|Y-\tilde Y_{k-1}\|_{R}^{2},
\end{aligned}
\label{eq:cfs_retraction_qp_tau}
\end{equation}
where \(R\succ0\) weights control coordinates.
The selected constraints align this local convex retraction with the
geometry used by the metric and tangent filter.
For locally conservative linearizations, \(\mathcal C_k\) provides a
local inner approximation for the selected constraints.
When retraction is enabled, the next reverse iterate is
\begin{math}
    Y_{k-1}
    =
    \mathrm{CFS}_{Y_k}(\tilde Y_{k-1}).
\label{eq:final_retracted_tau}
\end{math}

Thus, CFS acts within the generative process as an optional
constraint-geometric retraction, using the constraint geometry
that also shapes denoising refinement and tangent stochasticity.

\begin{algorithm}[tb!]
\caption{Model-Based Geometry-Aware Generative Optimization (2GO)}
\label{alg:twogo}
\small
\begin{algorithmic}[1]
\STATE \begingroup\parfillskip=0pt \textbf{Input:} $U_K\!\sim\!\mathcal N(0,I)$, dynamics $f(\cdot)$, constraints $g(\cdot)\!\le\!0$, task context $\{\zeta_k\}_{k=1}^{K}$, diffusion schedule $\{\alpha_k,\bar\alpha_k\}_{k=1}^{K}$, candidates $M$.\par\endgroup
\STATE \textbf{Output:} Clean trajectory $\tau_0$ for governed execution.
\STATE Initialize clean coordinates $Y_K=U_K/\sqrt{\bar\alpha_K}$.
\FOR{$k = K$ \textbf{to} $1$}
    \STATE Sample candidates $\{\tilde Y_{k,m}\}_{m=1}^{M}$ from $\Omega_k$ in \eqref{eq:clean_proposal_2go}.
    \STATE Evaluate candidates with $f(\cdot)$, and optionally apply CFS.
    \STATE Compute $\bar Y_k,\hat{\mathbf S}_k$ with original weights via \eqref{eq:mc_avg_2go}--\eqref{eq:mc_score_2go}.
    \STATE Compute $v_i$ and select $\mathcal A_k$ at the rollout of $Y_k$ via \eqref{eq:violation_proxy_tau}--\eqref{eq:active_set_tau}.
    \STATE Form $A_k$, $G_k$, and $P_k$ using \eqref{eq:A_tau}--\eqref{eq:tangent_projector_tau}.
    \STATE Compute the geometry-shaped direction $d_k^G$ via \eqref{eq:geometry_score_tau}.
    \STATE \begingroup\parfillskip=0pt Compute $\chi_k,\nu_k$, adapt $\rho_k$ online, and gate $\sigma_{k,\mathrm{eff}}^{\rho_k}$ via \eqref{eq:transport_selection_main}--\eqref{eq:rho_aware_gate}.\par\endgroup
    \STATE Apply transport-aware reverse update \eqref{eq:unified_reverse_update_k} with $z_k\sim\mathcal N(0,I)$: \\
    $\tilde Y_{k-1}\leftarrow
    \Gamma_k^{\rho_k}(Y_k,\bar Y_k)
    +\eta_k d_k^G+\sigma_{k,\mathrm{eff}}^{\rho_k}P_k z_k$.
    \STATE Optionally retract $\tilde Y_{k-1}$ via
    \eqref{eq:cfs_retraction_qp_tau}, bounded as $Y_{k-1}$.
\ENDFOR
\STATE Roll out $\tau_0=\mathcal R_f(Y_0;s_0)$, extract $r=\Pi_r(\tau_0)$, and pass it through the reference governor for execution.
\end{algorithmic}
\end{algorithm}

\vspace{-0.2cm}
\subsection{Putting Everything Together: 2GO}
\label{subsec:twogo_algorithm_tau}

2GO combines (\textbf{1}) the model-based MCSA sampler
(Section~\ref{subsec:mb_generative_sampling}), (\textbf{2}) the task-modulated
active constraint geometry (Section~\ref{subsec:constraint_task_geometry}),
(\textbf{3}) the geometry-shaped reverse update
(Section~\ref{subsec:mcsa_geometry_tau}), and (\textbf{4}) the CFS retraction
(Section~\ref{subsec:cfs_retraction_tau}).

At each reverse step $k$, 2GO first draws $M$ clean control candidates from the
Gaussian proposal \eqref{eq:clean_proposal_2go} and, as in MBD, rolls them out
through dynamics, while the reverse update operates in clean coordinates $Y_k$.
After optional CFS correction of selected candidates, the batch keeps its
original weights to form $\bar Y_k$ and the score-form surrogate
\eqref{eq:mc_avg_2go}--\eqref{eq:mc_score_2go}. 2GO then selects the active set
\eqref{eq:active_set_tau} from the task-weighted violation proxies
\eqref{eq:violation_proxy_tau} and forms the shared geometric objects $A_k$,
$G_k$, and $P_k$ at $Y_k$, which shape the displacement $d_k^G$ via
\eqref{eq:geometry_score_tau}. Finally, risk and dispersion select the
transport $\rho_k$ and gate the stochasticity via
\eqref{eq:transport_selection_main}--\eqref{eq:rho_aware_gate}; noise enters
only through the tangent filter. The reverse step
\eqref{eq:unified_reverse_update_k} produces $\tilde Y_{k-1}$; after optional
CFS retraction \eqref{eq:cfs_retraction_qp_tau}, control bounding and rollout
give the next reverse iterate. Algorithm~\ref{alg:twogo} summarizes the full
procedure.

\vspace{-0.2cm}
\subsection{Governed Deployment}
\label{subsec:real_deployment_governor}

Given the clean trajectory $\tau_0$, we extract a compact tracking reference
$r^t=\Pi_r(\tau_0^t)$ containing the commanded base motion and posture.
Following the reference-governor principle~\cite{gilbert1995discrete,garone2017reference},
a customized adaptation maps $r^t$ to a governed reference ${}^g r^t$
for the low-level locomotion policy. The quadruped interface constructs
a gait and reachable footholds, while the humanoid interface adjusts
commands using tracking-margin-tightened body clearance, command limits,
and a posture envelope. These execution constraints differ from the
planner's feasible set. Planning bounds and command regularization
reduce tracking demands, while the governor and tracker remain outside
the planning dynamics. Clearance preservation requires admissible
adapted references and tracking errors covered by the reserved margins;
the adaptation alone does not guarantee completion of the original task.

\vspace{-0.1cm}
\begin{figure*}[tb!]
    \caption{\textbf{Qualitative stepping-stone results.} \textbf{Top left:} A representative 2GO deployment with the safety governor. \textbf{Top right:} 2GO planning samples at 90\%, 50\%, and 10\% of the diffusion reverse process. \textbf{Bottom:} Final clean trajectories and execution tracking of position $s$ and velocity $v$ for baselines, and ours 2GO. Colored lines show the trajectories of the left rear (LR), left front (LF),
    right rear (RR), and right front (RF) feet, and the body center;
    red dots mark constraint violations.}
    \label{fig:humanoid}
    \centering
    \begingroup
    \setlength{\fboxsep}{1.5pt}
    \setlength{\fboxrule}{0.6pt}
    \newlength{\stpgap}\setlength{\stpgap}{3pt}
    \newlength{\stpov}\setlength{\stpov}{\dimexpr 2\fboxsep+2\fboxrule\relax}
    \newlength{\stpw}\setlength{\stpw}{\dimexpr (\textwidth - 5\stpov - 4\stpgap)/5 \relax}
    \newcommand{\steppingalgobox}[3]{%
        \fbox{\begin{minipage}[t]{\stpw}\centering
            {\scriptsize\textbf{#1}}\\[1pt]
            #2\\[1pt]
            #3%
        \end{minipage}}%
    }
    \resizebox{\textwidth}{!}{%
        \begin{minipage}[t]{\dimexpr 2\stpw+2\stpov+\stpgap\relax}\centering
            {\scriptsize\textbf{Deployment Via Governor}}\\
            \includegraphics[width=\linewidth]{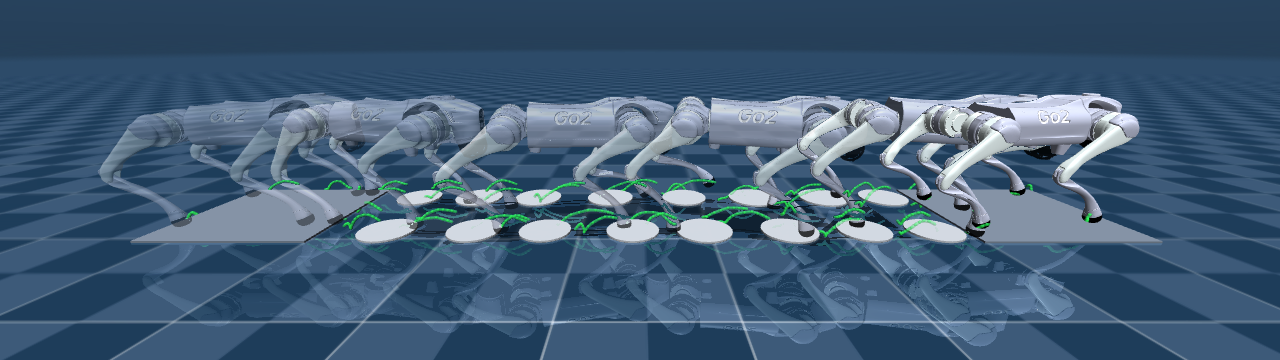}%
        \end{minipage}\hspace{\stpgap}%
        \fbox{\begin{minipage}[t]{\dimexpr 3\stpw+2\stpov+2\stpgap\relax}\centering
            {\scriptsize\textbf{Generative Reverse Processing at 90\%, 50\%, and 10\% for 2GO}}\\
            \includegraphics[width=0.333\linewidth]{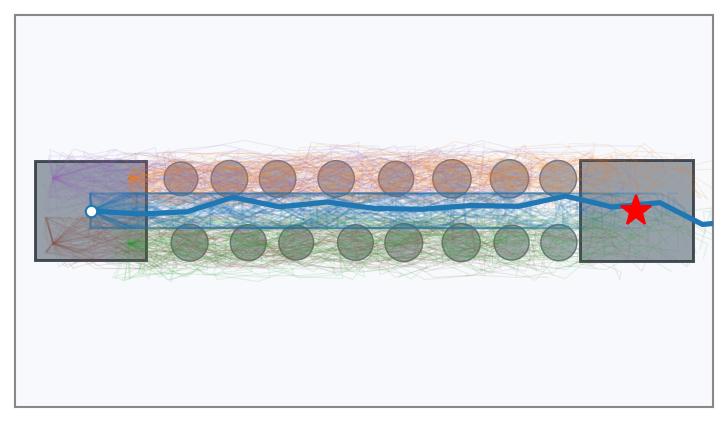}%
            \includegraphics[width=0.333\linewidth]{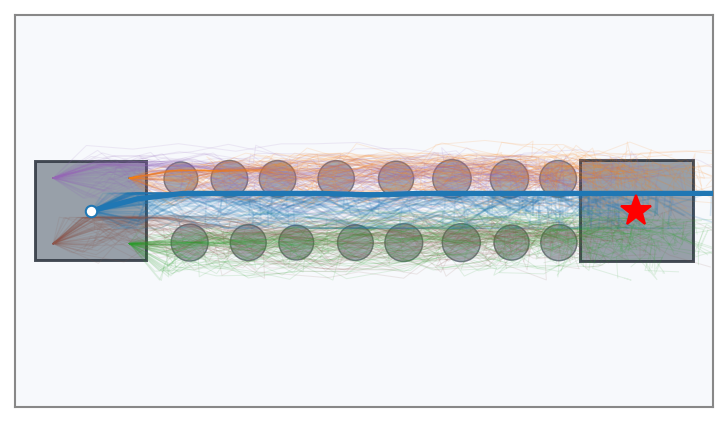}%
            \includegraphics[width=0.333\linewidth]{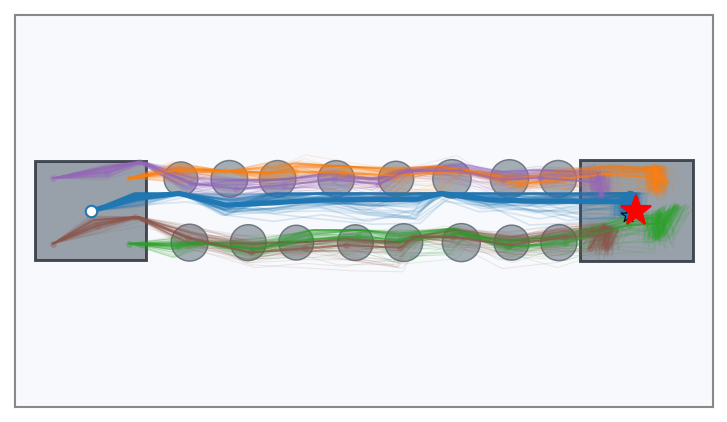}%
        \end{minipage}}%
    }%
    \par\vspace{3pt}
    \resizebox{\textwidth}{!}{%
        \steppingalgobox{MPPI}{\includegraphics[width=\linewidth]{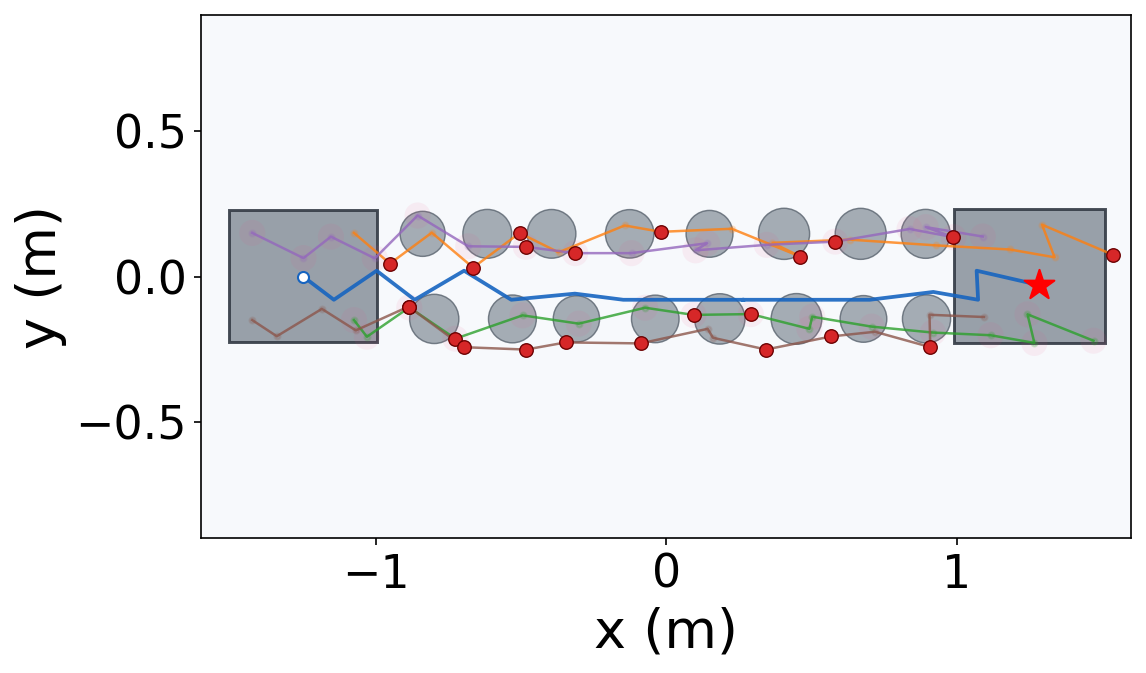}}{\includegraphics[width=\linewidth]{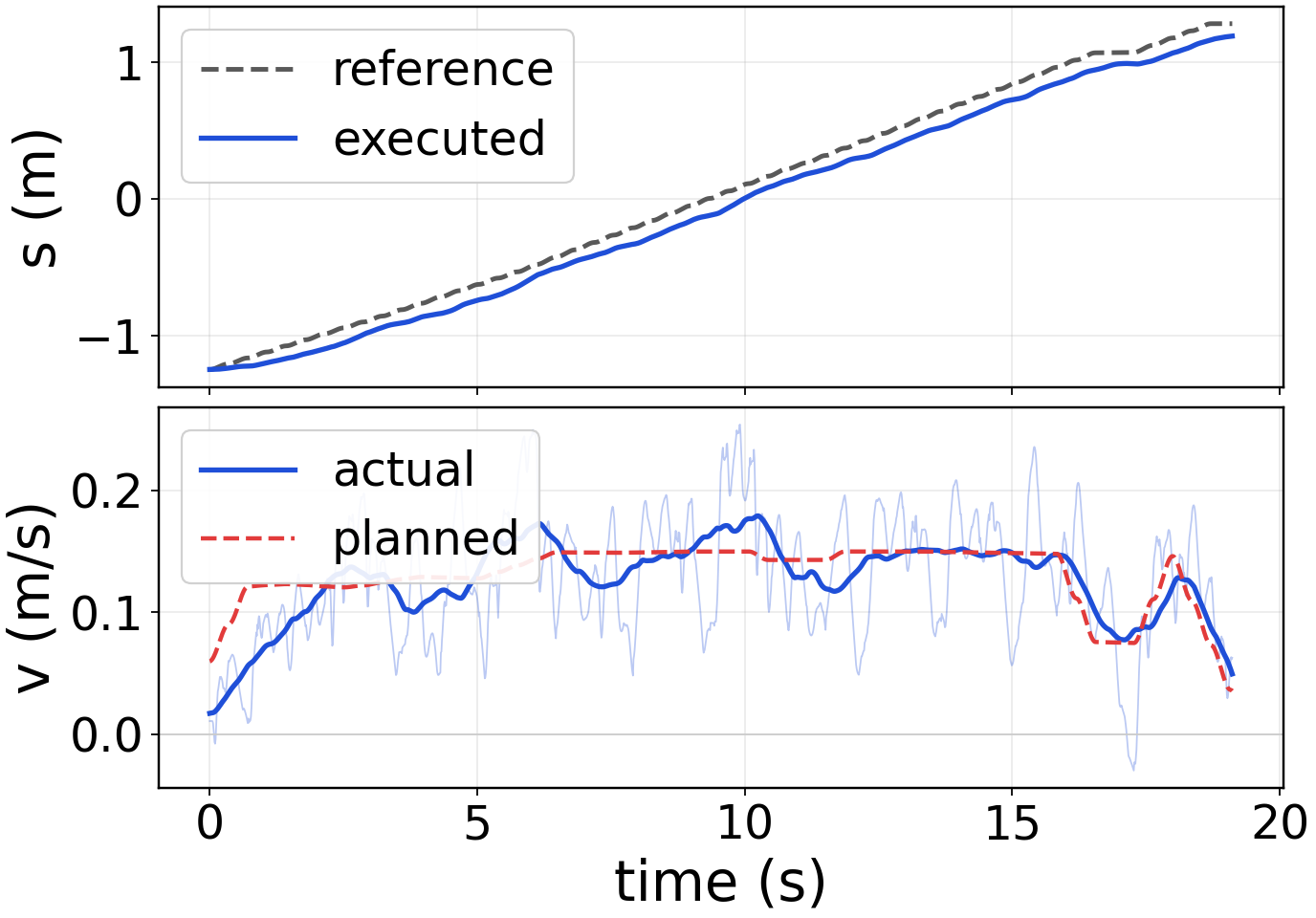}}\hspace{\stpgap}%
        \steppingalgobox{MBD}{\includegraphics[width=\linewidth]{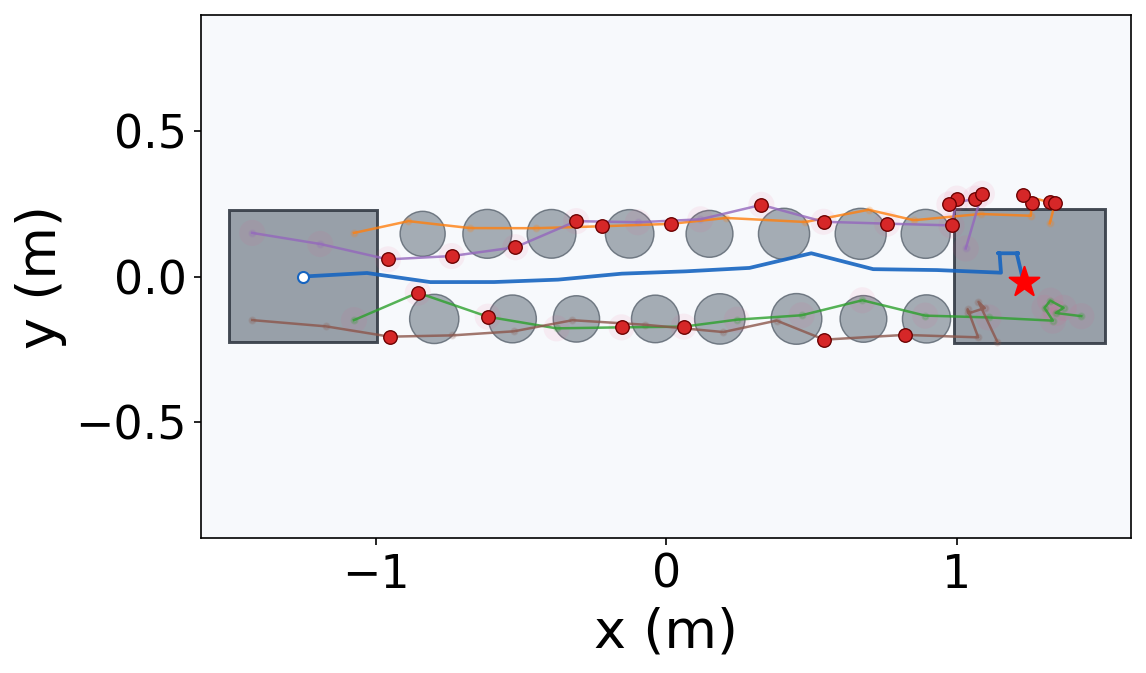}}{\includegraphics[width=\linewidth]{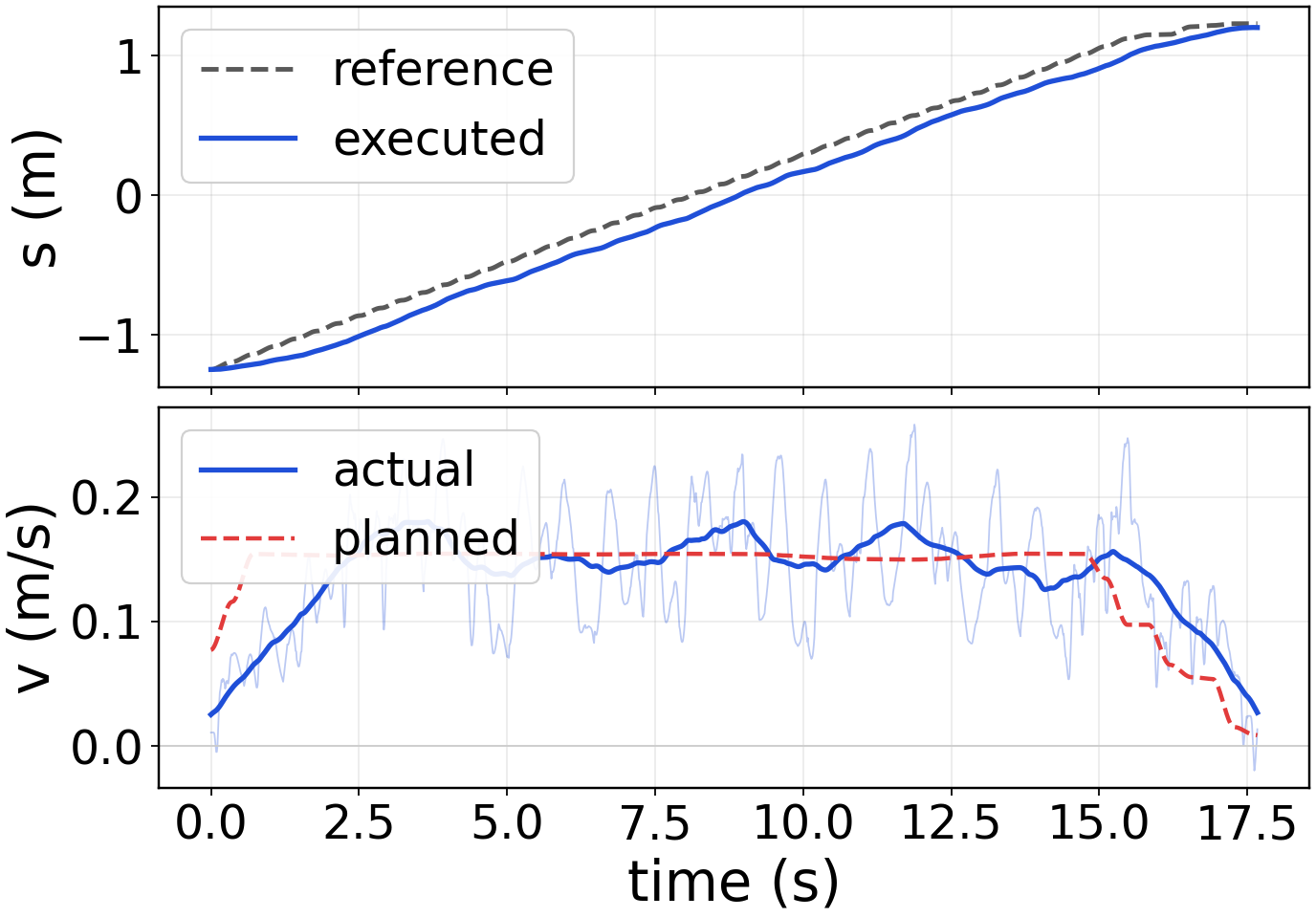}}\hspace{\stpgap}%
        \steppingalgobox{MDOC}{\includegraphics[width=\linewidth]{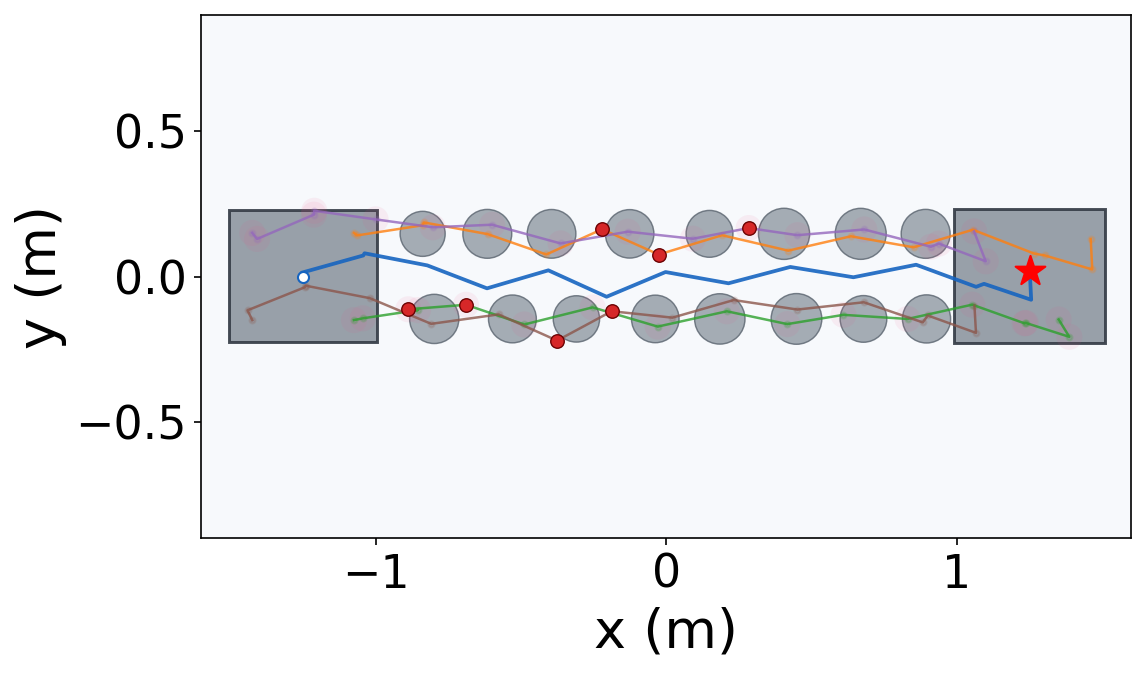}}{\includegraphics[width=\linewidth]{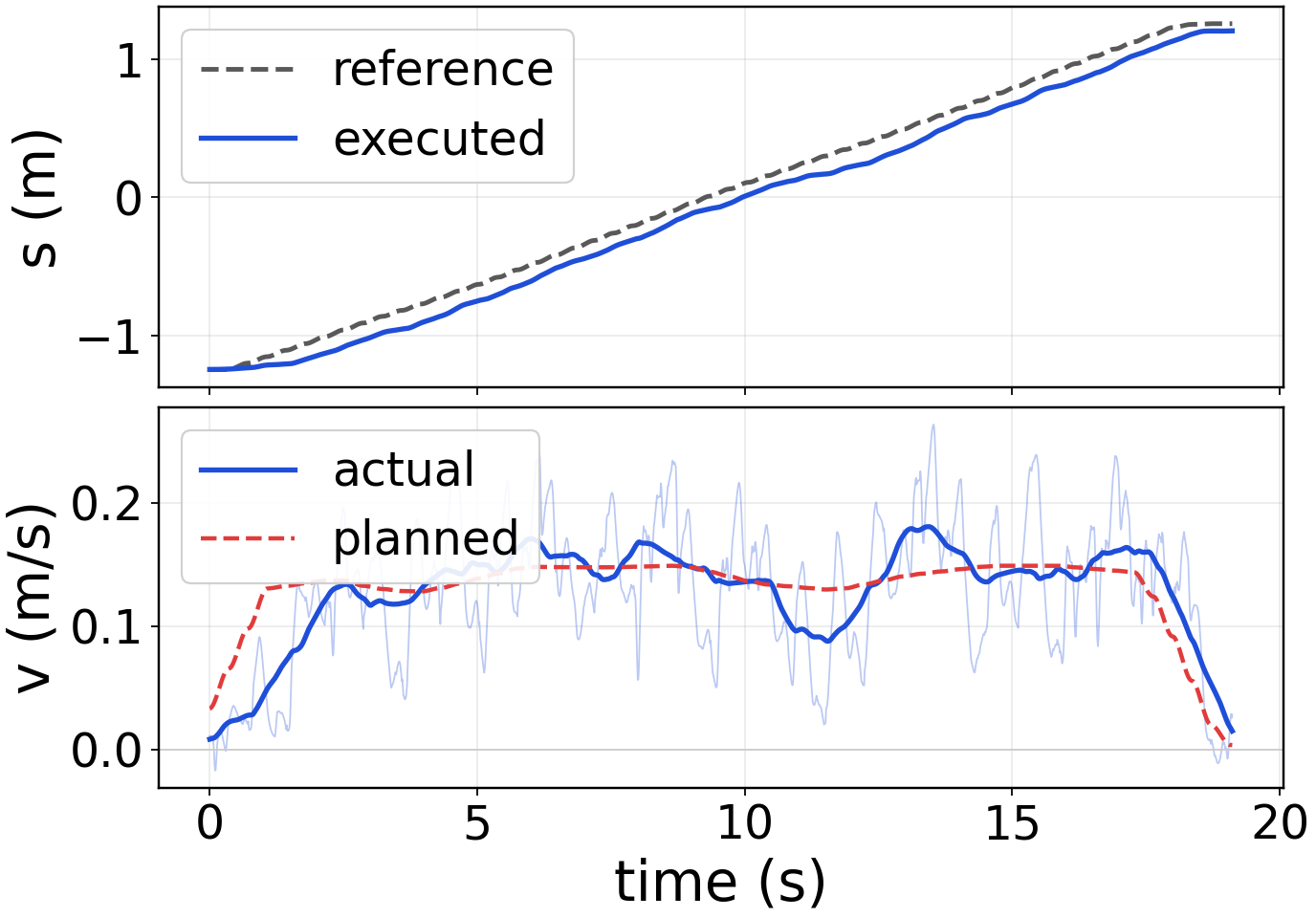}}\hspace{\stpgap}%
        \steppingalgobox{MD-COAS}{\includegraphics[width=\linewidth]{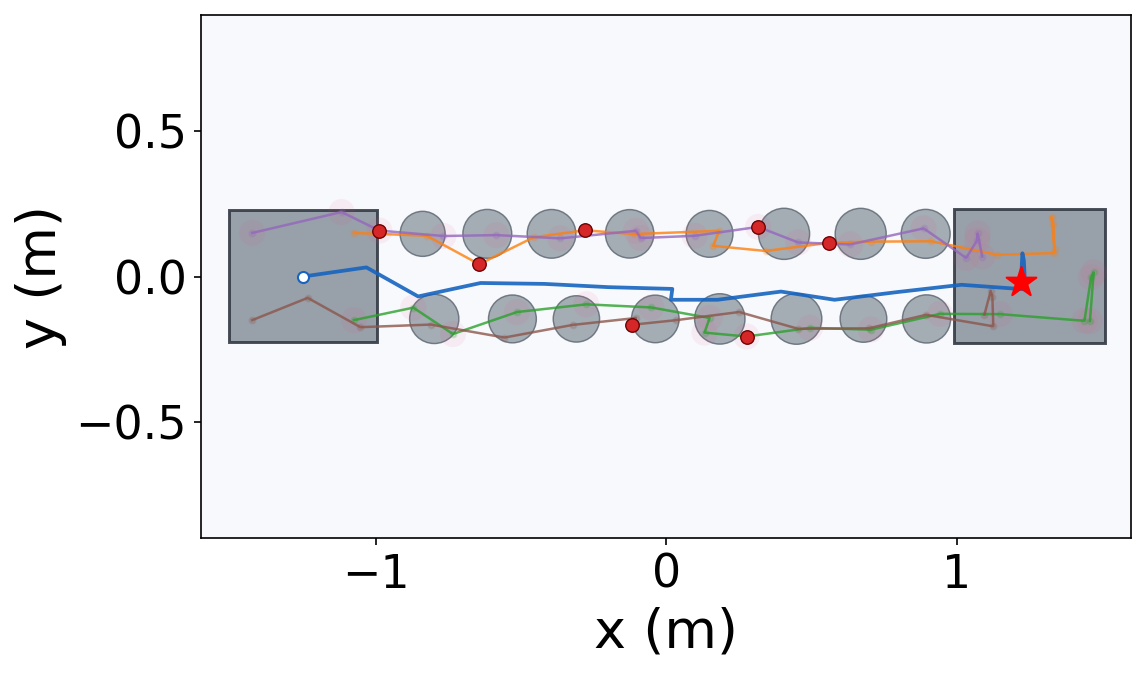}}{\includegraphics[width=\linewidth]{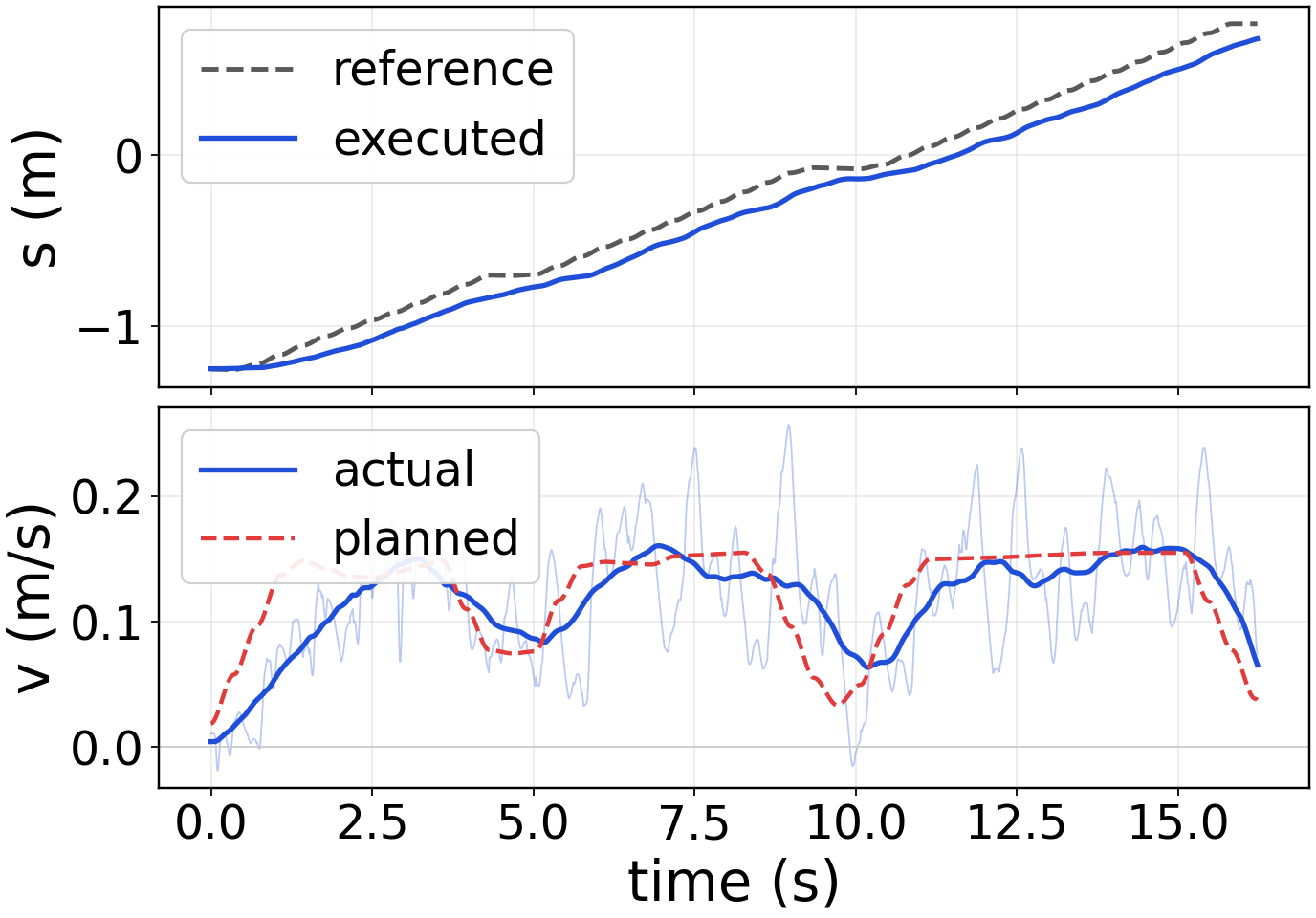}}\hspace{\stpgap}%
        \steppingalgobox{2GO}{\includegraphics[width=\linewidth]{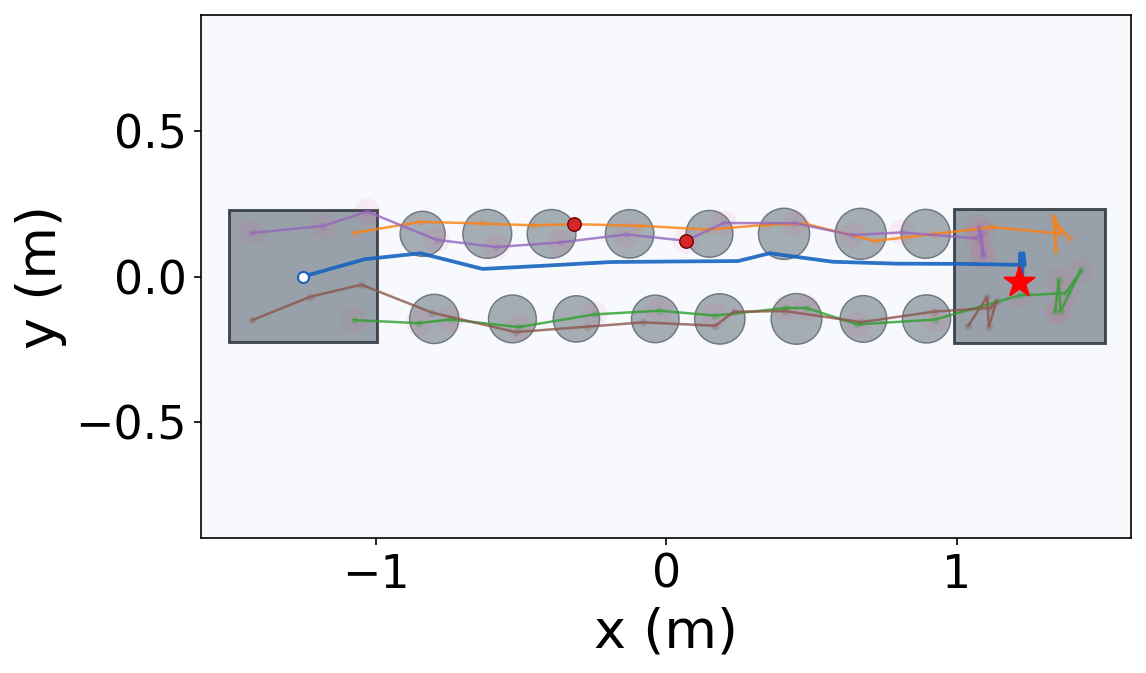}}{\includegraphics[width=\linewidth]{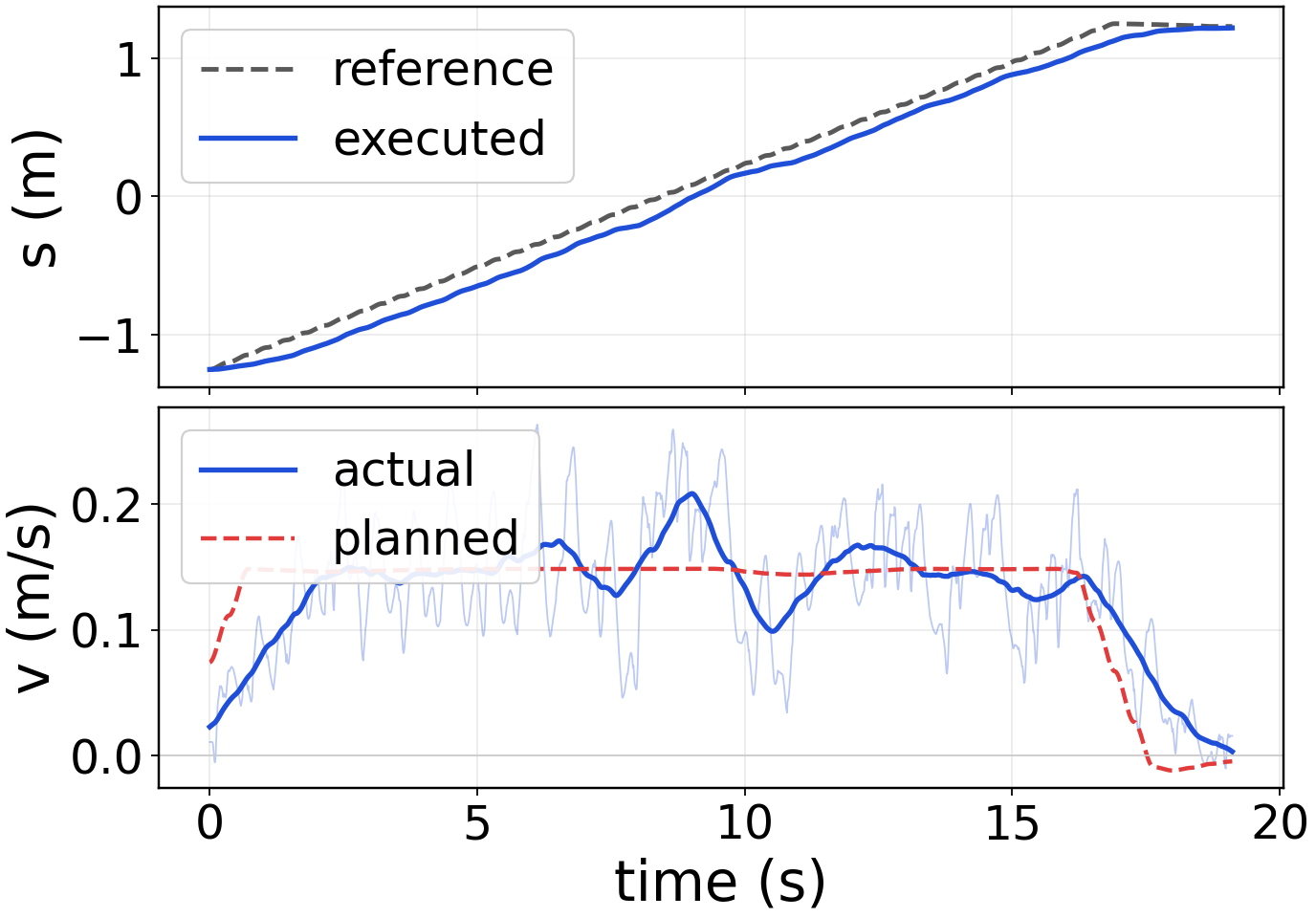}}%
    }%
    \endgroup
\end{figure*}

\begin{table*}[tb!]
\centering
\footnotesize
\setlength{\tabcolsep}{3pt}
\renewcommand{\arraystretch}{1}
\resizebox{\textwidth}{!}{%
\begin{tabular}{c|cccc|cccc|cccc|cccc|cccc}
\toprule
\multirow{2}{*}{Env.} & \multicolumn{4}{c|}{MPPI} & \multicolumn{4}{c|}{MBD} & \multicolumn{4}{c|}{MDOC} & \multicolumn{4}{c|}{MD-COAS} & \multicolumn{4}{c}{2GO} \\
 & pSSR & $pV$ & $pT$ & eSSR & pSSR & $pV$ & $pT$ & eSSR & pSSR & $pV$ & $pT$ & eSSR & pSSR & $pV$ & $pT$ & eSSR & pSSR & $pV$ & $pT$ & eSSR \\
\midrule
S1 & $0.40$ & $2.20^{0.14}$ & $\mathbf{2.87}^{0.34}$ & $0.20$ & $0.18$ & $3.15^{0.36}$ & $\underline{17.85}^{0.75}$ & $\underline{0.60}$ & $\underline{0.88}$ & $1.00^{0.08}$ & $18.37^{0.17}$ & $\mathbf{0.80}$ & $0.32$ & $\underline{0.85}^{0.22}$ & $288.22^{33.76}$ & $0.00$ & $\mathbf{0.98}$ & $\mathbf{0.47}^{0.14}$ & $21.51^{1.01}$ & $\mathbf{0.80}$ \\
S2 & $0.20$ & $2.75^{0.24}$ & $\mathbf{2.53}^{0.23}$ & $0.00$ & $0.04$ & $3.01^{0.30}$ & $\underline{18.35}^{3.63}$ & $\mathbf{0.40}$ & $\underline{0.68}$ & $\underline{1.11}^{0.12}$ & $19.73^{0.33}$ & $\underline{0.20}$ & $0.14$ & $1.45^{0.28}$ & $99.78^{42.47}$ & $0.00$ & $\mathbf{0.88}$ & $\mathbf{0.61}^{0.25}$ & $23.19^{2.74}$ & $\mathbf{0.40}$ \\
\midrule
ZA & $0.00$ & $1.15^{0.40}$ & $\mathbf{1.40}^{0.31}$ & $\underline{0.40}$ & $0.00$ & $1.16^{0.21}$ & $\underline{10.16}^{1.76}$ & $0.00$ & $\mathbf{0.96}$ & $\mathbf{0.00}^{0.00}$ & $18.04^{10.05}$ & $\mathbf{1.00}$ & $\underline{0.02}$ & $\underline{0.07}^{0.02}$ & $67.27^{18.94}$ & $0.20$ & $\mathbf{0.96}$ & $\mathbf{0.00}^{0.00}$ & $14.08^{6.91}$ & $\mathbf{1.00}$ \\
ZB & $0.00$ & $3.37^{0.58}$ & $\mathbf{1.37}^{0.36}$ & $0.00$ & $0.00$ & $2.78^{0.13}$ & $\underline{6.15}^{1.92}$ & $0.00$ & $\mathbf{0.90}$ & $\underline{0.77}^{0.70}$ & $39.16^{9.89}$ & $\underline{0.20}$ & $0.40$ & $\mathbf{0.18}^{0.13}$ & $99.79^{12.04}$ & $\mathbf{0.60}$ & $\underline{0.48}$ & $1.82^{0.31}$ & $37.09^{18.20}$ & $\underline{0.20}$ \\
ZC & $0.00$ & $1.68^{1.33}$ & $\mathbf{1.36}^{0.32}$ & $0.00$ & $0.00$ & $3.12^{0.80}$ & $\underline{4.88}^{1.00}$ & $0.00$ & $\mathbf{1.00}$ & $\mathbf{0.00}^{0.00}$ & $10.94^{5.95}$ & $\mathbf{1.00}$ & $\underline{0.02}$ & $\underline{0.22}^{0.09}$ & $105.30^{3.58}$ & $0.20$ & $\mathbf{1.00}$ & $1.23^{0.16}$ & $14.23^{7.98}$ & $\mathbf{1.00}$ \\
ZD & $0.00$ & $0.45^{0.31}$ & $\mathbf{1.35}^{0.33}$ & $0.00$ & $0.00$ & $0.56^{0.06}$ & $10.77^{12.89}$ & $0.00$ & $\underline{0.86}$ & $0.08^{0.02}$ & $\underline{10.14}^{4.19}$ & $\underline{0.40}$ & $0.24$ & $\underline{0.07}^{0.03}$ & $39.45^{7.07}$ & $\mathbf{0.60}$ & $\mathbf{0.98}$ & $\mathbf{0.00}^{0.00}$ & $11.34^{6.95}$ & $\mathbf{0.60}$ \\
\bottomrule
\end{tabular}}
\caption{Planning and execution performance across stepping stones (S1, S2) and corridors (ZA--ZD). pSSR and eSSR ($\uparrow$) are success fractions defined in the evaluation protocol; $pV$ ($\downarrow$) is foothold distance (cm) for S1/S2 and sampled obstacle occupancy (\%) for ZA--ZD. $pT$ ($\downarrow$) is planning time (s). Superscripts give population standard deviations across 5 seeds, with each seed generating 10 clean trajectories. Bold and underlining indicate the best and second-best displayed values, including ties.}
\label{tab:main_results}
\vspace{-0.4cm}
\end{table*}

\section{Experimental Analysis}
We evaluate 2GO on CLP to address following questions:

\begin{itemize}
    \item \textbf{Q1:} Does 2GO improve planning over representative baselines in sparse stepping-stone locomotion involving \textbf{discrete foothold choices}?
    \item \textbf{Q2:} Can 2GO improve planning and execution in humanoid navigation with \textbf{continuous body and posture references} under geometric obstacle constraints?
    \item \textbf{Q3:} How do geometry-aware updates and adaptive reverse stochasticity contribute to performance?
\end{itemize}

\textbf{Baselines.}
We compare 2GO with
(\textbf{1}) MPPI~\cite{williams2017mppi}, which importance-weights sampled forward rollouts to optimize controls;
(\textbf{2}) MBD~\cite{pan2024model}, a model-based diffusion planner;
(\textbf{3}) MDOC~\cite{he2026model}, which adds CBF-based projection to each MBD rollout step;
(\textbf{4}) MD-COAS~\cite{he2026motion}, which combines an iALM soft feasibility prior, CFS projection, and adaptive scheduling. The main comparison uses a 48GB NVIDIA RTX A6000 GPU with 9 CPU cores.

\textbf{Evaluation.} Planning Safe-and-Success Rate (pSSR) evaluates all returned candidates,
while supplementary execution SSR (eSSR) evaluates one selected, adapted reference per seed across all 5 trials, including failed planning attempts. Both use a $0.20$ m fixed-goal tolerance; quadruped planning uses margin-shrunken supports, execution allows $0.05$ m swing-end foot error relative to physical supports, and humanoid execution requires
no fall and nonnegative body-proxy clearance. Candidate selection and governor/controller adaptation can recover planning-level violations,
allowing eSSR to exceed pSSR.

\begin{figure*}[tb!]
    \centering
    \begingroup
    \setlength{\tabcolsep}{0pt}
    \setlength{\fboxsep}{0pt}
    \renewcommand{\arraystretch}{0}
    \newcommand{\twogocorridorimage}[2]{%
        \begin{minipage}[c]{#1}
            \centering
            #2
        \end{minipage}%
    }
    \newcommand{\twogocorridorblank}[1]{%
        \begin{minipage}[c]{#1}
            \centering
            \rule{0pt}{0.78in}\rule{\linewidth}{0pt}
        \end{minipage}%
    }
    \newcommand{\twogozonelabel}[1]{%
        \begin{minipage}[c]{0.025\textwidth}
            \centering
            \rotatebox[origin=c]{90}{\textbf{#1}}%
        \end{minipage}%
    }

    \begin{tabular}{@{}c@{}c@{}c@{}c@{}c@{}c@{}}
        \vspace{0.05cm}
        \twogozonelabel{Zone A} &
        \twogocorridorimage{0.250\textwidth}{\includegraphics[width=\linewidth]{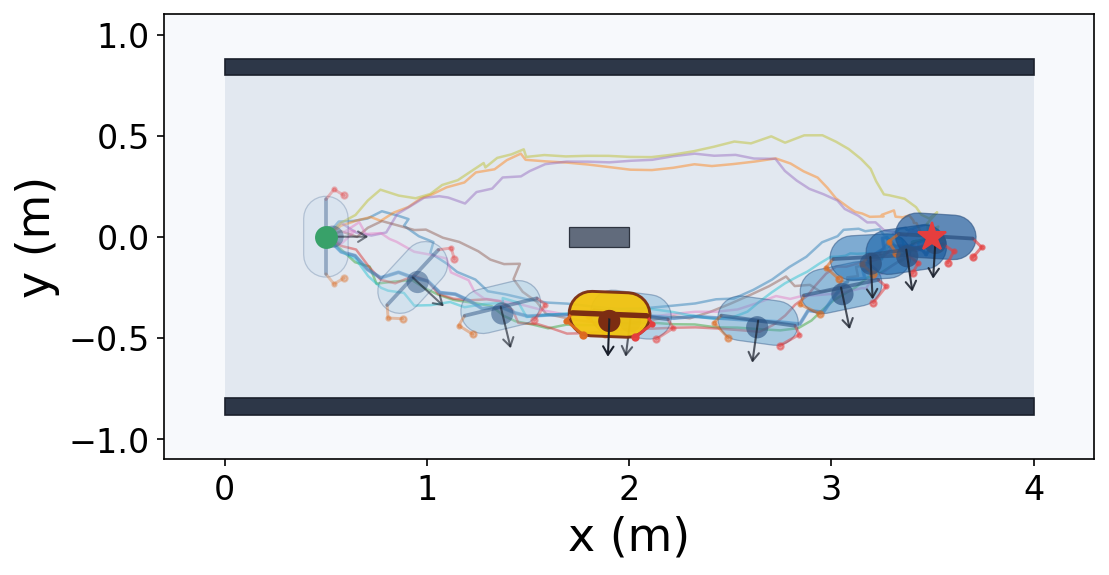}} &
        \twogocorridorimage{0.250\textwidth}{\includegraphics[width=\linewidth]{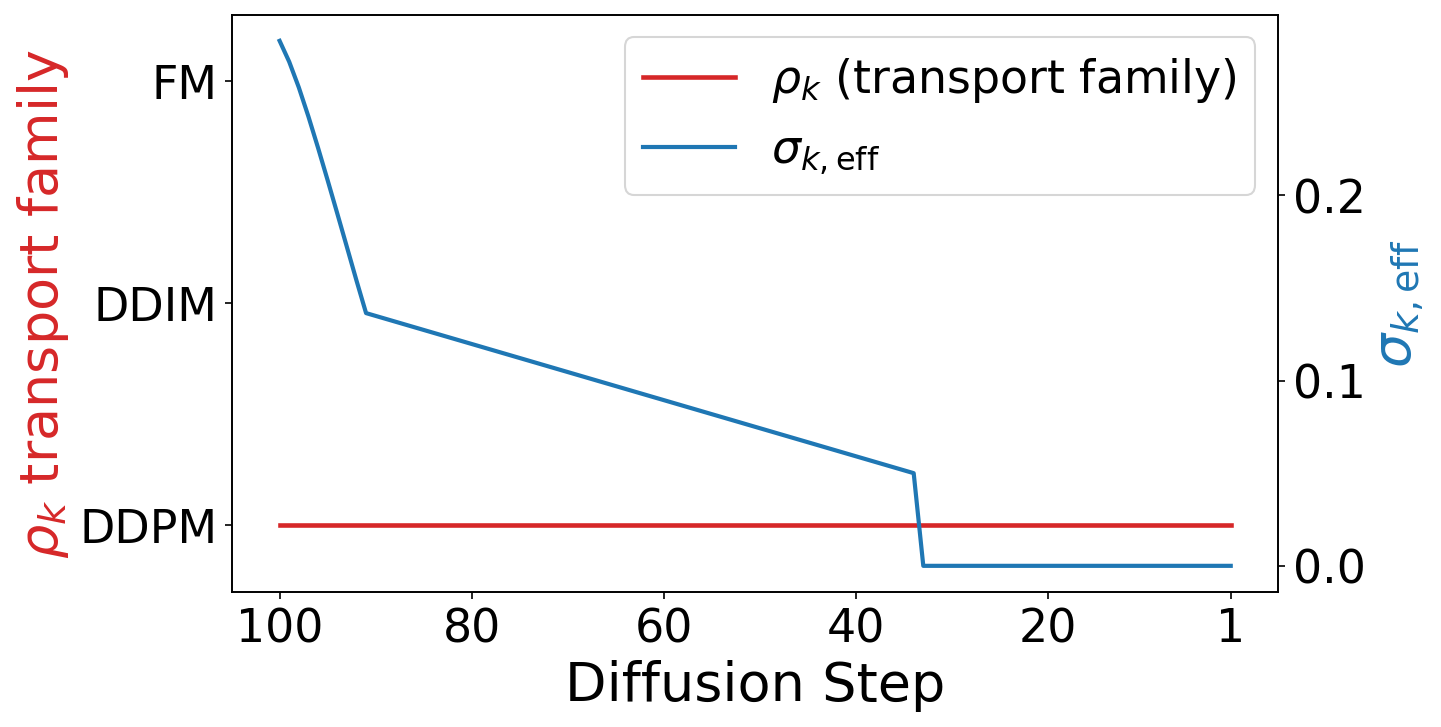}} &
        \multicolumn{3}{@{}c@{}}{\twogocorridorimage{0.450\textwidth}{\includegraphics[width=\linewidth]{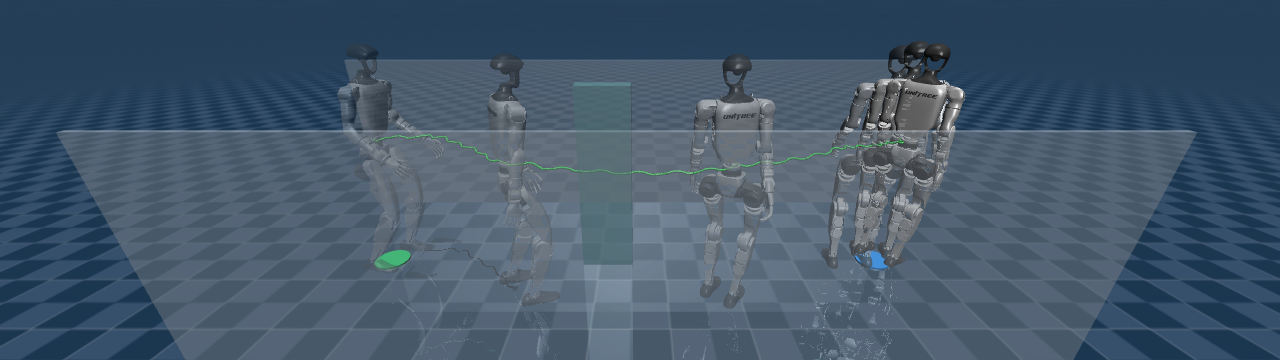}}} \\
        \vspace{0.05cm}
        \twogozonelabel{Zone B} &
        \twogocorridorimage{0.250\textwidth}{\includegraphics[width=\linewidth]{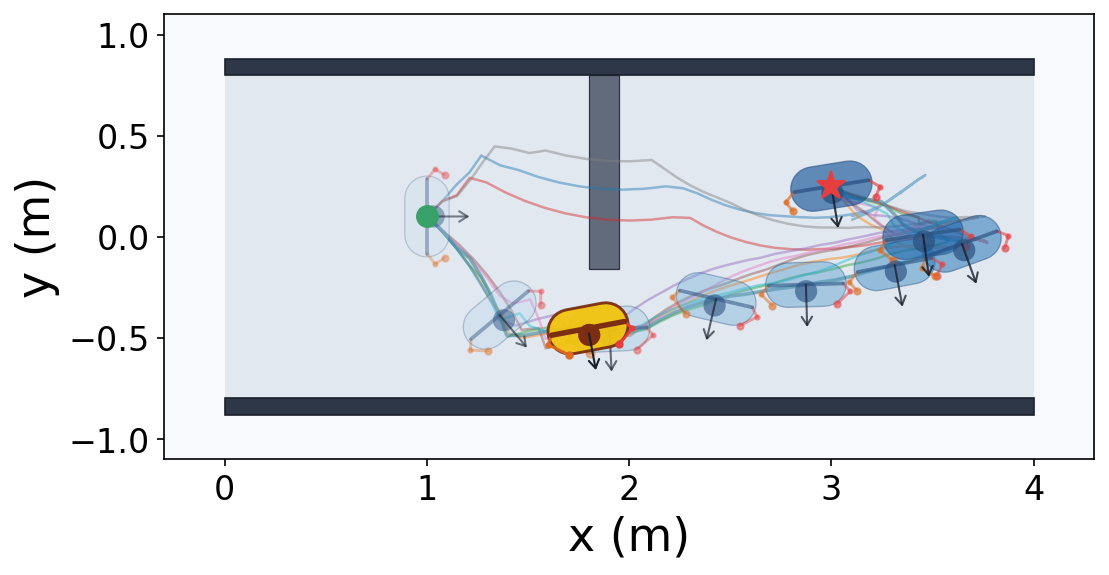}} &
        \twogocorridorimage{0.250\textwidth}{\includegraphics[width=\linewidth]{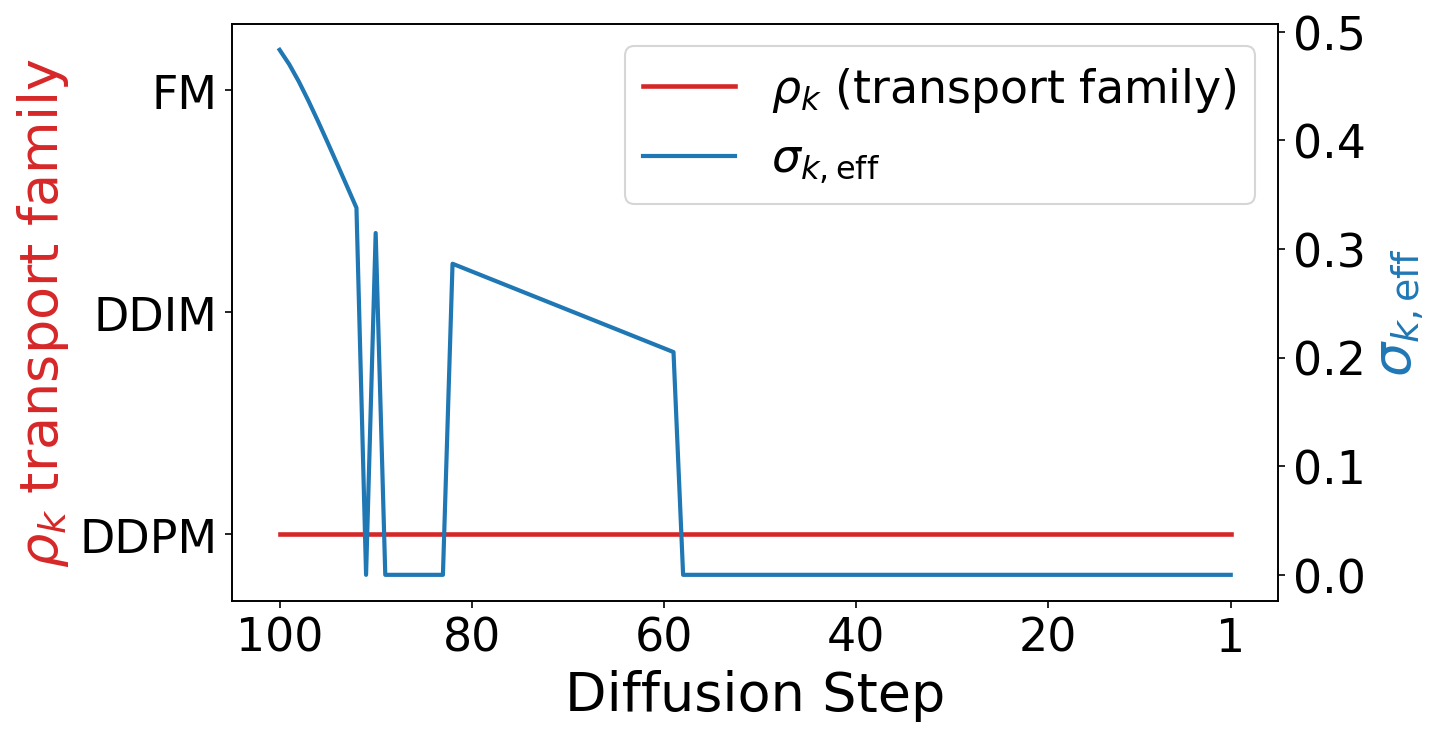}} &
        \multicolumn{3}{@{}c@{}}{\twogocorridorimage{0.450\textwidth}{\includegraphics[width=\linewidth]{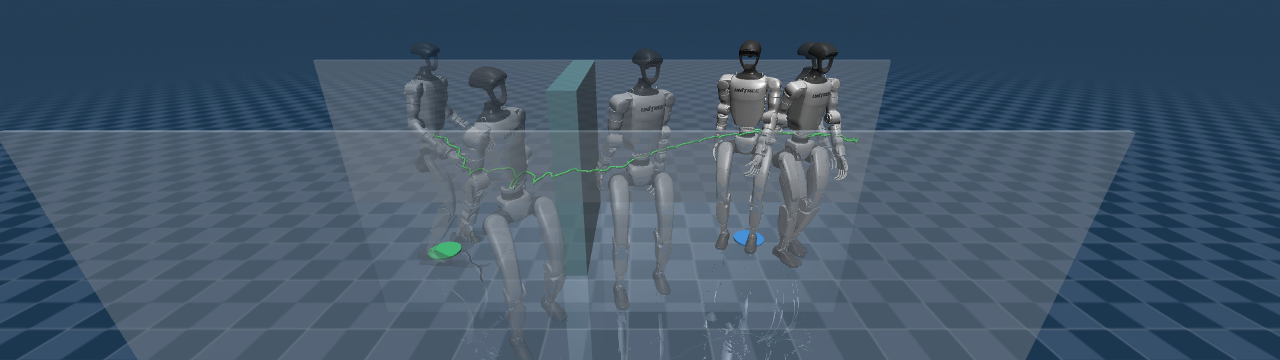}}} \\
        \vspace{0.05cm}
        \twogozonelabel{Zone C} &
        \twogocorridorimage{0.250\textwidth}{\includegraphics[width=\linewidth]{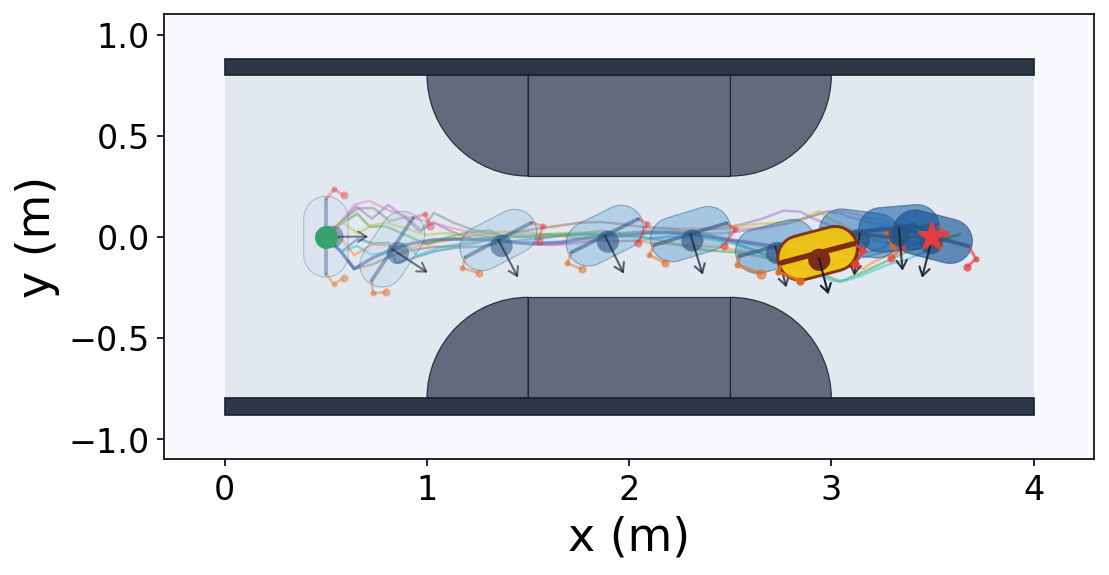}} &
        \twogocorridorimage{0.250\textwidth}{\includegraphics[width=\linewidth]{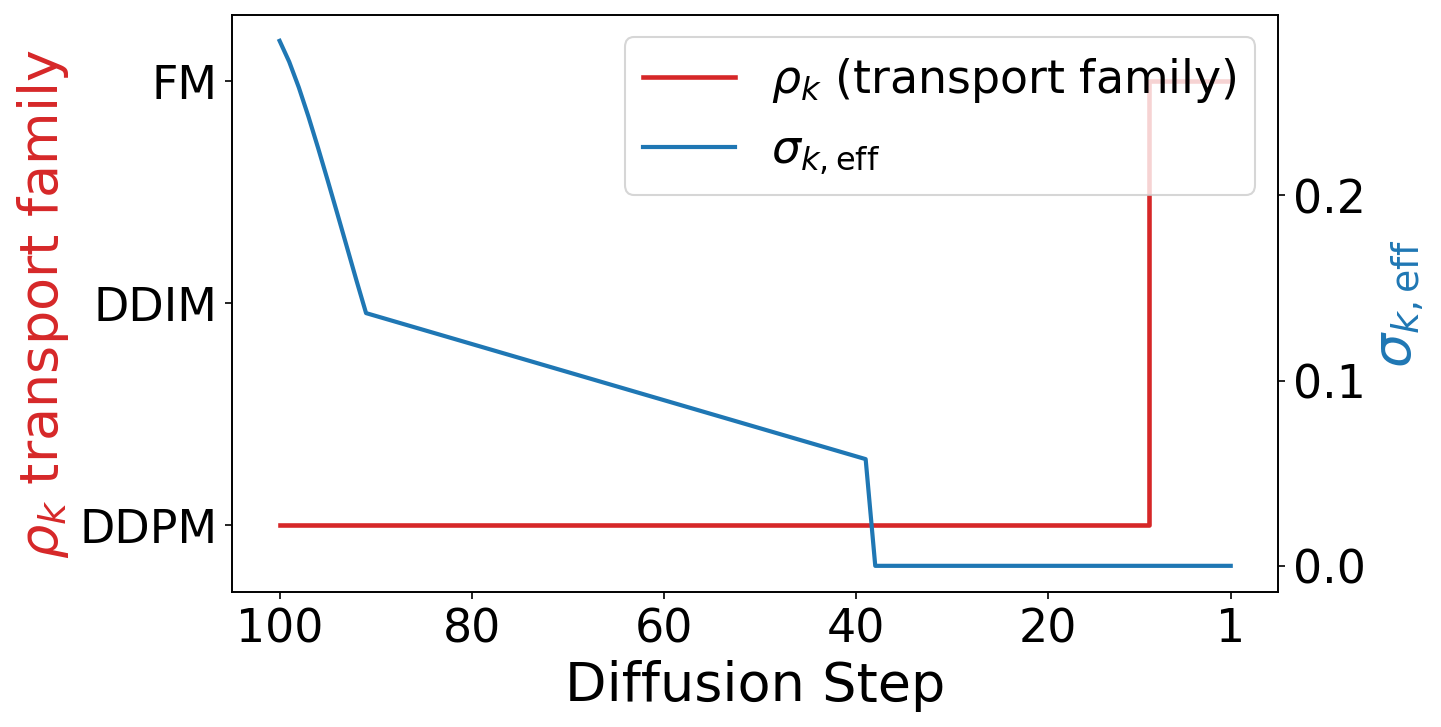}} &
        \multicolumn{3}{@{}c@{}}{\twogocorridorimage{0.450\textwidth}{\includegraphics[width=\linewidth]{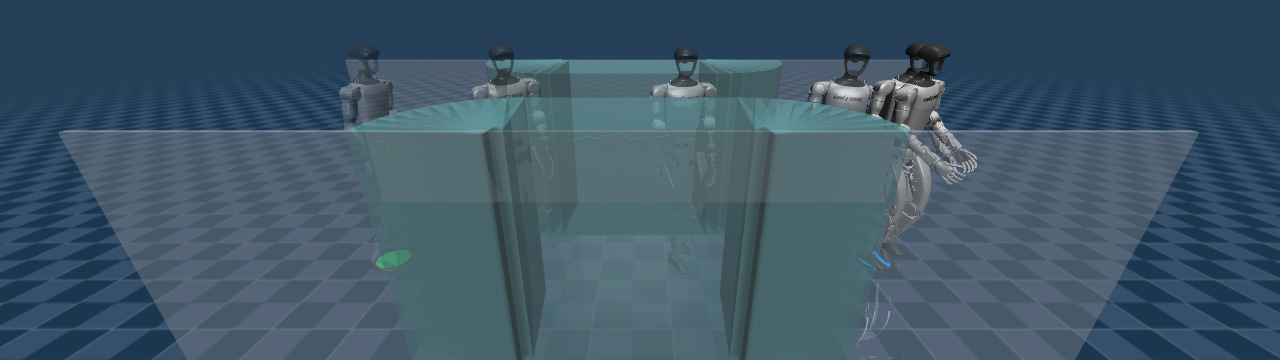}}} \\
        \vspace{0.05cm}
        \twogozonelabel{Zone D} &
        \twogocorridorimage{0.250\textwidth}{\includegraphics[width=\linewidth]{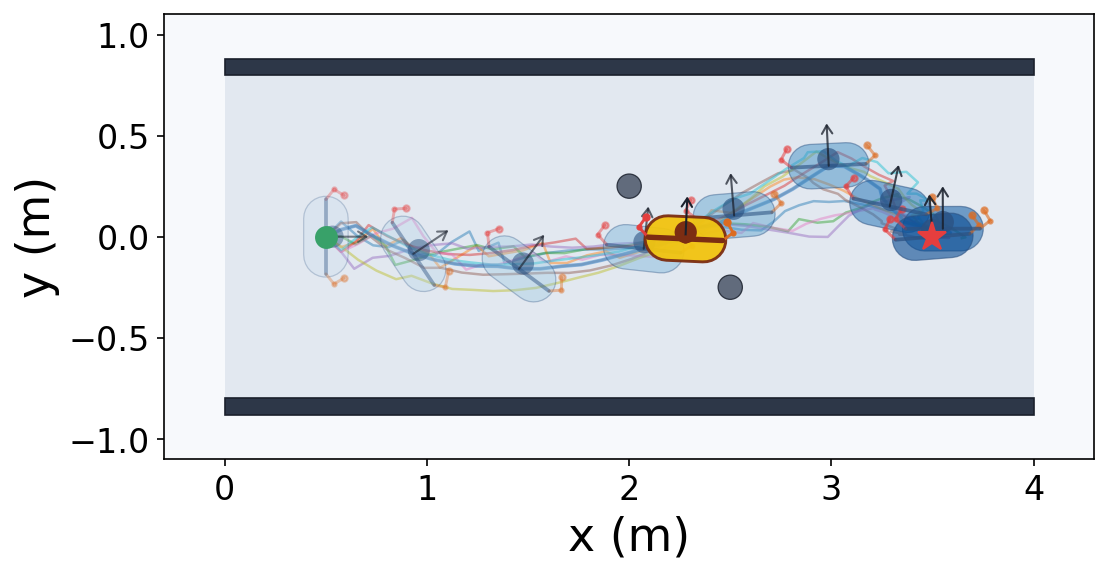}} &
        \twogocorridorimage{0.250\textwidth}{\includegraphics[width=\linewidth]{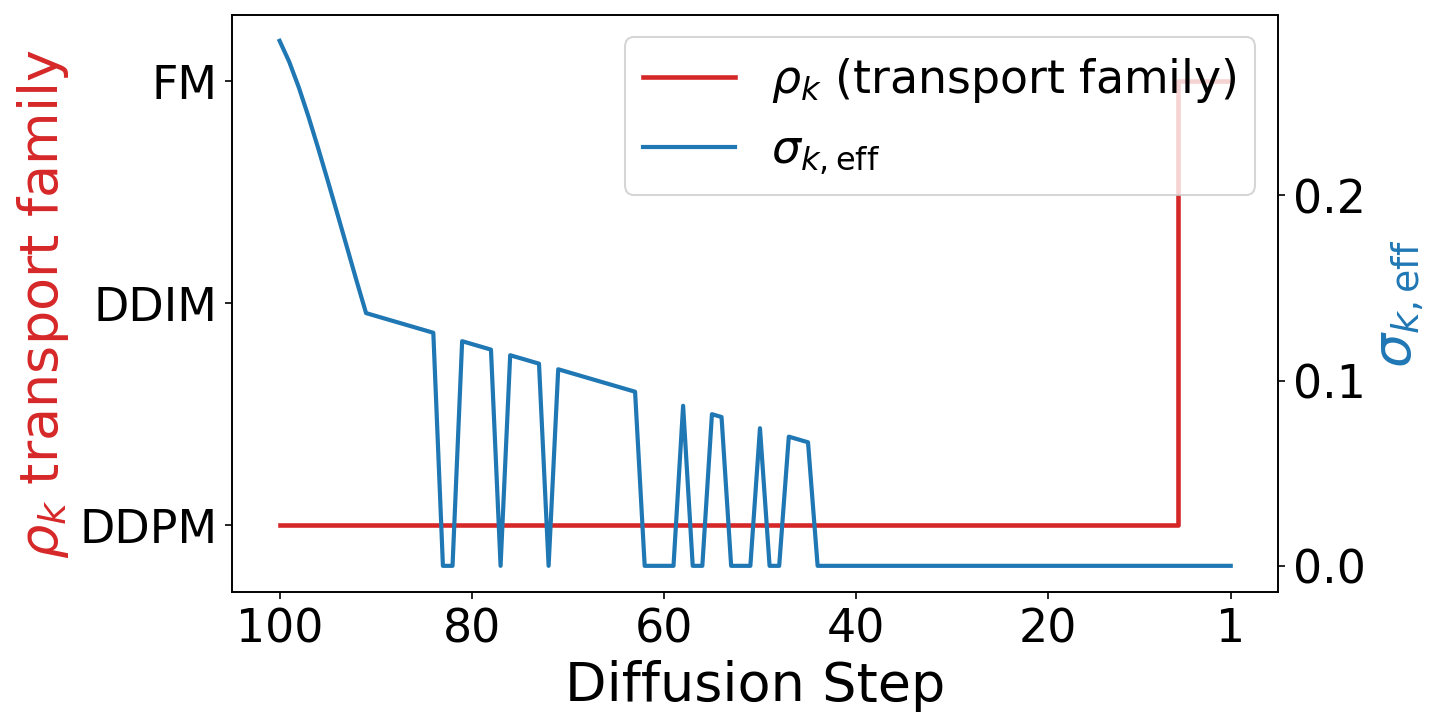}} &
        \multicolumn{3}{@{}c@{}}{\twogocorridorimage{0.450\textwidth}{\includegraphics[width=\linewidth]{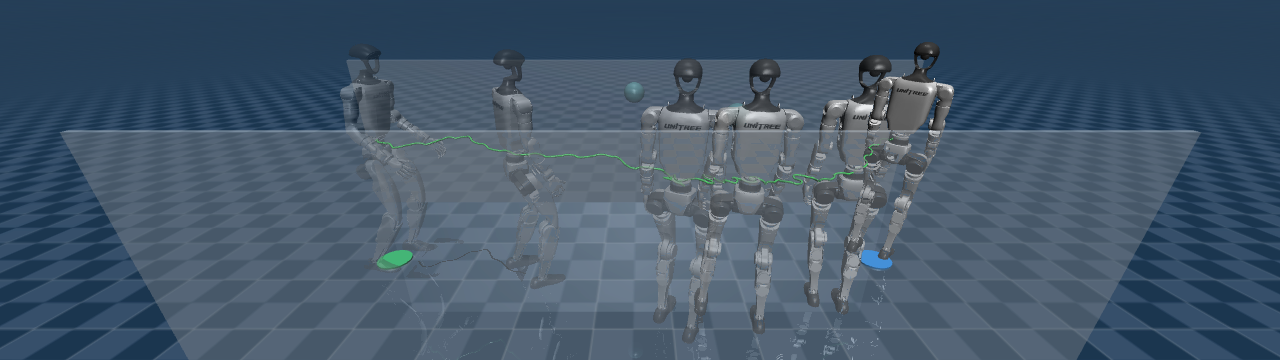}}}
    \end{tabular}
    \caption{Planning and deployment across the humanoid corridor locomotion benchmark. Per zone (rows): the planned multi-modal whole-body latent trajectories (left), the 2GO adaptive transport schedule $\rho_k/\sigma_{k,\mathrm{eff}}^{\rho_k}$ within fixed 100 diffusion steps (middle), and the deployed Unitree G1 sport-mode execution as a ghosted motion strip (right).}
    \label{fig:corridor}
    \endgroup
    \vspace{-0.4cm}
\end{figure*}

\vspace{-0.2cm}
\subsection{Quadruped Stepping-Stone Locomotion}
\vspace{-0.1cm}
We evaluate quadruped locomotion across stepping stones from
$(-1.25,0)$ to $(1.25,0)$. S1 uses stone radii of
$0.076$--$0.088$ m and nominal along-lane gaps of
$0.22$--$0.28$ m; S2 uses smaller radii of
$0.068$--$0.080$ m and larger gaps of $0.26$--$0.32$ m.
The 16-state reference model includes planar body position and
velocity, yaw and yaw rate, four two-dimensional foot offsets,
and support mode and phase. Its 12 controls adjust planar body
motion, yaw, foot offsets, and phase. Bounded Euler updates
keep the body near the support-foot midpoint.
World-frame foot positions satisfy
\begin{math}
    p_\ell=p_b+R(\psi)(p_\ell^{\mathrm{nom}}+r_\ell),
\label{eq:quad_model_main}
\end{math}
where $p_b$ is body position, $R(\psi)$ the yaw rotation, and
$r_\ell$ a foot offset from the nominal body-frame position
$p_\ell^{\mathrm{nom}}$.
Foothold violation is the distance to the union of margin-shrunken
stone disks and support rectangles. Diagonal support pairs
alternate at phase completion only if swing-foot violation is
at most $0.04$ m. The governor constructs a walking gait and
reachable touchdowns for the follower with auxiliary
base stabilization.

The planning results show that 2GO achieves the highest pSSR across both S1 and S2. The reverse snapshots illustrate this refinement: early steps preserve multiple foothold candidates, while late steps concentrate onto contact-consistent routes rather than crossing unsupported regions. MPPI is the fastest baseline and improves pSSR over MBD, but its cost-weighted rollouts still incur high foothold violation and low eSSR, consistent with unresolved discrete support assignments. MBD likewise often selects invalid contacts in goal-directed rollouts even though the follower recovers some execution trials. MDOC and MD-COAS impose stronger corrections, yet their continuous projection geometry remains local around the current sample and can restrict exploration when stones are disconnected. In contrast, 2GO uses task-modulated active foothold geometry, tangent-space exploration, and transport-aware refinement to coordinate refinement and exploration across feasible foothold regions.

\vspace{-0.2cm}
\subsection{Humanoid Corridor Navigation}
\vspace{-0.2cm}

We evaluate humanoid navigation in a $1.6$ m-wide corridor containing
a thin center wall (Zone A), a U-shaped wall (B), a tapered squeeze (C),
and spherical obstacles (D). The 14-state reference model describes
planar pose, body height, torso yaw, arm configuration, and motion rates.
Its nine controls command body-frame velocity $v^b$, yaw rate
$\omega_c$, and posture rates $u_q$. Position $p$, yaw $\psi$, and
posture $q$ evolve as
\begin{equation}
\begin{aligned}
p^+&=p+\Delta t R(\psi)v^b,\quad
\psi^+=\operatorname{wrap}(\psi+\Delta t\omega_c),\\
q^+&=\operatorname{clip}(q+\Delta t u_q,q_{\min},q_{\max}),
\end{aligned}
\label{eq:humanoid_model_main}
\end{equation}
where $R(\psi)$ rotates body-frame velocities into the world frame.
The governor adapts the resulting references to command limits and
clearance margins before low-level tracking.

Table~\ref{tab:main_results} reports corridor results across four isolated zones. Compared with stepping stones, the difficulty shifts from discrete foothold selection to continuous body-wide clearance: reaching the goal may still fail SSR through torso or arm collision. MPPI has the lowest planning time, but its pSSR is zero in all four zones; importance-weighted rollouts retain goal progress without consistently satisfying body-wide clearance. MBD is slower and likewise has zero pSSR and eSSR in all four zones. Although MDOC is competitive in Zone B pSSR, 2GO achieves higher SSR and eSSR in Zone D. MD-COAS recovers safe best trajectories, but has much longer planning time and lower pSSR. In contrast, 2GO reaches full or near-full pSSR in Zones A, C, and D. Its task-modulated collision geometry shapes the reverse direction and tangent exploration for side-passing and posture-adjusted motion.

\textbf{Real-world deployment.}
The Unitree G1 demonstrations in Fig.~\ref{fig:2go_deploy} use Augmented Reality (AR) to register virtual obstacles in a common scene frame. Wall (Light Gray) and wall obstacles (Dark Green) are shorten for better virtualization. The hardware interface converts adapted references into body-frame planar velocity, yaw-rate, and height requests for Unitree’s built-in locomotion controller; torso and arm variables define joint-position references. State feedback, reference retiming, and clearance adaptation accommodate the planning–tracking mismatch. These qualitative demonstrations complement the simulation evaluation.

\vspace{-0.2cm}
\subsection{Discussion}
\vspace{-0.2cm}
\textbf{Why Geometric-Aware?}
Table~\ref{tab:main_results} shows favorable planning safety and execution
compatibility for 2GO. Within each reverse update, 2GO uses shared active geometry to shape the denoising
direction, filter noise toward the tangent space, and apply
optional CFS correction. This design supports exploration among disconnected
foothold regions and clearance-aware refinement in continuous corridors.
The results support the combined method across both environments.

\textbf{Why Adaptive Reverse Stochasticity?}
Useful stochasticity can change with geometry: ambiguous steps may benefit
from exploration, whereas locally resolved steps favor deterministic
refinement. The schedule
$\rho_k/\sigma_{k,\mathrm{eff}}^{\rho_k}$ in Fig.~\ref{fig:corridor}
is consistent with this pattern during side selection in Zone A,
entrance and exit maneuvers in Zone B, and passage through the squeeze
and clutter in Zones C and D. These observations illustrate how 2GO
balances exploration and deterministic refinement.
\vspace{-0.2cm}
\section{Conclusion}
\vspace{-0.2cm}
We presented 2GO, a model-based geometry-aware generative
optimization framework for constrained locomotion. 2GO uses task-modulated
active constraint geometry to shape denoising directions, attenuate normal
stochasticity, and define local CFS corrections. Its transport interface decouples
deterministic updates from stochastic exploration, while a reference governor
adapts the planned motion for execution. 
The results show improved foothold feasibility for quadruped and high planning success in
humanoid corridor zones, while execution depends on the reference
interface. Future work includes richer whole-body dynamics and receding-horizon
deployment.

\bibliographystyle{IEEEtran}
\bibliography{references}

@inproceedings{li2025continual,
  title={Continual Learning and Lifting of {Koopman} Dynamics for Linear Control of Legged Robots},
  author={Li, Feihan and Abuduweili, Abulikemu and Sun, Yifan and Chen, Rui and Zhao, Weiye and Liu, Changliu},
  booktitle={Learning for Dynamics {\&} Control Conference},
  pages={136--148},
  year={2025}
}

@inproceedings{xue2025full,
  title={Full-order sampling-based mpc for torque-level locomotion control via diffusion-style annealing},
  author={Xue, Haoru and Pan, Chaoyi and Yi, Zeji and Qu, Guannan and Shi, Guanya},
  booktitle={International Conference on Robotics and Automation},
  pages={4974--4981},
  year={2025}
}

@inproceedings{dhedin2025simultaneous,
  title={Simultaneous Contact Sequence and Patch Planning for Dynamic Locomotion},
  author={Dh{\'e}din, Victor and Zhao, Haizhou and Khadiv, Majid},
  booktitle={International Conference on Humanoid Robotics},
  pages={245--252},
  year={2025}
}

@inproceedings{pan2024model,
  title={Model-based diffusion for trajectory optimization},
  author={Pan, Chaoyi and Yi, Zeji and Shi, Guanya and Qu, Guannan},
  booktitle={Advances in Neural Information Processing Systems},
  pages={57914--57943},
  year={2024}
}

@inproceedings{ho2020denoising,
  title={Denoising diffusion probabilistic models},
  author={Ho, Jonathan and Jain, Ajay and Abbeel, Pieter},
  booktitle={Advances in Neural Information Processing Systems},
  pages={6840--6851},
  year={2020}
}

@inproceedings{song2021score,
  title={Score-Based Generative Modeling through Stochastic Differential Equations},
  author={Song, Yang and Sohl-Dickstein, Jascha and Kingma, Diederik P. and Kumar, Abhishek and Ermon, Stefano and Poole, Ben},
  booktitle={International Conference on Learning Representations},
  year={2021}
}

@inproceedings{carvalho2023motion,
  title={Motion planning diffusion: Learning and planning of robot motions with diffusion models},
  author={Carvalho, Joao and Le, An T and Baierl, Mark and Koert, Dorothea and Peters, Jan},
  booktitle={International Conference on Intelligent Robots and Systems},
  pages={1916--1923},
  year={2023}
}

@inproceedings{christopher2024constrained,
  title={Constrained synthesis with projected diffusion models},
  author={Christopher, Jacob K and Baek, Stephen and Fioretto, Nando},
  booktitle={Advances in Neural Information Processing Systems},
  pages={89307--89333},
  year={2024}
}

@inproceedings{xiao2023safediffuser,
  title={Safediffuser: Safe planning with diffusion probabilistic models},
  author={Xiao, Wei and Wang, Tsun-Hsuan and Gan, Chuang and Hasani, Ramin and Lechner, Mathias and Rus, Daniela},
  booktitle={International Conference on Learning Representations},
  year={2023}
}

@inproceedings{ratliff2009chomp,
  title={{CHOMP}: Gradient optimization techniques for efficient motion planning},
  author={Ratliff, Nathan and Zucker, Matt and Bagnell, J Andrew and Srinivasa, Siddhartha},
  booktitle={International Conference on Robotics and Automation},
  pages={489--494},
  year={2009}
}

@inproceedings{peters2010relative,
  title={Relative entropy policy search},
  author={Peters, Jan and Mulling, Katharina and Altun, Yasemin},
  booktitle={AAAI Conference on Artificial Intelligence},
  pages={1607--1612},
  year={2010}
}

@inproceedings{huang2025flexible,
  title={Flexible Locomotion Learning with Diffusion Model Predictive Control},
  author={Huang, Runhan and Balim, Haldun and Yang, Heng and Du, Yilun},
  booktitle={International Conference on Robotics and Automation},
  year={2026}
}

@inproceedings{ye2025pegasusflow,
  title={{PegasusFlow}: Parallel Rolling-Denoising Score Sampling for Robot Diffusion Planner Flow Matching},
  author={Ye, Lei and Gao, Haibo and Xu, Peng and Zhang, Zhelin and Shan, Junqi and Zhang, Ao and Zhang, Wei and Zhou, Ruyi and Deng, Zongquan and Ding, Liang},
  booktitle={International Conference on Robotics and Automation},
  year={2026}
}

@inproceedings{mishra2025eb,
  title={{EB-MBD}: Emerging-Barrier Model-Based Diffusion for Safe Trajectory Optimization in Highly Constrained Environments},
  author={Mishra, Raghav and Manchester, Ian R},
  booktitle={International Conference on Robotics and Automation},
  year={2026}
}

@inproceedings{romer2024diffusion,
  title={Diffusion predictive control with constraints},
  author={R{\"o}mer, Ralf and von Rohr, Alexander and Schoellig, Angela P},
  booktitle={Learning for Dynamics {\&} Control Conference},
  pages={791--803},
  year={2025}
}

@inproceedings{cheng2025safe,
  title={Safe and stable control via lyapunov-guided diffusion models},
  author={Cheng, Xiaoyuan and Tang, Xiaohang and Yang, Yiming},
  booktitle={Advances in Neural Information Processing Systems},
  pages={160909--160948},
  year={2025}
}

@inproceedings{huang2024diffuseloco,
  title={Diffuseloco: Real-time legged locomotion control with diffusion from offline datasets},
  author={Huang, Xiaoyu and Chi, Yufeng and Wang, Ruofeng and Li, Zhongyu and Peng, Xue Bin and Shao, Sophia and Nikolic, Borivoje and Sreenath, Koushil},
  booktitle={Conference on Robot Learning},
  pages={1567--1589},
  year={2025}
}

@inproceedings{song2020denoising,
  title={Denoising diffusion implicit models},
  author={Song, Jiaming and Meng, Chenlin and Ermon, Stefano},
  booktitle={International Conference on Learning Representations},
  year={2021}
}

@inproceedings{janner2022planning,
  title={Planning with diffusion for flexible behavior synthesis},
  author={Janner, Michael and Du, Yilun and Tenenbaum, Joshua B and Levine, Sergey},
  booktitle={International Conference on Machine Learning},
  pages={9902--9915},
  year={2022}
}

@inproceedings{zhang2025constrained,
  title={Constrained diffusers for safe planning and control},
  author={Zhang, Jichen and Zhao, Liqun and Papachristodoulou, Antonis and Umenberger, Jack},
  booktitle={Advances in Neural Information Processing Systems},
  pages={39283--39316},
  year={2025}
}

@inproceedings{liang2025simultaneous,
  title={Simultaneous multi-robot motion planning with projected diffusion models},
  author={Liang, Jinhao and Christopher, Jacob K and Koenig, Sven and Fioretto, Ferdinando},
  booktitle={International Conference on Machine Learning},
  pages={37162--37180},
  year={2025}
}

@inproceedings{kyaw2026geometry,
  title={Geometry-Aware Sampling-Based Motion Planning on Riemannian Manifolds},
  author={Kyaw, Phone Thiha and Kelly, Jonathan},
  booktitle={World Symposium on the Algorithmic Foundations of Robotics},
  year={2026}
}

@inproceedings{bakker2025tamedpuma,
  title={{TamedPUMA}: safe and stable imitation learning with geometric fabrics},
  author={Bakker, Saray and P{\'e}rez-Dattari, Rodrigo and Della Santina, Cosimo and B{\"o}hmer, Wendelin and Alonso-Mora, Javier},
  booktitle={Learning for Dynamics {\&} Control Conference},
  pages={405--418},
  year={2025}
}

@inproceedings{he2026model,
  title={Model-Based Diffusion Optimal Control for Multi-Robot Motion Planning},
  author={He, Zhilin and Shaoul, Yorai and Li, Jiaoyang},
  booktitle={Robotics: Science and Systems},
  year={2026}
}

@article{mayne2000constrained,
  title={Constrained model predictive control: Stability and optimality},
  author={Mayne, David Q and Rawlings, James B and Rao, Christopher V and Scokaert, Pierre OM},
  journal={Automatica},
  volume={36},
  number={6},
  pages={789--814},
  year={2000},
  publisher={Elsevier}
}

@article{garone2017reference,
  title={Reference and command governors for systems with constraints: A survey on theory and applications},
  author={Garone, Emanuele and Di Cairano, Stefano and Kolmanovsky, Ilya V},
  journal={Automatica},
  volume={75},
  pages={306--328},
  year={2017},
  publisher={Elsevier}
}

@article{wensing2023optimization,
  title={Optimization-based control for dynamic legged robots},
  author={Wensing, Patrick M and Posa, Michael and Hu, Yue and Escande, Adrien and Mansard, Nicolas and Del Prete, Andrea},
  journal={IEEE Transactions on Robotics},
  volume={40},
  pages={43--63},
  year={2023},
  publisher={IEEE}
}

@article{sleiman2021unified,
  title={A unified mpc framework for whole-body dynamic locomotion and manipulation},
  author={Sleiman, Jean-Pierre and Farshidian, Farbod and Minniti, Maria Vittoria and Hutter, Marco},
  journal={IEEE Robotics and Automation Letters},
  volume={6},
  number={3},
  pages={4688--4695},
  year={2021},
  publisher={IEEE}
}

@article{posa2014direct,
  title={A direct method for trajectory optimization of rigid bodies through contact},
  author={Posa, Michael and Cantu, Cecilia and Tedrake, Russ},
  journal={International Journal of Robotics Research},
  volume={33},
  number={1},
  pages={69--81},
  year={2014},
  publisher={Sage Publications Sage UK: London, England}
}

@article{gu2026humanoid,
  title={Humanoid locomotion and manipulation: Current progress and challenges in control, planning, and learning},
  author={Gu, Zhaoyuan and Li, Junheng and Shen, Wenlan and Yu, Wenhao and Xie, Zhaoming and McCrory, Stephen and Cheng, Xianyi and Shamsah, Abdulaziz and Griffin, Robert and Liu, C Karen and others},
  journal={IEEE/ASME Transactions on Mechatronics},
  volume={31},
  number={2},
  pages={2300--2330},
  year={2026},
  publisher={IEEE}
}

@article{kim2025contact,
  title={Contact-implicit model predictive control: Controlling diverse quadruped motions without pre-planned contact modes or trajectories},
  author={Kim, Gijeong and Kang, Dongyun and Kim, Joon-Ha and Hong, Seungwoo and Park, Hae-Won},
  journal={International Journal of Robotics Research},
  volume={44},
  number={3},
  pages={486--510},
  year={2025},
  publisher={Sage Publications Sage UK: London, England}
}

@article{williams2017mppi,
  title={Model Predictive Path Integral Control: From Theory to Parallel Computation},
  author={Williams, Grady and Aldrich, Andrew and Theodorou, Evangelos A.},
  journal={Journal of Guidance, Control, and Dynamics},
  volume={40},
  number={2},
  pages={344--357},
  year={2017}
}

@article{liu2018convex,
  title={The convex feasible set algorithm for real time optimization in motion planning},
  author={Liu, Changliu and Lin, Chung-Yen and Tomizuka, Masayoshi},
  journal={SIAM Journal on Control and Optimization},
  volume={56},
  number={4},
  pages={2712--2733},
  year={2018},
  publisher={SIAM}
}

@article{chi2025diffusion,
  title={Diffusion policy: Visuomotor policy learning via action diffusion},
  author={Chi, Cheng and Xu, Zhenjia and Feng, Siyuan and Cousineau, Eric and Du, Yilun and Burchfiel, Benjamin and Tedrake, Russ and Song, Shuran},
  journal={International Journal of Robotics Research},
  volume={44},
  number={10-11},
  pages={1684--1704},
  year={2025},
  publisher={Sage Publications Sage UK: London, England}
}

@article{cheng2021rmpflow,
  title={Rmpflow: A geometric framework for generation of multitask motion policies},
  author={Cheng, Ching-An and Mukadam, Mustafa and Issac, Jan and Birchfield, Stan and Fox, Dieter and Boots, Byron and Ratliff, Nathan},
  journal={IEEE Transactions on Automation Science and Engineering},
  volume={18},
  number={3},
  pages={968--987},
  year={2021},
  publisher={IEEE}
}

@article{van2022geometric,
  title={Geometric fabrics: Generalizing classical mechanics to capture the physics of behavior},
  author={Van Wyk, Karl and Xie, Mandy and Li, Anqi and Rana, Muhammad Asif and Babich, Buck and Peele, Bryan and Wan, Qian and Akinola, Iretiayo and Sundaralingam, Balakumar and Fox, Dieter and others},
  journal={IEEE Robotics and Automation Letters},
  volume={7},
  number={2},
  pages={3202--3209},
  year={2022},
  publisher={IEEE}
}

@article{gilbert1995discrete,
  title={Discrete-time reference governors and the nonlinear control of systems with state and control constraints},
  author={Gilbert, Elmer G and Kolmanovsky, Ilya and Tan, Kok Tin},
  journal={International Journal of Robust and Nonlinear Control},
  volume={5},
  number={5},
  pages={487--504},
  year={1995},
  publisher={Wiley Online Library}
}

@article{fishman2023diffusion,
  title={Diffusion models for constrained domains},
  author={Fishman, Nic and Klarner, Leo and De Bortoli, Valentin and Mathieu, Emile and Hutchinson, Michael},
  journal={Transactions on Machine Learning Research},
  year={2023}
}

@article{kim2019highly,
  title={Highly dynamic quadruped locomotion via whole-body impulse control and model predictive control},
  author={Kim, Donghyun and Di Carlo, Jared and Katz, Benjamin and Bledt, Gerardo and Kim, Sangbae},
  journal={arXiv preprint arXiv:1909.06586},
  year={2019}
}

@article{liang2024multi,
  title={Multi-agent path finding in continuous spaces with projected diffusion models},
  author={Liang, Jinhao and Christopher, Jacob K and Koenig, Sven and Fioretto, Ferdinando},
  journal={arXiv preprint arXiv:2412.17993},
  year={2024}
}

@article{levine2018reinforcement,
  title={Reinforcement learning and control as probabilistic inference: Tutorial and review},
  author={Levine, Sergey},
  journal={arXiv preprint arXiv:1805.00909},
  year={2018}
}

@article{he2026motion,
  title={Motion Planning with Model-Based Diffusion via Constraint Optimization and Adaptive Scheduling},
  author={He, Zhilin and Li, Bowei and Dou, Jianlin and Zhang, Yuner and Liu, Changliu},
  journal={arXiv preprint arXiv:2607.14455},
  year={2026}
}

@IEEEtranBSTCTL{twogo_bib_format,
  ctluse_forced_etal={yes},
  ctlmax_names_forced_etal={6},
  ctlnames_show_etal={1}
}
\par\endgroup

\clearpage
\noindent\textbf{Guide to the appendix.}
\begin{itemize}
\item Appendix~\ref{app:model}: rollout derivatives,
generative identities, local geometry, and conditional feasibility results.
\item Appendix~\ref{app:env}: task models,
experimental variants, deployment, evaluation protocol, and per-seed outcomes.
\item Appendix~\ref{app:baselines}: baseline
mechanisms and comparison budgets.
\item Appendix~\ref{app:hparams}: numerical
schedules, objective weights, and parameter settings.
\end{itemize}
\appendices
\section{Method and Theoretical Details}
\label{app:model}

\subsection{Rollout Parameterization and Constraint Derivatives}

Let $s_0$ be the initial planning state and $Y$ the stacked free controls
used in the main text. For the detailed dynamics, we index physical time
from zero: $H=T-1$, $s^0=s_0$, and
$Y=(u^0,\ldots,u^{H-1})\in\mathbb R^{H d_u}$.
The trajectory associated with $Y$ is
\begin{equation}
 s^{t+1}=f(s^t,u^t),\quad
 \mathcal R_f(Y;s_0)=(s^{0:H},Y).
 \label{eq:app_rollout}
\end{equation}
The trajectory notation in the main text refers to this rollout-induced
representation. Gaussian sampling and Euclidean corrections act in the free
control coordinates; the state sequence is regenerated by $f$. Accordingly,
a density over dynamically feasible trajectories is understood with respect
to the measure induced by the control parameterization.

The task cost and a trajectory constraint pull back to
\begin{align}
 J(\mathcal R_f(Y;s_0),c)
 &=\sum_{t=0}^{H-1}\ell^t(s^t,u^t,c)+\ell_T(s^H,c),
 \label{eq:app_energy}\\
 \widetilde g_i(Y)&=g_i(\mathcal R_f(Y;s_0),c),\nonumber\\
 \nabla_Y\widetilde g_i&=D_Y\mathcal R_f(Y;s_0)^\top\nabla_\tau g_i.
 \label{eq:app_rollout_chain_rule}
\end{align}
On a differentiable branch of the dynamics, the state sensitivity satisfies
\begin{equation}
 B_{t+1,j}=f_s(s^t,u^t)B_{t,j}
 +\mathbb I[t=j]f_u(s^t,u^t),\quad B_{0,j}=0,
 \label{eq:app_rollout_sensitivity}
\end{equation}
where $B_{t,j}=\partial s^t/\partial u^j$. Thus the complete constraint
Jacobian includes the influence of a control on all subsequent states.
Piecewise models are differentiated on the current contact, clipping, and
nearest-geometry branch. The task-local approximation below retains selected
one-step effects rather than this complete horizon Jacobian. The low-level
tracking policy is outside the planning transition $f$.

\subsection{Clean Estimates and Reverse Transports}
\label{app:transport_identities}

For the control-space realization, let $U_k=\sqrt{\bar\alpha_k}Y_k$
denote the noised free-control vector, with clean endpoint $U_0=Y_0$.
The physical trajectory is obtained separately through $\mathcal R_f$.
The forward process is
$U_k=\sqrt{\bar\alpha_k}U_0+\sqrt{1-\bar\alpha_k}\epsilon$ for $k\geq1$,
with $\bar\alpha_0=1$. For a fixed noisy coordinate $U_k$, the Gaussian
kernel as a function of $U_0$ is proportional to
\begin{equation}
 \Omega_k=\mathcal N\!\left(U_k/\sqrt{\bar\alpha_k},
 [(1-\bar\alpha_k)/\bar\alpha_k]I\right).
 \label{eq:app_clean_proposal}
\end{equation}
For samples from this proposal, target-density weighting approximates the
conditional clean mean. With a soft feasibility prior, the weights are
\begin{align}
 w_{k,m}=\frac{\exp[-(J_{k,m}^{\mathrm{aug}}-J_{k,\min}^{\mathrm{aug}})/\lambda_k]}
 {\sum_{n=1}^{M}\exp[-(J_{k,n}^{\mathrm{aug}}-J_{k,\min}^{\mathrm{aug}})/\lambda_k]},\nonumber\\
 \bar Y_k=\sum_{m=1}^{M}w_{k,m}\tilde Y_{k,m}.
 \label{eq:app_reward_weight}
\end{align}
Here $\tilde Y_{k,m}$ is a clean control candidate,
$J_{k,m}^{\mathrm{aug}}$ is its rollout task cost $J$ augmented by the
feasibility penalty, and $J_{k,\min}^{\mathrm{aug}}$ is the batch minimum.
The mean $\bar Y_k$ is the uncorrected estimate in Eq.~\eqref{eq:mc_avg_2go},
and $\lambda_k$ combines batch cost scaling with the sampling temperature.
Subtracting the minimum leaves the normalized weights unchanged. The
experimental parameterization samples bounded clean controls and clips them
to the action domain; it therefore realizes a bounded finite-bank
approximation to the Gaussian construction.

Differentiating the Gaussian forward kernel gives
\begin{equation}
 \nabla_{U_k}\log p_k(U_k)
 =\frac{\sqrt{\bar\alpha_k}\,\mathbb E[U_0\mid U_k]-U_k}
 {1-\bar\alpha_k}.
 \label{eq:app_score_identity}
\end{equation}
Replacing the conditional mean by its finite-sample estimate yields MCSA.
The geometry-shaped clean estimate $\hat Y^G_{0\mid k}=Y_k+d_k^G$
determines a
mutually consistent score and noise representation:
\begin{equation}
 \hat{\mathbf S}_k^G=\frac{\sqrt{\bar\alpha_k}\hat Y^G_{0\mid k}-U_k}
 {1-\bar\alpha_k}, \quad
 \hat\epsilon^G_k=
 \frac{U_k-\sqrt{\bar\alpha_k}\hat Y^G_{0\mid k}}
 {\sqrt{1-\bar\alpha_k}}.
 \label{eq:app_shaped_estimates}
\end{equation}
Substituting $U_k=\sqrt{\bar\alpha_k}Y_k$ recovers
Eq.~\eqref{eq:geometry_estimates_tau}; the score is taken with respect to
the noisy controls $U_k$. These are algebraic representations of the shaped
estimate. The extra task
drift and metric shaping do not imply that $\hat{\mathbf S}_k^G$ is the score
of another known density. The reverse maps admit these equivalent
representations. In 2GO, the base transport uses the weighted clean estimate,
and the metric-shaped displacement is added separately in
Eq.~\eqref{eq:unified_reverse_update_k}.

The DDPM-style clean-moment realization is given in
Eq.~\eqref{eq:ddpm_reverse_2go}.
Substituting the score identity above into the cumulative-variance form
gives $\mu_k^{\mathrm{clean}}=\sqrt{\bar\alpha_{k-1}}\bar Y_k$.
Dividing by $\sqrt{\bar\alpha_{k-1}}$ yields the clean-coordinate
base transport $\Gamma_k^{\mathrm{DDPM}}(Y_k,\bar Y_k)=\bar Y_k$.

\subsection{General Geometry and Task-Local Approximations}
\label{app:geometry_proxy}

The full active geometry is formed from normalized rows of
$\nabla_Y\widetilde g_i$, with cardinality $K_{a,k}\leq K_A$ and
$A_k\in\mathbb R^{K_{a,k}\times H d_u}$. In the locomotion instances, a local
matrix $\widehat A_k\in\mathbb R^{K_{a,k}\times d_u}$ instead collects one-step
action-space directions. This approximation selects aggregate time rows by
their norm, rather than individual constraints by their weighted residual.
Up to $K_A=8$ such rows are retained and
regularized by row normalization; the resulting map is applied to each control row along
the horizon. The associated linear systems have dimension equal to the
number $K_{a,k}$ of retained rows.
This construction preserves the metric/filter interface while reducing the
cost of forming a complete horizon Jacobian.

For stepping stones, foot $\ell$ has world position
\begin{equation}
 p_{\ell,t}=\begin{bmatrix}s_t\\y_c+e_{y,t}\end{bmatrix}
 +R(e_{\psi,t})(p_\ell^{\mathrm{nom}}+r_{\ell,t}).
 \label{eq:app_quad_foot}
\end{equation}
With disk centers $c_j$ and radii $r_j$, define the local radial proxy
\begin{align}
 d_{\mathrm{rad}}(p)&=\min_j\bigl(\|p-c_j\|_2-(r_j-m_{\mathrm{stone}})\bigr),\\
 q_{\ell,t}&=[c_{\mathrm{clr}}-d_{\mathrm{rad}}(p_{\ell,t+1})]_+.
 \label{eq:app_quad_sdf}
\end{align}
The minimum applies to the complete center-radius expression. This radial
quantity is negative inside a shrunken disk; it is a directional proxy,
whereas actual foothold feasibility is measured by distance to the admissible
support union defined in Appendix~\ref{app:env}. With $L_\ell$ inserting a
two-dimensional vector into foot $\ell$'s residual-rate coordinates, the row is
\begin{equation}
 a_t^{\mathrm{loc}}=-\Delta t\sum_{\ell=1}^{4}
 \omega_{\ell,t}q_{\ell,t}L_\ell\nabla_p d_{\mathrm{rad}}(p_{\ell,t+1}).
 \label{eq:app_quad_geometry_row}
\end{equation}
We compute $\nabla_p d_{\mathrm{rad}}$ by automatic differentiation on the
active nearest-stone branch, then apply the explicit lift
$\Delta t L_\ell$. This local approximation retains residual-foot
directions and omits body and downstream-time sensitivities.
The phase weights $\omega_{\ell,t}$ instantiate the task modulation
$\omega_i(k;\zeta)$ along the current rollout. Support feet have unit geometry weight; swing
feet have weight $0.2$ before phase $0.7$, then a linear ramp to $1$ at
touchdown. These geometry weights differ from the phase-dependent rates in
the transition model.

For the humanoid, $d_{\mathrm{body}}(s)$ is the minimum proxy clearance of
the torso ellipse and two arm tips from corridor boundaries and obstacles.
The local row is
\begin{align}
 a_t^{\mathrm{loc}}&=-2q_t D_t\nabla_{u^t}
 d_{\mathrm{body}}(f(s^t,u^t)),\\
 q_t&=[c_{\mathrm{clr}}-d_{\mathrm{body}}(f(s^t,u^t))]_+.
 \label{eq:app_human_proxy}
\end{align}
The action derivative is computed by automatic differentiation through
the one-step planning transition and body-clearance function, holding
$s^t$ fixed. It does not differentiate the full horizon or the downstream
tracking policy. The diagonal $D_t$ has unit entries except
for the longitudinal command, whose weight is $0.1$ when $q_t\leq0.3$ and
$0.5$ otherwise. Both the residual and its derivative are evaluated at the
same next state. The raw rows $a_t^{\mathrm{loc}}$ are normalized before
stacking $\widehat A_k$; the hat on the main-text $\hat a_i$ denotes this
normalization.

\subsection{Metric and Tangent Maps}
\label{app:metric_maps}

Let $A\in\mathbb R^{K_{a,k}\times d}$ denote either the full matrix or its local
approximation in the corresponding coordinates. For $z\in\mathbb R^d$,
\begin{align}
 P_{\mathrm{met}}z
 &=z-A^\top\bigl(I+AA^\top+\epsilon_G I\bigr)^{-1}Az,
 \label{eq:app_metric_map}\\
 P_{\mathrm{tan}}z
 &=z-A^\top\bigl(AA^\top+\epsilon_P I\bigr)^{-1}Az.
 \label{eq:app_tangent_map}
\end{align}
For $\epsilon_G=0$, Woodbury's identity gives
$P_{\mathrm{met}}=(I+A^\top A)^{-1}$ exactly. Positive regularization
provides a numerically stable approximation; the experiments use
$\epsilon_G=\epsilon_P=10^{-6}$ and Cholesky solves. Row normalization limits
the influence of differences in derivative scale.
With $A=A_k$, these maps recover $G_k^{-1}$ (when $\epsilon_G=0$)
and $P_k$, respectively. With $A=\widehat A_k$, they are the corresponding
task-local operators.

\begin{lemma}[Row-Space Attenuation]
\label{lem:app_attenuation}
Let $s$ be a singular value of $A$. Along its right singular vector,
$P_{\mathrm{met}}$ has gain
$(1+\epsilon_G)/(1+\epsilon_G+s^2)$ and $P_{\mathrm{tan}}$ has gain
$\epsilon_P/(\epsilon_P+s^2)$. Both maps leave $\ker(A)$ unchanged.
\end{lemma}
\begin{proof}
Substitute an SVD $A=Q\Sigma V^\top$ into
Eqs.~\eqref{eq:app_metric_map}--\eqref{eq:app_tangent_map}. Each nonzero
singular direction gives the stated scalar factor; when $Az=0$, the
correction term vanishes.
\end{proof}

The lemma establishes attenuation with respect to the represented normals.
For exact normals, their null space is a local tangent subspace. With local
proxies, the same algebra applies to the represented geometry; nonlinear
feasibility additionally depends on approximation error and the subsequent
correction. In particular, regularized tangent noise can retain a small
normal component.

\subsection{Local CFS Correction}
\label{app:cfs}

For post-transport correction, CFS constructs at $Y_k$ a convex local
approximation to the selected rollout constraints and retracts the
transported proposal $\tilde Y_{k-1}$. Allowing optional nonnegative slack
$\xi$, the metric-proximal subproblem is
\begin{align}
 \min_{Y,\,\xi\geq0}\quad
 &\tfrac12\|Y-\tilde Y_{k-1}\|_R^2
 +\tfrac12\mu_{\mathrm{slack}}\|\xi\|_2^2\nonumber\\
 \text{s.t.}\quad&\widetilde g_i(Y_k)+a_i(Y_k)^\top(Y-Y_k)
 \leq-b_i+\xi_i,\nonumber\\
 &i\in\mathcal A_k,\quad u_{\min}\leq Y\leq u_{\max}.
 \label{eq:app_cfs_target}
\end{align}
Here $a_i(Y_k)=\nabla_Y\widetilde g_i(Y_k)$,
$b_i=b_i(k;\zeta_k)$ is the main-text clearance offset, $R\succ0$, and
$\mu_{\mathrm{slack}}>0$ penalizes slack.
The linearization point $Y_k$ stays fixed while the objective minimizes
displacement from $\tilde Y_{k-1}$. With zero slack, an exact solution gives
the retraction in Eq.~\eqref{eq:cfs_retraction_qp_tau}, with control bounds
included. The corrected trajectory is $\mathcal R_f(Y^+;s_0)$.
A finite sequence of box projections and most-violated-half-space corrections
gives an inexact realization. Candidate-point relinearization in the
experimental variant is described separately in
Appendix~\ref{app:task_instantiations}.

\begin{prop}[Conditional Local Feasibility]
\label{prop:app_local_feasibility}
Suppose the selected CFS half-spaces are inner approximations of their
nonlinear feasible sets in a neighborhood of $Y_k$. If the
returned control remains in that neighborhood, satisfies those half-spaces
and the box bounds, and has zero slack, it is feasible for the corresponding
selected nonlinear constraints.
\end{prop}
\begin{proof}
Within the stated neighborhood, membership in each half-space implies
membership in its nonlinear feasible set. Intersecting these sets and the
box bounds gives the result.
\end{proof}

This local implication separates the geometry of the ideal subproblem from
its numerical solution. Approximate solution, nonzero slack, a change in the
nearest obstacle, or a change in contact mode requires accounting for the
corresponding residual or model error. The proposition does not assert
feasibility of omitted constraints or global convergence.

For an inexact correction, write the affine constraint model as
$\ell_i(Y)=\widetilde g_i(Y_k)+a_i(Y_k)^\top(Y-Y_k)$.
If the returned control satisfies
$\ell_i(Y^+)\leq-b_i+\epsilon_i^{\mathrm{solve}}$ and its local model error is
$|\widetilde g_i(Y^+)-\ell_i(Y^+)|\leq\epsilon_i^{\mathrm{lin}}$, then
\begin{equation}
 \widetilde g_i(Y^+)
 \leq-b_i+\epsilon_i^{\mathrm{solve}}+\epsilon_i^{\mathrm{lin}}.
 \label{eq:app_inexact_cfs_bound}
\end{equation}
This follows by adding the model error to the affine residual bound.
The solve residual includes any allowed slack. A margin
$b_i\geq\epsilon_i^{\mathrm{solve}}+\epsilon_i^{\mathrm{lin}}$
therefore suffices for that selected constraint, provided the returned
control remains in the neighborhood where the model-error bound holds.
These are conditional error bounds, not assumed numerical guarantees of
every finite correction.

\subsection{Conditional Transfer to Execution}
\label{app:execution_conditions}
\label{app:tracking_bound}

Let $z^t_{\mathrm{ref}}$ collect the geometric features of the final
reference delivered to the tracker, after adaptation, filtering, and
saturation, and $z^t_{\mathrm{exec}}$ the corresponding executed features.
Assume the tracking interface satisfies
$\|z^t_{\mathrm{exec}}-z^t_{\mathrm{ref}}\|\leq\delta$ over the considered
segment, and let each selected body-proxy constraint $g_j$ be
$L_j$-Lipschitz on that neighborhood. Then
\begin{equation}
 g_j(z^t_{\mathrm{exec}})
 \leq g_j(z^t_{\mathrm{ref}})+L_j\delta\leq0
 \label{eq:app_execution_transfer}
\end{equation}
whenever $g_j(z^t_{\mathrm{ref}})\leq-L_j\delta$ holds for this final
reference. The feature vector must
include every base or posture coordinate entering the chosen clearance
model. This is a conditional clearance implication for the selected
geometric constraints; recursive feasibility
and arrival at the original goal require additional tracking and progress
conditions. The empirical margin calibration in Appendix~\ref{app:deployment}
estimates observed clearance degradation; it neither proves the assumed
uniform tracking bound nor establishes that its margin covers $L_j\delta$.

\section{Environments and Evaluation}
\label{app:env}

In the task models below, $x_t\equiv s^t$ denotes the planning state,
$u_t\equiv u^t$ its control, and $H=T-1$ the number of control intervals.
Thus $d_x$ in Table~\ref{tab:app_env_summary} is the dimension of the
main-text state $s^t$. Scalar $x_t,y_t$ in the humanoid planar equations
denote position components.
The task models are followed by the experimental variants and deployment
interfaces; the final three subsections define the metrics and report
supplementary results.

\begin{table*}[tb!]
\centering
\scriptsize
\setlength{\tabcolsep}{4pt}
\begin{tabular}{lccccp{8.4cm}}
\toprule
Task & $d_x$ & $d_u$ & $H$ & $\Delta t$ & State and action summary \\
\midrule
Stepping stones & 16 & 12 & 16 & 1.0 &
$x=[s,e_y,e_\psi,v_s,v_y,\omega,r_{FL}^{\top},r_{FR}^{\top},r_{RL}^{\top},r_{RR}^{\top},m,\vartheta]$;
$u=[\Delta s_b,\Delta y_b,\dot\psi_b,\dot r_{FL}^{\top},\dot r_{FR}^{\top},\dot r_{RL}^{\top},\dot r_{RR}^{\top},\dot\vartheta]$. \\
Humanoid corridor & 14 & 9 & 50 & 0.25 &
$x=[x,y,\psi,h,\psi_T,a_L,a_R,p_L,p_R,v_x,v_y,\omega,\dot h,\dot\psi_T]$;
$u=[v_x^b,v_y^b,\omega_c,\dot h_c,\dot\psi_{T,c},\dot a_L,\dot a_R,\dot p_L,\dot p_R]$. \\
\bottomrule
\end{tabular}
\caption{Environment dimensions and 2GO planning discretization.
Humanoid baselines use $H=50$. Humanoid velocity actions are body-frame commands.}
\label{tab:app_env_summary}
\end{table*}

\subsection{Quadruped Stepping Stones}

\textbf{Hybrid reference transition.}
The state is $x=(s,e_y,e_\psi,v_s,v_y,\omega,r_1,\ldots,r_4,m,\vartheta)$,
where each $r_\ell\in\mathbb R^2$ is a body-frame foot residual. Footholds
are induced by Eq.~\eqref{eq:app_quad_foot}; stone indices are not independent
discrete optimization variables. The diagonal support mode $m$ follows a
fixed alternation, while continuous residual and phase controls determine
when admissible landing configurations are reached.

Let $a_s,a_y,a_\psi$ be bounded body and yaw commands, and let
$b_\ell(m,\vartheta)$ scale the bounded residual-rate command. It is $0.06$
on support legs, $1$ on swing legs before $\vartheta=0.7$, and
$0.2+0.8(1-\vartheta)/0.3$ on swing legs thereafter. The provisional update is
\begin{align}
 \widetilde s&=\max(s-0.005,s+\Delta t\,a_s),\\
 \widetilde e_y&=\operatorname{clip}(e_y+\Delta t\,a_y,-0.08,0.08),\\
 e_\psi^+&=\operatorname{clip}(\operatorname{wrap}(e_\psi+\Delta t\,a_\psi),-0.22,0.22),\\
 \widetilde r_\ell&=\operatorname{clip}
 (r_\ell+\Delta t\,b_\ell\dot r_\ell,-0.05,0.05).
 \label{eq:app_quad_transition}
\end{align}
The body command is first restricted to a norm ball of radius $0.22$ and
the component limits below. Each raw residual-rate vector has norm at most
$0.06$. Let $\widetilde p_b=(\widetilde s,y_c+\widetilde e_y)$ and let
$c_{\mathrm{sup}}$ be the midpoint of the two support feet decoded from this
provisional state. The final body position is obtained from
\begin{equation}
 p_b^+=c_{\mathrm{sup}}+\operatorname{clip}_{b_{\mathrm{sup}}}
 (\widetilde p_b-c_{\mathrm{sup}}),
 \label{eq:app_quad_support_update}
\end{equation}
where $b_{\mathrm{sup}}=(0.14,0.08)$ gives symmetric component bounds,
followed by the same longitudinal and lateral bounds. Stored velocity
coordinates are the bounded finite differences of the provisional body and
yaw updates. Pairwise residual adjustments separate feet closer than $0.10$ m
while retaining the residual bounds.

The admissible support set is the union $\mathcal S_{\mathrm{foot}}$ of
margin-shrunken stone disks and support rectangles. Its violation is
\begin{align}
 d_{\mathrm{disk}}(p)&=\min_j[\|p-c_j\|-(r_j-m_{\mathrm{stone}})]_+,\nonumber\\
 d_{\mathrm{plat}}(p)&=\min_q\operatorname{dist}(p,\mathcal P_q^-),\nonumber\\
 v_{\mathrm{foot}}(p)&=\min\{d_{\mathrm{disk}}(p),d_{\mathrm{plat}}(p)\}.
 \label{eq:app_quad_support_distance}
\end{align}
With $v_{\mathrm{sw}}$ the maximum violation of the current swing pair and
$\widetilde\vartheta=\operatorname{clip}(\vartheta+\Delta t\dot\vartheta,0,1)$,
the hybrid reset is
\begin{equation}
 (m^+,\vartheta^+)=
 \begin{cases}
 (1-m,0),&\widetilde\vartheta=1,\ v_{\mathrm{sw}}\leq0.04,\\
 (m,0.99),&\widetilde\vartheta=1,\ v_{\mathrm{sw}}>0.04,\\
 (m,\widetilde\vartheta),&\widetilde\vartheta<1.
 \end{cases}
 \label{eq:app_quad_contact_reset}
\end{equation}
At a successful switch, residuals of the previous support pair reset to zero.
This model describes body--foothold references; the deployment gait is
constructed by the execution interface below.

The course starts at $(-1.25,0)$ and ends at $(1.25,0)$. Nominal foot offsets
are $x=0.18$ for front feet, $x=-0.18$ for rear feet, and lateral offsets
$\pm0.15$. Mode $m=0$ supports $FL$ and $RR$; mode $m=1$ supports $FR$ and
$RL$.
The S1 generator draws stone radii from
$[0.076,0.088]$ m and nominal lane gaps from $[0.22,0.28]$ m; S2
uses $[0.068,0.080]$ m and $[0.26,0.32]$ m. The run
configuration multiplies lane gaps by $0.97$ and evaluates foot clearance with
a $0.032$ m stone-interior margin.

In addition to the bounds in the transition above, stored forward velocity
is capped at $0.28$, lateral velocity at $\pm0.10$, and yaw rate at
$\pm0.22$. The phase-rate limit is $0.45$ and the kinematic reach bound is
$0.35$.

The planning cost is a weighted sum of terminal goal and progress,
time-guide, lateral and velocity regularization, heading, support-center,
terminal-foot, smoothness, foothold, support, geometry, and step-bound terms.
Table~\ref{tab:app_energy_weights} reports the explicit multipliers.

\subsection{Humanoid Corridor}

Longitudinal and lateral commands are expressed in the body frame. The planar
transition is
\begin{align}
 x_{t+1}&=x_t+\Delta t(v_{x,t}^b\cos\psi_t-v_{y,t}^b\sin\psi_t),\\
 y_{t+1}&=y_t+\Delta t(v_{x,t}^b\sin\psi_t+v_{y,t}^b\cos\psi_t),\\
 \psi_{t+1}&=\operatorname{wrap}(\psi_t+\Delta t\,\omega_{c,t}).
 \label{eq:app_human_dynamics}
\end{align}
Writing $q=(h,\psi_T,a_L,a_R,p_L,p_R)$, the remaining posture transition and
stored command coordinates are
\begin{align}
 q_{t+1}&=\operatorname{clip}(q_t+\Delta t\,u_{q,t},q_{\min},q_{\max}),\\
 (v_x,v_y,\omega,\dot h,\dot\psi_T)_{t+1}
 &=(v_x^b,v_y^b,\omega_c,\dot h_c,\dot\psi_{T,c})_t.
 \label{eq:app_human_posture_dynamics}
\end{align}
The stored planar velocity components are body-frame commands. Actions are
clipped before this update. The reference model integrates base and posture
commands rather than the joint-level closed-loop dynamics. Bounds are $v_x^b,v_y^b\in[-0.3,0.3]$, height rate
$[-0.3,0.3]$, torso-yaw rate $[-0.4,0.4]$, and tuck/posture rates
$[-1,1]$. Yaw-rate bounds are $[-0.3,0.3]$ in Zone B and $[-0.5,0.5]$
elsewhere. Height is restricted to $[0.55,0.85]$, torso yaw to
$[-1.57,1.57]$, tuck to $[0,1]$, and arm posture to $[-1,1]$.

The body proxy has torso ellipse semi-axes $(0.20,0.11)$, an additional
$0.03$ crouch allowance, nominal height $0.75$, body half-height $0.25$, and
arm-tip radius $0.05$. Arm reach decreases from $0.35$ when open to $0.10$
when tucked. The geometry-aware reverse step uses the minimum clearance of the
torso ellipse and the two arm tips. CFS constructs a denser five-point set
containing the torso, both tips, and arm midpoints. Obstacle avoidance is
handled by feasibility terms, geometry, and the candidate-correction variant
specified below; the humanoid
obstacle contribution to the scalar planning cost is zero.

\begin{table}[H]
\centering
\scriptsize
\setlength{\tabcolsep}{3pt}
\begin{tabular}{@{}clp{5.2cm}@{}}
\toprule
Zone & Start $\rightarrow$ goal & Constraint geometry \\
\midrule
A & $(0.5,0)\rightarrow(3.5,0)$ & Centered thin wall,
$x\in[1.7,2.0]$, $y\in[-0.05,0.05]$. \\
B & $(1.0,0.1)\rightarrow(3.0,0.25)$ & Connected U-wall with
top segment $x\in[1.8,1.95]$, $y\in[-0.16,0.8]$. \\
C & $(0.5,0)\rightarrow(3.5,0)$ & Tapered squeeze over
$x\in[1.5,2.5]$ with a $0.60$ main gap and quarter-circle entry and exit. \\
D & $(0.5,0)\rightarrow(3.5,0)$ & Two full-height planning obstacles
centered at $(2.0,0.25)$ and $(2.5,-0.25)$, each with radius $0.06$. \\
\bottomrule
\end{tabular}
\caption{Humanoid corridor geometry. All corridors are 4.0 m
long.}
\label{tab:app_zones}
\end{table}

\subsection{Task Instantiations of the Generative Framework}
\label{app:task_instantiations}
\label{app:transport_instance}

Both tasks use local geometry, bounded-control sampling, and adaptive
noise. The reported quadruped experiments use the geometry-aware update
without CFS; the humanoid experiments use a candidate-correction variant.
This variant applies CFS before the weighted clean estimate is formed,
with post-transport retraction disabled. Only candidate correction is
enabled in Algorithm~\ref{alg:twogo}.

For candidate correction, 32 trajectories are selected: half from the
highest-violation quarter and half uniformly from the candidate bank. One
local correction pass produces $\tilde Y_{k,m}^{+}$ for selected candidates; the
others remain unchanged. The original evaluation weights are retained, giving
\begin{equation}
 \bar Y_k^{\mathrm{corr}}=\sum_{m=1}^{M}w_{k,m}\tilde Y_{k,m}^{+}.
 \label{eq:app_corrected_bank}
\end{equation}
This is a finite-bank correction of the weighted clean estimate. Because the
points have moved while their weights are retained, it is not an unchanged
importance-sampling identity for the original proposal. Correction residuals
also contribute to the local geometry with blend coefficient $0.3$. The
humanoid correction retains at most eight half-spaces and uses $R=I$
with explicit slack disabled. Task settings are listed in
Table~\ref{tab:app_twogo_task_params}.

For this variant, selected derivatives are recomputed at each candidate
being corrected. The half-spaces share the physical constraint model with
the reverse geometry but need not equal the compressed $\widehat A_k$.
The humanoid safety-point Jacobian $D_u p_\ell(f(x_t,u_t))$ is obtained by
forward-mode automatic differentiation with $x_t$ fixed. Its transpose
maps obstacle normals to control-space half-space normals. Corridor-wall
normals are analytic; other obstacle normals use automatic differentiation
on their active branches. These local Jacobians omit cross-time effects.

The transport coefficients are those in
Eqs.~\eqref{eq:clean_base_main}--\eqref{eq:clean_interpolation_main}, using
$\bar Y_k^{\mathrm{corr}}$ for the main-text mean in the humanoid and
uncorrected $\bar Y_k$ in the quadruped. The same estimate supplies the displacement
$d_k^0$ and the base transport. With no post-transport retraction, the task
instances return $Y_{k-1}=\operatorname{clip}(\tilde Y_{k-1},u_{\min},u_{\max})$
after the unified reverse update. Numerical schedules are specified in
Appendix~\ref{app:hparams}.

\subsection{Deployment Interfaces}
\label{app:deployment}

\textbf{Interface scope and provenance.}
We define two customized reference-adaptation schemes following the
reference-governor principle~\cite{garone2017reference}.
These works motivate reference supervision of an already controlled system;
the schemes below are task-specific adaptations of that principle, not
direct implementations of a particular cited governor. Recursive-feasibility
and convergence guarantees from those formulations are not assumed here.
The quadruped scheme constructs a gait and reachable footholds along the
planned body route; the humanoid scheme adjusts commands under local
clearance and command constraints.

\textbf{Quadruped execution.}
The quadruped interface resamples the planned body path, synthesizes a walking
gait with leg order $FL,RR,FR,RL$, holds stance anchors, and projects nominal
touchdown targets into the interiors of reachable stones or platforms. It
therefore performs gait and foothold construction in addition to temporal
smoothing. Support-polygon membership is evaluated as a diagnostic; it is not imposed as
an additional optimization constraint in this execution interface.

Execution uses the MuJoCo Go2 model with a contact-dependent gait state
machine, task-space foot tracking, and joint-space PD regularization.
Foot and joint errors generate actuator torques. A supplementary PD
wrench acts directly on the floating base to stabilize its pose and
velocity. This assistance is enabled for all compared methods; the
quadruped eSSR therefore characterizes this assisted simulation model.

\textbf{Humanoid reference adaptation.}
The execution interface lowers the reference pace and holds the terminal
reference. Position feedback augments the reference velocity with a bounded
planar correction. Numerical settings are listed in
Table~\ref{tab:app_governor_params}.

The reference governor selects a reference position/yaw increment $d$
by approximately solving the local constrained adaptation
\begin{equation}
 \min_{d\in\mathcal D_t}\quad
 q\|d-d_{\mathrm{ref}}\|_2^2+\eta\|d\|_2^2,
 \label{eq:app_governor_obj}
\end{equation}
where $\mathcal D_t$ intersects reference-speed, posture, and corridor-safety boxes
with an active linearized body-clearance half-space. The unconstrained preference is
$q d_{\mathrm{ref}}/(q+\eta)$. A bounded alternating correction uses
box/half-space sweeps, with spatial derivatives approximated by finite
differences. The active half-space corrects world-frame lateral position,
with planned posture held fixed. If the box intersection is empty, the
adapter prioritizes finite safety-box bounds and can relax the nominal
speed and posture bounds.
The execution interface uses the adapter's translational output, filtered
and then converted to the measured body frame;
yaw feedforward is computed separately from the retimed reference, rather
than taken from the adapter's yaw output.
The downstream policy then clips each planar velocity component and yaw rate.
Thus the adapter's nominal box and the policy input limits are distinct.
Finite correction can leave clearance residuals. Clearance is not rechecked
after filtering and saturation; executed body-proxy clearance is evaluated
a posteriori.

The humanoid follower uses the SPARK-derived sport-mode policy interface.
Learned leg targets and upper-body commands from a separate PD path are
merged, followed by joint-limit clipping.

\textbf{Planning--execution compatibility.}
Table~\ref{tab:app_execution_compatibility} distinguishes constraints in the
planning model, reference adapter, and tracker. Planning bounds and
control-magnitude penalties limit nominal tracking demands; retiming and
filtering further moderate the executed commands. The reference-speed
box bounds position increments divided by the control period, not changes
between successive velocity commands. No explicit locomotion-command slew
bound is imposed. The governor and low-level tracker are downstream of the
14-state planning model and are not part of the transition in
Eq.~\eqref{eq:app_human_dynamics}. These measures do not establish that every
planned reference is admissible for execution; the clearance-transfer
conditions are stated in Eq.~\eqref{eq:app_execution_transfer}.

\textbf{Calibration and hardware interface.}
For 2GO, the empirical clearance margin is calibrated as
\begin{equation}
 \epsilon=\max(0,d_{\min}^{\mathrm{plan}}-d_{\min}^{\mathrm{exec}}),
 \quad m_{\mathrm{track}}=1.3\epsilon.
 \label{eq:app_tracking_margin}
\end{equation}

Here, $d_{\min}^{\mathrm{plan}}$ and $d_{\min}^{\mathrm{exec}}$ are the
minimum body-proxy clearances of the selected model trajectory and its
ungoverned calibration rollout. Their difference is a trajectory-level
statistic, not a pointwise tracking-error bound.
The executed body proxy uses measured base pose and pelvis height
together with commanded arm and torso posture; it evaluates this geometric
representation rather than every measured robot link. Calibrated margins
are attached to individual references; the baseline executions use the
common fallback margin in Table~\ref{tab:app_governor_params}.
The stored $m_{\mathrm{track}}$ tightens the execution-side corridor wall
box by adding to the modeled torso half-width and sets the clearance
target of the local obstacle half-space. It is not added to the planner's
CFS margin. The half-space activates below clearance
$m_{\mathrm{track}}+b_{\mathrm{act}}$, where $b_{\mathrm{act}}$ is the
activation band; its correction target remains $m_{\mathrm{track}}$.
The lookahead distance selects where clearance is queried, while the
angular posture margin $m_\phi$ tightens a separate lateral-speed proxy.
Neither is an additional distance margin.

The G1 hardware interface uses Unitree's built-in locomotion controller
for planar velocity, yaw-rate, and standing-height requests. Torso and arm
coordinates are mapped to waist and arm joint-position references.
An external base-pose estimate is transformed into the planning scene
frame; this feedback anchors local clearance adaptation and the conversion
of planar velocity to the robot frame. Reference retiming and low-pass
filtering act upstream of the locomotion commands. The tracking margin is
supplied with the reference or set to a fallback value, separately from the
simulation-based calibration in Eq.~\eqref{eq:app_tracking_margin}.
These measures address the reduced planning model's mismatch with the
hardware tracker. The qualitative demonstrations do not establish a
uniform hardware tracking-error bound.

\begin{table}[tb!]
\centering
\scriptsize
\setlength{\tabcolsep}{3pt}
\begin{tabular}{@{}p{1.8cm}p{\dimexpr\columnwidth-1.8cm-2\tabcolsep\relax}@{}}
\toprule
Component & Value or behavior \\
\midrule
Quadruped & Interior clearance $0.07$, stride $0.22$, maximum swing $0.35$,
touchdown tolerance $0.035$, simulation step $0.01$; leg order
$FL,RR,FR,RL$; phase length 36 simulation steps. \\
Base assistance & Planar PD gains $(140,24)$, height gains $(260,34)$,
roll/pitch gains $(95,12)$, yaw gains $(35,6)$; direct base wrench
scaled by the support-contact count. \\
Evaluation timing & Interface 50 Hz; simulation step $0.002$ s; warmup $0.8$ s;
reference speed factor $0.5$; terminal hold 5 s. \\
Reference feedback & Planar gain $1.5$, correction cap $0.25$ m/s;
yaw feedback gain $0$; planar low-pass weight $0.3$. \\
Adapter speed box & World-$x$ $[-0.5,1.0]$ and world-$y$ $[-0.55,0.55]$ m/s;
local step $0.02$ s. \\
Command slew & No explicit bound on successive locomotion-command
differences. \\
Posture proxy & Pelvis-roll proxy gain $k_\phi=0.08$ rad/(m/s),
$\phi_{\max}=0.30$ rad, $m_\phi=0.05$ rad; lateral-speed bound
$|v_y|\leq(\phi_{\max}-m_\phi)/k_\phi$. \\
Policy clipping & Body-frame planar velocity components $[-0.3,0.3]$ m/s and
yaw rate $[-0.3,0.3]$ rad/s, applied at the policy input. \\
Local correction & $q=1$, $\eta=0$; 12 box/half-space correction sweeps. \\
Clearance queries & Activation band $b_{\mathrm{act}}=0.08$,
finite-difference step $10^{-3}$, lookahead $0.15$. \\
2GO margins & Seed-dependent $m_{\mathrm{track}}$:
A $[0.0237,0.1002]$, B $[0.0795,0.1698]$,
C $[0,0.0752]$, D $[0.0013,0.0896]$. \\
Baseline margin & Common fallback $m_{\mathrm{track}}=0.02$. \\
\bottomrule
\end{tabular}
\caption{Simulation execution settings. Distances are in meters;
adapter bounds and final policy limits are listed separately.}
\label{tab:app_governor_params}
\end{table}

\begin{table*}[t]
\centering
\footnotesize
\setlength{\tabcolsep}{4pt}
\renewcommand{\arraystretch}{1.12}
\begin{tabular}{@{}p{1.7cm}p{4.85cm}p{4.85cm}p{4.85cm}@{}}
\toprule
Quantity & Planning model & Reference adapter & Tracker / policy \\
\midrule
Command magnitude &
Humanoid body-frame velocity and posture-rate bounds; control-magnitude
penalties. &
World-frame position-increment bounds specify reference speeds;
safety-box fallback can relax nominal bounds. &
Body-frame planar velocity and separately generated yaw-rate commands are
clipped at the policy input. \\
Command slew &
Control-magnitude penalties do not constrain successive-command
differences. &
Retiming and low-pass filtering moderate command variation; the
reference-speed box is not a slew bound. &
Input saturation bounds magnitude; no explicit command-slew limit is
imposed. \\
Posture &
Bounds on height, torso/arm posture, and their command rates. &
Planned posture enters body clearance; a pelvis-roll proxy adds a
lateral-speed bound, not a measured-roll constraint. &
Upper-body targets follow a separate mapping/PD path; the learned policy
produces leg targets, followed by joint-limit clipping. \\
Clearance &
Task geometry and planning/CFS margins are evaluated along the model
rollout. &
The corridor wall box and local lateral body-clearance half-space use
$m_{\mathrm{track}}$; finite correction can retain residuals. &
Filtering and saturation follow adaptation without a final-reference
clearance recheck; executed proxy clearance is evaluated afterward. \\
Quadruped footholds and support &
Foot-placement geometry, support-offset bounds, and support-proximity
cost. &
Gait construction, stance anchors, and reachable touchdown projection;
support-polygon membership is diagnostic. &
Foot tracking and joint PD execute the gait with the common supplementary
floating-base wrench. \\
\bottomrule
\end{tabular}
\caption{Planning and execution constraints. Numerical execution settings
are listed in Table~\ref{tab:app_governor_params}.}
\label{tab:app_execution_compatibility}
\end{table*}

\subsection{Metric Definitions and Aggregation}
\label{app:metrics}

Each task and method is evaluated with five seeds. For seed $j$, let $N_j$
be the number of returned planning candidates and let $\mathsf P_{j,i}$ be
the planning success event. We report
\begin{equation}
 \mathrm{pSSR}=\frac15\sum_{j=0}^{4}
 \frac1{N_j}\sum_{i=1}^{N_j}\mathbb I[\mathsf P_{j,i}],
 \quad
 \mathrm{eSSR}=\frac15\sum_{j=0}^{4}\mathbb I[\mathsf E_j],
 \label{eq:app_pssr}
\end{equation}
where $\mathsf E_j$ evaluates the selected, adapted execution for that seed.
All five executions enter the denominator, including trials whose planner
returned no successful candidate.

For stepping stones, $\mathsf P$ requires terminal task-goal error at most
$0.20$ m, maximum mean-over-feet foothold violation at most $0.05$ m,
step-bound violation at most $0.03$ m, and terminal foot error at most
$0.20$ m. Its $pV$ is the selected planning candidate's mean foothold
distance, converted to centimeters. Execution requires the actual final base
position to be within $0.20$ m of the fixed task goal $(1.25,0)$ and the
maximum support-region distance of all recorded swing-end foot positions to
be at most $0.05$ m. The foot statistic includes attempts without detected
touchdown; it evaluates geometric placement rather than asserting contact
establishment. Execution distances use the physical stones and support
platforms without the planning interior margin.

For humanoid planning, $\mathsf P$ requires nonnegative body-proxy clearance
throughout the model rollout and terminal task-goal error at most $0.20$ m.
The occupancy metric $pV$ samples the base polyline at up to 150 longitudinal
locations with eight perimeter points at radius $0.20$ m. It is the fraction
of these band samples inside obstacle geometry, averaged across candidates
and expressed as a percentage. It is distinct from the body-proxy success
predicate, so zero violation under one measure need not imply zero under the
other.

Humanoid execution success requires no fall, actual final base position
within $0.20$ m of the fixed scene goal, and nonnegative minimum body-proxy
SDF over the evaluated rollout. The proxy uses executed base pose and pelvis
height with commanded upper-body posture. It assesses the modeled body
geometry rather than full measured-link collision. Warmup precedes the
evaluated rollout. Warmup compensation translates the reference but not the
scene geometry or task goal; goal error is therefore measured in the original
world frame, separately from error to the selected reference endpoint.

The quadruped deployment interface chooses among adapted candidates.
Humanoid 2GO additionally screens candidates using ungoverned simulated
execution clearance, calibrates the selected reference's margin from that
rollout, and then applies the governed execution; the compared humanoid
baselines retain their planning selection and use the common fallback
margin. Screening uses clearance and non-fall preferences; original-goal
completion is assessed in the final execution by $\mathsf E_j$ above.
Calibration precedes the retimed, governed evaluation, so it does not
measure tracking error under the complete final command pipeline.
Thus eSSR evaluates the
reported selection-and-adaptation pipelines, including their differences,
rather than an equal-cost downstream test of individual planner outputs.
Clearance calibration is part of this pipeline and is not treated as an
independent generalization test.

Means and population standard deviations for $pV$ and $pT$ are computed over
the five seed-level values. Planning time covers the complete multi-candidate
solve, but excludes later execution screening, margin calibration, and
deployment. All results use the task-level predicates above. In particular,
arrival at a selected reference endpoint is not substituted for reaching the
task goal.

\subsection{Paired Candidate-CFS Ablation}
\label{app:cfs_ablation}

We isolate the candidate-correction variant described in
Appendix~\ref{app:task_instantiations} on the
challenging Zone B task. Both variants use 1024 samples per reverse update,
100 updates, ten returned candidates, and seeds 0--4. The only algorithmic
change is enabling or disabling candidate CFS; post-transport retraction remains
disabled. This supplementary study runs on 8 CPU
cores. Its timings are reported separately from the GPU main
comparison.

Table~\ref{tab:app_cfs_ablation} shows that candidate correction increases
mean pSSR from $0.30$ to $0.56$, with a positive paired difference of
$0.20$--$0.30$ for every seed. Mean CPU planning time increases from
$42.65$ to $74.13$ s. The band-occupancy diagnostic also increases, from
$1.275\%$ to $1.665\%$, illustrating its distinction from the body-proxy
success predicate. This comparison supports a success--computation trade-off
for candidate correction in this setting; it does not test final-proposal
retraction, establish a uniform reduction of every violation metric, or
demonstrate an execution-success gain.
\begin{table*}[t]
\centering
\footnotesize
\setlength{\tabcolsep}{9pt}
\begin{tabular}{lrrrrrr}
\toprule
 & \multicolumn{2}{c}{pSSR} & \multicolumn{2}{c}{$pV$ (\%)} & \multicolumn{2}{c}{$pT$ (s)} \\
Seed & CFS on & CFS off & CFS on & CFS off & CFS on & CFS off \\
\midrule
0 & 0.60 & 0.30 & 2.01 & 1.25 & 111.15 & 64.33 \\
1 & 0.70 & 0.40 & 0.50 & 0.67 & 77.40 & 34.31 \\
2 & 0.40 & 0.20 & 1.93 & 1.76 & 58.54 & 27.84 \\
3 & 0.60 & 0.30 & 1.44 & 1.17 & 60.49 & 49.19 \\
4 & 0.50 & 0.30 & 2.45 & 1.53 & 63.07 & 37.58 \\
\midrule
Mean & $0.56\pm0.10$ & $0.30\pm0.06$ & $1.67\pm0.66$ & $1.27\pm0.37$ & $74.13\pm19.66$ & $42.65\pm12.86$ \\
\bottomrule
\end{tabular}
\caption{Paired ablation of the candidate-correction variant in Zone B. Both variants use 1024 samples, 100 reverse updates, ten returned candidates, and the same five seeds. Post-transport retraction is disabled in both. $pV$ is band occupancy (\%); $pT$ is CPU planning time (s). The final row gives mean $\pm$ population standard deviation.}
\label{tab:app_cfs_ablation}
\end{table*}

\subsection{Per-Seed Outcomes}
The following tables report per-seed planning outcomes for all five methods
and six tasks, with execution details where reported. Metrics and aggregation
follow Appendix~\ref{app:metrics}; aggregate eSSR values are reported in
Table~\ref{tab:main_results}. Final goal errors and geometric residuals
separate goal completion from clearance preservation.
\begin{table*}[p]
\centering
\footnotesize
\setlength{\tabcolsep}{6pt}
\renewcommand{\arraystretch}{1.08}
\begin{tabular}{llrrrrrrc}
\toprule
Method & Seed & pSSR & $pV$ & $pT$ & Goal error (m) & Foot distance (m) & Fall & Success \\
\midrule
MPPI & 0 & 1.00 & 2.232 & 3.538 & 0.0622 & 0.0521 & --- & 0 \\
 & 1 & 0.00 & 2.236 & 2.750 & 0.0792 & 0.0526 & --- & 0 \\
 & 2 & 1.00 & 2.300 & 2.626 & 0.1364 & 0.0477 & --- & 1 \\
 & 3 & 0.00 & 1.932 & 2.652 & 0.0974 & 0.0615 & --- & 0 \\
 & 4 & 0.00 & 2.314 & 2.761 & 0.0837 & 0.0520 & --- & 0 \\
\midrule
MBD & 0 & 0.30 & 2.736 & 17.594 & 0.0512 & 0.0454 & --- & 1 \\
 & 1 & 0.30 & 2.972 & 17.730 & 0.0811 & 0.0442 & --- & 1 \\
 & 2 & 0.20 & 3.782 & 17.730 & 0.1033 & 0.0415 & --- & 1 \\
 & 3 & 0.10 & 2.964 & 19.236 & 0.1013 & 0.0631 & --- & 0 \\
 & 4 & 0.00 & 3.276 & 16.957 & 0.0613 & 0.0679 & --- & 0 \\
\midrule
MDOC & 0 & 1.00 & 0.998 & 18.472 & 0.0480 & 0.0483 & --- & 1 \\
 & 1 & 0.90 & 0.879 & 18.086 & 0.0623 & 0.0494 & --- & 1 \\
 & 2 & 0.90 & 1.044 & 18.520 & 0.0342 & 0.0422 & --- & 1 \\
 & 3 & 0.80 & 0.950 & 18.283 & 0.0368 & 0.0467 & --- & 1 \\
 & 4 & 0.80 & 1.107 & 18.507 & 0.0941 & 0.0545 & --- & 0 \\
\midrule
MD-COAS & 0 & 0.30 & 0.897 & 308.795 & 0.5688 & 0.0535 & --- & 0 \\
 & 1 & 0.30 & 0.503 & 250.309 & 0.5226 & 0.0539 & --- & 0 \\
 & 2 & 0.40 & 0.706 & 257.533 & 0.1212 & 0.0539 & --- & 0 \\
 & 3 & 0.20 & 1.035 & 282.719 & 0.1009 & 0.0599 & --- & 0 \\
 & 4 & 0.40 & 1.116 & 341.745 & 0.7527 & 0.0600 & --- & 0 \\
\midrule
2GO & 0 & 1.00 & 0.340 & 23.016 & 0.0353 & 0.0484 & --- & 1 \\
 & 1 & 1.00 & 0.347 & 20.859 & 0.1109 & 0.0412 & --- & 1 \\
 & 2 & 1.00 & 0.504 & 21.605 & 0.0895 & 0.0428 & --- & 1 \\
 & 3 & 0.90 & 0.450 & 22.027 & 0.0416 & 0.0517 & --- & 0 \\
 & 4 & 1.00 & 0.731 & 20.060 & 0.0786 & 0.0499 & --- & 1 \\
\bottomrule
\end{tabular}
\caption{Per-seed outcomes for S1. $pV$ is in cm and $pT$ in seconds. Goal error is measured from the executed final base position to the fixed task goal. Foot distance is the maximum support-region distance of recorded swing-end foot positions, including unsuccessful touchdown attempts. Success is the binary event contributing to eSSR; all five trials are included.}
\label{tab:seed_s1}
\end{table*}

\begin{table*}[p]
\centering
\footnotesize
\setlength{\tabcolsep}{6pt}
\renewcommand{\arraystretch}{1.08}
\begin{tabular}{llrrrrrrc}
\toprule
Method & Seed & pSSR & $pV$ & $pT$ & Goal error (m) & Foot distance (m) & Fall & Success \\
\midrule
MPPI & 0 & 0.00 & 2.719 & 2.647 & 0.0713 & 0.0533 & --- & 0 \\
 & 1 & 0.00 & 2.559 & 2.693 & 0.0823 & 0.0510 & --- & 0 \\
 & 2 & 1.00 & 2.445 & 2.231 & 0.0814 & 0.0550 & --- & 0 \\
 & 3 & 0.00 & 3.133 & 2.794 & 0.1479 & 0.0594 & --- & 0 \\
 & 4 & 0.00 & 2.873 & 2.263 & 0.1402 & 0.0583 & --- & 0 \\
\midrule
MBD & 0 & 0.00 & 2.781 & 16.041 & 0.0639 & 0.0489 & --- & 1 \\
 & 1 & 0.10 & 3.256 & 20.040 & 0.0854 & 0.0425 & --- & 1 \\
 & 2 & 0.00 & 2.784 & 16.968 & 0.1305 & 0.0541 & --- & 0 \\
 & 3 & 0.10 & 2.757 & 14.175 & 0.1028 & 0.0642 & --- & 0 \\
 & 4 & 0.00 & 3.472 & 24.538 & 0.1174 & 0.0633 & --- & 0 \\
\midrule
MDOC & 0 & 0.80 & 1.150 & 20.157 & 0.0572 & 0.0569 & --- & 0 \\
 & 1 & 0.80 & 0.961 & 19.923 & 0.0583 & 0.0588 & --- & 0 \\
 & 2 & 0.70 & 1.107 & 19.675 & 0.0563 & 0.0488 & --- & 1 \\
 & 3 & 0.50 & 1.017 & 19.152 & 0.0657 & 0.0587 & --- & 0 \\
 & 4 & 0.60 & 1.317 & 19.758 & 0.0759 & 0.0565 & --- & 0 \\
\midrule
MD-COAS & 0 & 0.40 & 1.155 & 63.722 & 0.0791 & 0.0598 & --- & 0 \\
 & 1 & 0.10 & 1.887 & 62.096 & 0.1800 & 0.0560 & --- & 0 \\
 & 2 & 0.20 & 1.187 & 151.566 & 0.5852 & 0.0549 & --- & 0 \\
 & 3 & 0.00 & 1.350 & 151.839 & 0.4747 & 0.0558 & --- & 0 \\
 & 4 & 0.00 & 1.666 & 69.675 & 0.6206 & 0.0493 & --- & 0 \\
\midrule
2GO & 0 & 1.00 & 0.432 & 22.597 & 0.0344 & 0.0562 & --- & 0 \\
 & 1 & 0.90 & 0.371 & 22.649 & 0.0632 & 0.0555 & --- & 0 \\
 & 2 & 0.90 & 0.524 & 26.855 & 0.0791 & 0.0499 & --- & 1 \\
 & 3 & 0.80 & 0.654 & 25.094 & 0.0817 & 0.0454 & --- & 1 \\
 & 4 & 0.80 & 1.061 & 18.731 & 0.0418 & 0.0607 & --- & 0 \\
\bottomrule
\end{tabular}
\caption{Per-seed outcomes for S2; metric definitions and units follow Table~\ref{tab:seed_s1}.}
\label{tab:seed_s2}
\end{table*}

\begin{table*}[p]
\centering
\footnotesize
\setlength{\tabcolsep}{6pt}
\renewcommand{\arraystretch}{1.08}
\begin{tabular}{llrrrrrrc}
\toprule
Method & Seed & pSSR & $pV$ & $pT$ & Goal error (m) & Min. proxy SDF (m) & Fall & Success \\
\midrule
MPPI & 0 & 0.00 & 0.667 & 1.966 & 0.1618 & 0.0090 & 0 & 1 \\
 & 1 & 0.00 & 0.750 & 1.163 & 0.1610 & 0.0074 & 0 & 1 \\
 & 2 & 0.00 & 1.333 & 1.136 & 0.1909 & -0.0105 & 0 & 0 \\
 & 3 & 0.00 & 1.250 & 1.476 & 0.4769 & -0.0395 & 0 & 0 \\
 & 4 & 0.00 & 1.750 & 1.244 & 0.2070 & -0.0615 & 0 & 0 \\
\midrule
MBD & 0 & 0.00 & 1.067 & 12.037 & 0.1056 & -0.0525 & 0 & 0 \\
 & 1 & 0.00 & 1.392 & 8.194 & 0.3219 & -0.0239 & 0 & 0 \\
 & 2 & 0.00 & 0.950 & 12.514 & 0.0489 & -0.0468 & 0 & 0 \\
 & 3 & 0.00 & 0.975 & 8.997 & 0.1487 & -0.0189 & 0 & 0 \\
 & 4 & 0.00 & 1.433 & 9.037 & 0.1347 & -0.0261 & 0 & 0 \\
\midrule
MDOC & 0 & 1.00 & 0.000 & 20.518 & 0.0662 & 0.0787 & 0 & 1 \\
 & 1 & 1.00 & 0.000 & 8.574 & 0.0970 & 0.1108 & 0 & 1 \\
 & 2 & 0.90 & 0.000 & 8.040 & 0.1016 & 0.0605 & 0 & 1 \\
 & 3 & 1.00 & 0.000 & 35.614 & 0.0838 & 0.0918 & 0 & 1 \\
 & 4 & 0.90 & 0.000 & 17.433 & 0.0855 & 0.1009 & 0 & 1 \\
\midrule
MD-COAS & 0 & 0.10 & 0.050 & 80.036 & 0.2087 & -0.0205 & 0 & 0 \\
 & 1 & 0.00 & 0.083 & 73.863 & 1.6697 & -0.1153 & 0 & 0 \\
 & 2 & 0.00 & 0.042 & 91.687 & 1.6822 & -0.0088 & 0 & 0 \\
 & 3 & 0.00 & 0.083 & 41.419 & 0.1430 & 0.0579 & 0 & 1 \\
 & 4 & 0.00 & 0.108 & 49.330 & 1.7215 & -0.0035 & 0 & 0 \\
\midrule
2GO & 0 & 1.00 & 0.000 & 27.406 & 0.0544 & 0.1097 & 0 & 1 \\
 & 1 & 0.80 & 0.000 & 8.907 & 0.0498 & 0.1165 & 0 & 1 \\
 & 2 & 1.00 & 0.000 & 9.664 & 0.1032 & 0.1315 & 0 & 1 \\
 & 3 & 1.00 & 0.000 & 10.239 & 0.0871 & 0.0815 & 0 & 1 \\
 & 4 & 1.00 & 0.000 & 14.187 & 0.1120 & 0.0482 & 0 & 1 \\
\bottomrule
\end{tabular}
\caption{Per-seed outcomes for Zone A. $pV$ is in \% and $pT$ in seconds. Goal error is measured from the executed final base position to the fixed task goal. Clearance uses the reduced body proxy at the executed base pose and commanded upper-body posture. Success is the binary event contributing to eSSR; all five trials are included.}
\label{tab:seed_za}
\end{table*}

\begin{table*}[p]
\centering
\footnotesize
\setlength{\tabcolsep}{6pt}
\renewcommand{\arraystretch}{1.08}
\begin{tabular}{llrrrrrrc}
\toprule
Method & Seed & pSSR & $pV$ & $pT$ & Goal error (m) & Min. proxy SDF (m) & Fall & Success \\
\midrule
MPPI & 0 & 0.00 & 3.333 & 2.087 & 0.2071 & -0.0433 & 0 & 0 \\
 & 1 & 0.00 & 3.917 & 1.140 & 0.4910 & -0.1241 & 0 & 0 \\
 & 2 & 0.00 & 3.917 & 1.177 & 0.4753 & -0.2029 & 0 & 0 \\
 & 3 & 0.00 & 2.333 & 1.218 & 0.7573 & -0.0170 & 0 & 0 \\
 & 4 & 0.00 & 3.333 & 1.208 & 0.6378 & -0.2059 & 0 & 0 \\
\midrule
MBD & 0 & 0.00 & 2.858 & 8.711 & 0.1561 & -0.1236 & 0 & 0 \\
 & 1 & 0.00 & 2.892 & 4.003 & 0.1242 & -0.1236 & 0 & 0 \\
 & 2 & 0.00 & 2.725 & 4.717 & 0.0820 & -0.0283 & 0 & 0 \\
 & 3 & 0.00 & 2.883 & 5.147 & 0.1363 & -0.0539 & 0 & 0 \\
 & 4 & 0.00 & 2.558 & 8.178 & 0.0982 & -0.2043 & 0 & 0 \\
\midrule
MDOC & 0 & 1.00 & 0.058 & 25.296 & 0.0940 & -0.0531 & 0 & 0 \\
 & 1 & 0.90 & 0.508 & 29.253 & 0.0888 & -0.0327 & 0 & 0 \\
 & 2 & 1.00 & 0.058 & 48.188 & 0.0345 & 0.0174 & 0 & 1 \\
 & 3 & 0.80 & 1.708 & 48.667 & 0.0365 & -0.0630 & 0 & 0 \\
 & 4 & 0.80 & 1.492 & 44.380 & 0.0999 & -0.0209 & 0 & 0 \\
\midrule
MD-COAS & 0 & 0.60 & 0.033 & 77.704 & 0.1315 & 0.0607 & 0 & 1 \\
 & 1 & 0.40 & 0.092 & 97.515 & 0.1068 & -0.0411 & 0 & 0 \\
 & 2 & 0.40 & 0.342 & 106.994 & 0.0608 & 0.0353 & 0 & 1 \\
 & 3 & 0.30 & 0.325 & 112.412 & 0.1633 & 0.0139 & 0 & 1 \\
 & 4 & 0.30 & 0.117 & 104.324 & 0.1977 & -0.0009 & 0 & 0 \\
\midrule
2GO & 0 & 0.50 & 1.700 & 27.204 & 0.0942 & -0.0230 & 0 & 0 \\
 & 1 & 0.70 & 1.258 & 38.993 & 0.0862 & -0.0047 & 0 & 0 \\
 & 2 & 0.40 & 2.075 & 58.904 & 0.0876 & 0.0282 & 0 & 1 \\
 & 3 & 0.40 & 2.050 & 8.013 & 0.0779 & -0.0457 & 0 & 0 \\
 & 4 & 0.40 & 2.000 & 52.358 & 0.0922 & -0.0226 & 0 & 0 \\
\bottomrule
\end{tabular}
\caption{Per-seed outcomes for Zone B; metric definitions and units follow Table~\ref{tab:seed_za}.}
\label{tab:seed_zb}
\end{table*}

\begin{table*}[p]
\centering
\footnotesize
\setlength{\tabcolsep}{6pt}
\renewcommand{\arraystretch}{1.08}
\begin{tabular}{llrrrrrrc}
\toprule
Method & Seed & pSSR & $pV$ & $pT$ & Goal error (m) & Min. proxy SDF (m) & Fall & Success \\
\midrule
MPPI & 0 & 0.00 & 0.000 & 1.999 & 0.1393 & -0.0753 & 0 & 0 \\
 & 1 & 0.00 & 1.250 & 1.150 & 0.2134 & -0.0792 & 0 & 0 \\
 & 2 & 0.00 & 4.083 & 1.170 & 0.0737 & -0.1283 & 0 & 0 \\
 & 3 & 0.00 & 1.417 & 1.235 & 0.2089 & -0.0359 & 0 & 0 \\
 & 4 & 0.00 & 1.667 & 1.240 & 0.4041 & -0.0181 & 0 & 0 \\
\midrule
MBD & 0 & 0.00 & 1.708 & 4.112 & 0.1093 & -0.0400 & 0 & 0 \\
 & 1 & 0.00 & 2.783 & 3.711 & 0.3239 & -0.0435 & 0 & 0 \\
 & 2 & 0.00 & 3.817 & 5.251 & 0.0461 & -0.0780 & 0 & 0 \\
 & 3 & 0.00 & 3.808 & 4.746 & 0.1495 & -0.0439 & 0 & 0 \\
 & 4 & 0.00 & 3.458 & 6.579 & 0.1354 & -0.1260 & 0 & 0 \\
\midrule
MDOC & 0 & 1.00 & 0.000 & 22.810 & 0.0908 & 0.0661 & 0 & 1 \\
 & 1 & 1.00 & 0.000 & 8.224 & 0.0790 & 0.0622 & 0 & 1 \\
 & 2 & 1.00 & 0.000 & 8.527 & 0.0376 & 0.1145 & 0 & 1 \\
 & 3 & 1.00 & 0.000 & 7.508 & 0.0876 & 0.0297 & 0 & 1 \\
 & 4 & 1.00 & 0.000 & 7.626 & 0.0809 & 0.0515 & 0 & 1 \\
\midrule
MD-COAS & 0 & 0.00 & 0.150 & 112.179 & 0.3606 & -0.0099 & 0 & 0 \\
 & 1 & 0.00 & 0.167 & 101.997 & 0.1524 & 0.0028 & 0 & 1 \\
 & 2 & 0.00 & 0.392 & 104.823 & 0.1498 & -0.0382 & 0 & 0 \\
 & 3 & 0.10 & 0.167 & 103.101 & 0.2418 & 0.0088 & 0 & 0 \\
 & 4 & 0.00 & 0.217 & 104.388 & 0.4024 & 0.0442 & 0 & 0 \\
\midrule
2GO & 0 & 1.00 & 1.025 & 30.184 & 0.0174 & 0.0338 & 0 & 1 \\
 & 1 & 1.00 & 1.433 & 9.710 & 0.0532 & 0.0659 & 0 & 1 \\
 & 2 & 1.00 & 1.350 & 10.363 & 0.0701 & 0.0184 & 0 & 1 \\
 & 3 & 1.00 & 1.275 & 10.401 & 0.0903 & 0.0093 & 0 & 1 \\
 & 4 & 1.00 & 1.075 & 10.470 & 0.0256 & 0.0752 & 0 & 1 \\
\bottomrule
\end{tabular}
\caption{Per-seed outcomes for Zone C; metric definitions and units follow Table~\ref{tab:seed_za}.}
\label{tab:seed_zc}
\end{table*}

\begin{table*}[p]
\centering
\footnotesize
\setlength{\tabcolsep}{6pt}
\renewcommand{\arraystretch}{1.08}
\begin{tabular}{llrrrrrrc}
\toprule
Method & Seed & pSSR & $pV$ & $pT$ & Goal error (m) & Min. proxy SDF (m) & Fall & Success \\
\midrule
MPPI & 0 & 0.00 & 0.250 & 1.997 & 0.1389 & -0.0261 & 0 & 0 \\
 & 1 & 0.00 & 0.500 & 1.230 & 0.2114 & -0.0297 & 0 & 0 \\
 & 2 & 0.00 & 0.917 & 1.117 & 0.0761 & -0.0681 & 0 & 0 \\
 & 3 & 0.00 & 0.000 & 1.124 & 0.2080 & -0.0632 & 0 & 0 \\
 & 4 & 0.00 & 0.583 & 1.280 & 0.4044 & -0.0273 & 0 & 0 \\
\midrule
MBD & 0 & 0.00 & 0.558 & 4.821 & 0.1049 & -0.0354 & 0 & 0 \\
 & 1 & 0.00 & 0.658 & 4.346 & 0.3199 & -0.0621 & 0 & 0 \\
 & 2 & 0.00 & 0.533 & 5.039 & 0.0474 & -0.0590 & 0 & 0 \\
 & 3 & 0.00 & 0.542 & 36.510 & 0.1486 & -0.0498 & 0 & 0 \\
 & 4 & 0.00 & 0.483 & 3.145 & 0.1347 & -0.0537 & 0 & 0 \\
\midrule
MDOC & 0 & 0.80 & 0.108 & 18.507 & 0.0803 & 0.0134 & 0 & 1 \\
 & 1 & 0.90 & 0.083 & 7.776 & 0.0630 & 0.0084 & 0 & 1 \\
 & 2 & 0.80 & 0.108 & 7.789 & 0.0535 & -0.0984 & 0 & 0 \\
 & 3 & 0.80 & 0.058 & 8.226 & 0.0307 & -0.0009 & 0 & 0 \\
 & 4 & 1.00 & 0.067 & 8.412 & 0.0501 & -0.0319 & 0 & 0 \\
\midrule
MD-COAS & 0 & 0.20 & 0.050 & 50.780 & 0.1339 & -0.0390 & 0 & 0 \\
 & 1 & 0.30 & 0.092 & 30.882 & 0.1772 & 0.0010 & 0 & 1 \\
 & 2 & 0.20 & 0.100 & 33.287 & 0.1886 & 0.0508 & 0 & 1 \\
 & 3 & 0.30 & 0.025 & 42.675 & 0.1226 & -0.0236 & 0 & 0 \\
 & 4 & 0.20 & 0.092 & 39.627 & 0.0942 & 0.0046 & 0 & 1 \\
\midrule
2GO & 0 & 0.90 & 0.000 & 25.237 & 0.0907 & -0.0256 & 0 & 0 \\
 & 1 & 1.00 & 0.000 & 7.716 & 0.1141 & -0.0235 & 0 & 0 \\
 & 2 & 1.00 & 0.000 & 8.049 & 0.0435 & 0.0451 & 0 & 1 \\
 & 3 & 1.00 & 0.000 & 8.276 & 0.1829 & 0.0432 & 0 & 1 \\
 & 4 & 1.00 & 0.000 & 7.446 & 0.0693 & 0.0450 & 0 & 1 \\
\bottomrule
\end{tabular}
\caption{Per-seed outcomes for Zone D; metric definitions and units follow Table~\ref{tab:seed_za}.}
\label{tab:seed_zd}
\end{table*}

\clearpage

\section{Baseline Details}
\label{app:baselines}

All planners use the same task dynamics and objective formulation for a given
environment. The stochastic methods are compared at matched nominal sample
and update budgets: 512 samples for the quadruped, 1024 for the humanoid, and
100 optimizer or reverse steps. Diffusion-family methods return ten candidates
per seed. Matching these counts does not imply identical
wall-clock cost because geometry construction, automatic differentiation, and
CFS calls add different overheads.

\begin{table}[H]
\centering
\scriptsize
\setlength{\tabcolsep}{3pt}
\begin{tabular}{@{}lp{6.3cm}@{}}
\toprule
Method & Constraint and sampling mechanism \\
\midrule
MPPI & Gaussian action perturbations with exponential cost weighting; no
geometry-aware reverse transport; 100 update iterations. \\
MBD & Model-based denoising with task cost and no explicit local constraint
geometry. \\
MDOC & MBD with CBF-style derivative correction using the configured step,
gain, and margin; no 2GO metric/tangent split or route/risk gate. \\
MD-COAS & Model-based diffusion with its own active-set correction, scheduled
margins, and configured cap on local constraints. \\
2GO & Reward-weighted reverse transport with task-local metric and tangent
maps and the task-specific correction variants in
Appendix~\ref{app:task_instantiations}. \\
\bottomrule
\end{tabular}
\caption{Baseline mechanisms used in the reported comparison.}
\label{tab:app_baseline_detail}
\end{table}

\newpage
\section{Hyperparameters and Scheduling}
\label{app:hparams}

\subsection{Risk, Dispersion, and Task Modulation}

For each candidate, write $V_{k,m}=V(\tilde\tau_{k,m},c)$ for the
nonnegative aggregate residual used in Eq.~\eqref{eq:risk_proxy_k}.
The empirical upper-tail risk at level $\alpha=0.9$ is
\begin{equation}
 \nu_k=\frac{1}{|\mathcal T_{\alpha,k}|}
 \sum_{m\in\mathcal T_{\alpha,k}}V_{k,m},
 \label{eq:app_cvar}
\end{equation}
where $\mathcal T_{\alpha,k}$ indexes the largest
$\lceil(1-\alpha)M\rceil$ candidate residuals. This aggregate controls
refinement; the task-specific planning violation $\mathrm{pV}$ is defined
separately in Appendix~\ref{app:metrics}.

Using the feasibility level $h_k$ and denoising gain $\kappa_k$ defined in
Section~\ref{subsec:mcsa_geometry_tau}, the numerical schedules are
\begin{align}
 \delta_k&=\max\!\left(\delta_{\min},
 \frac{c_\delta m_k}{1+c_\mu\mu_k}\right),\\
 \theta_k&=\theta_{\max}-(\theta_{\max}-\theta_{\min})h_k,\\
 \eta_k&=\eta_{\mathrm{base}}\left[\eta_{\min}
 +(\eta_{\max}-\eta_{\min})(1-h_k)^{p_\eta}\right].
 \label{eq:app_overlay}
\end{align}
Here $m_k$ is the scheduled clearance margin, distinct from the CFS
offset $b_i$. The main-text dispersion statistic $\chi_k$ uses the final
quarter of the horizon, restricted to the active window, with exponential smoothing of
coefficient $0.7$. The transport family follows
Eq.~\eqref{eq:transport_selection_main}.

The reported experiments share noise thresholds across transport families:
$\bar\chi_k^{\rho_k}=\theta_k$ and
$\bar\nu_k^{\rho_k}=0.05(1+4h_k)$ in Eq.~\eqref{eq:rho_aware_gate}.
Including the additional annealed control-noise scale $\sigma_u$, the
nominal clean-coordinate amplitude is
\begin{equation}
 \sigma_k^{\rho_k}=\sigma_{\max}(1-h_k)^{p_\sigma}
 +\sigma_u\frac{k-1}{K-1},\quad k=K,\ldots,1.
\end{equation}
This scale is also shared across families; the effective amplitude is
$\sigma_{k,\mathrm{eff}}^{\rho_k}=\gamma_k^{\rho_k}\sigma_k^{\rho_k}$.
A cyclic window of $W$ control rows,
advanced by stride $S$, localizes the geometric update. Its time indices
start at $((K-k)S)\bmod H$ and wrap at the horizon boundary, so the first
reverse step starts at control row zero.

\subsection{Numerical Settings and Selection Principles}

Geometric margins use the same distance units as the corresponding task
constraints; phase weights express the changing role of support and swing
feet. The dimensionless gains balance task progress, geometric correction,
and exploration. Values below are fixed manually and shared across seeds
unless stated otherwise.

\begin{table*}[t]
\centering
\scriptsize
\setlength{\tabcolsep}{6pt}
\begin{tabular}{lll|lll}
\toprule
Parameter & Symbol & Value & Parameter & Symbol & Value \\
\midrule
Reverse steps & $K$ & 100 & Sampling temperature & $T_{\mathrm{samp}}$ & 0.3 \\
Beta endpoints & $(\beta_0,\beta_1)$ & $(10^{-5},10^{-2})$ & Returned modes & $N$ & 10 \\
Seeds & --- & $0,1,2,3,4$ & Active-row cap & $K_A$ & 8 \\
CVaR tail level & $\alpha$ & 0.9 & Geometry regularizers & $\varrho,\epsilon_G,\epsilon_P$ & $10^{-6}$ \\
Base refinement gain & $\eta_{\mathrm{base}}$ & 0.08 & Route EMA coefficient & --- & 0.7 \\
Penalty reference & $\mu_{\mathrm{ref}}$ & 500 & $(\kappa_0,\kappa_\mu)$ & --- & $(1,0.01)$ \\
Margin-to-risk scale & $c_\delta$ & 0.4 & Penalty attenuation & $c_\mu$ & 0.02 \\
Minimum risk gate & $\delta_{\min}$ & $5\times10^{-4}$ & Route thresholds & $(\theta_{\max},\theta_{\min})$ & $(0.6,0.2)$ \\
Gain range & $(\eta_{\min},\eta_{\max})$ & $(0.5,1.0)$ & Powers & $(p_\sigma,p_\eta)$ & $(1,1)$ \\
\bottomrule
\end{tabular}
\caption{Common 2GO generative and geometry parameters.}
\label{tab:app_twogo_common}
\end{table*}

\begin{table*}[t]
\centering
\scriptsize
\setlength{\tabcolsep}{4.4pt}
\begin{tabular}{lcccccccc}
\toprule
Task & $M$ & $W/S$ & $\sigma_{\max}$ & $\sigma_u$ & $\lambda_g$ & $m_k$ & $b_i$ & Corrected candidates \\
\midrule
Stones S1/S2 & 512 & $8/4$ & 0.10 & 0.20 & 0.0 & 0.02 & --- & 0 \\
Zone A & 1024 & $6/3$ & 0.15 & 0.15 & 0.3 & 0.02 & 0.01 & 32 \\
Zone B & 1024 & $6/3$ & 0.15 & 0.35 & 0.1 & 0.05 & 0.04 & 32 \\
Zone C & 1024 & $6/3$ & 0.15 & 0.15 & 0.3 & 0.02 & 0.01 & 32 \\
Zone D & 1024 & $6/3$ & 0.15 & 0.15 & 0.3 & 0.07 & 0.01 & 32 \\
\bottomrule
\end{tabular}
\caption{Task-specific parameters for the experimental variants in
Appendix~\ref{app:task_instantiations}. Here $\lambda_g$ is the task-direction
gain, $m_k$ the scheduled clearance margin, and $b_i$ the candidate-CFS
offset. Post-transport retraction is disabled in these experiments.}
\label{tab:app_twogo_task_params}
\end{table*}

\begin{table*}[t]
\centering
\scriptsize
\setlength{\tabcolsep}{4.8pt}
\begin{tabular}{lp{5.7cm}p{8.4cm}}
\toprule
Task & Term & Coefficient \\
\midrule
Stepping stones & Smoothness & $0.92[0.165$ body step $+0.15$ yaw
$+0.165$ foot residual rate $+0.042$ phase rate$]$. \\
Stepping stones & Feasibility and shape & $1.12(40)$ foothold violation;
$1.08(14)$ support violation; $1.0[1.2$ geometric shape $+0.8$ residual
shape$]$; step-bound factors $8$ and $4$. \\
Stepping stones & Task terms & Terminal goal weight $120$, terminal-foot
multiplier $0.5$, plus progress, time guide, lateral/velocity/yaw/backward,
$0.08$ heading, $1.05$ support-center, and phase-ramped terminal-foot terms. \\
Humanoid & Task terms & $0.2w_g$ terminal goal, $1.0$ corridor centering,
$0.3$ each arm tuck, $0.3$ torso yaw, $6.0$ height, $0.05$ velocity,
$0.06$ control, $-0.3$ progress, and $5.0$ squeeze term. \\
Humanoid & Zone-specific goal & $w_g=5$ in Zone C and $w_g=1$ in Zones A,
B, and D. Goal deceleration is enabled only in Zone C. \\
Humanoid & Obstacle cost & $0$; obstacle feasibility enters through the
violation evaluator, local geometry, augmented-Lagrangian schedule, and
candidate correction. \\
\bottomrule
\end{tabular}
\caption{Explicit weights used by the task objectives.}
\label{tab:app_energy_weights}
\end{table*}

\begin{table*}[t]
\centering
\scriptsize
\setlength{\tabcolsep}{5pt}
\begin{tabular}{llccp{7.4cm}}
\toprule
Domain & Method & Samples & Steps & Additional parameters \\
\midrule
Quadruped & MPPI & 512 & 100 & Action noise $0.20$, temperature $1.0$. \\
Quadruped & MBD & 512 & 100 & $T_{\mathrm{samp}}=0.3$, extra action noise $0.02$, 10 returned modes. \\
Quadruped & MDOC & 512 & 100 & $T_{\mathrm{samp}}=0.3$, correction step $0.01$, gain $0.1$,
margin $0.01$, 10 modes. \\
Quadruped & MD-COAS & 512 & 100 & $T_{\mathrm{samp}}=0.3$, at most 12 constraints,
constraint margin $0.01$, scheduler margin $0.02$, 10 modes. \\
Humanoid & MPPI & 1024 & 100 & Action noise $0.15$ ($0.35$ in Zone B), temperature $1.0$. \\
Humanoid & MBD & 1024 & 100 & $T_{\mathrm{samp}}=0.3$, extra action noise $0.15$ ($0.35$ in Zone B), 10 modes. \\
Humanoid & MDOC & 1024 & 100 & $T_{\mathrm{samp}}=0.3$, correction step $0.02$, gain $1.2$,
margin $0.03$, 10 modes. \\
Humanoid & MD-COAS & 1024 & 100 & $T_{\mathrm{samp}}=0.3$, at most 8 constraints,
constraint margin $0.005$, 10 modes; scheduler margins are task-specific. \\
\bottomrule
\end{tabular}
\caption{Optimizer parameters used in the baseline comparison.}
\label{tab:app_baseline_params}
\end{table*}

\end{document}